\documentclass[final]{article}
\usepackage{iclr2022_conference,times}
\iclrfinalcopy

\usepackage[vlined,linesnumbered,ruled,resetcount]{algorithm2e}
\newcommand{\nextnr}{\stepcounter{AlgoLine}\ShowLn}
\SetKwFor{For}{for}{do}{endfor for}%
\usepackage{hyperref} 
    \hypersetup{
    	colorlinks = true,
    	linkcolor = blue,
    	anchorcolor = blue,
    	citecolor = blue,
    	filecolor = blue,
    	urlcolor = blue
    }
\usepackage{url}
\usepackage{booktabs}
\usepackage{nicefrac}
\usepackage{times}
\usepackage{enumitem}
\usepackage{footnote}
\makesavenoteenv{tabular}
\makesavenoteenv{table}
\usepackage{etoolbox}
\AfterEndEnvironment{equation}{\reseteqhint}
\AfterEndEnvironment{equation*}{\reseteqhint}
\AfterEndEnvironment{align}{\reseteqhint}
\AfterEndEnvironment{align*}{\reseteqhint}
\newcounter{hints}
\renewcommand{\thehints}{\alph{hints}}
\newcommand{\eqhint}[2][]{%
  \stepcounter{hints}%
  \if\relax\detokenize{#1}\relax\else\csxdef{hint@#1}{\thehints}\fi
  \mathrel{\overset{\textrm{(\thehints\hspace{0.01em})}}{\vphantom{\le}{#2}}}%
}
\newcommand{\reseteqhint}{\setcounter{hints}{0}}

\makeatletter
 \let\Ginclude@graphics\@org@Ginclude@graphics 
\makeatother
\usepackage{float}
\usepackage{subcaption}
\usepackage{graphicx}
\title{Universal Approximation Under Constraints is Possible with Transformers}
\author{Anastasis Kratsios\thanks{Corresponding authors.}, Tianlin Liu \& Ivan Dokmani\'{c}\\
Universit\"{a}t Basel, \\
Departement Mathematik und Informatik
\\
\texttt{\{firstname.lastname\}@unibas.ch} \\
\And
Behnoosh Zamanlooy$^*$ \\
Universit\"{a}t Z\"{u}rich, \\
Department of Informatics
\\
\texttt{bzamanlooy@ifi.uzh.ch} \\
}

\usepackage[utf8]{inputenc}
\usepackage{appendix}
\usepackage{mathtools}
\usepackage{bbm}
\usepackage{booktabs}
\usepackage{xparse}
\usepackage{sidecap}
\usepackage{float}
\usepackage{mdwlist}
\usepackage[bbgreekl]{mathbbol}
\usepackage{amsfonts}
\usepackage{amscd}
\usepackage{amsmath,amsfonts,mathrsfs}
\usepackage{amsthm}
\usepackage{cleveref}
\usepackage{bm}
\usepackage{latexsym}
\usepackage{mleftright}
\usepackage{color}
\usepackage{xifthen}
\usepackage{makeidx}
\usepackage{amsmath}
\usepackage{amsfonts}
\usepackage{amssymb}
\usepackage{mathrsfs}
\usepackage{tikz,tikz-cd}
\usepackage{tkz-euclide}
\usetikzlibrary{matrix}
\usetikzlibrary{calc}
\usetikzlibrary{shapes.geometric,shapes.misc, shapes.symbols}
\usetikzlibrary{fit}
\usepackage{smartdiagram}
\usepackage{multirow}

\newcommand{\rr}{{\mathbb{R}}}
\newcommand{\pp}{{\mathbb{P}}}
\newcommand{\qq}{{\mathbb{Q}}}
\newcommand{\rrflex}[1]{{\ensuremath{\rr^{#1}
}}}

\newcommand{\rrn}{{\rrflex{n}}}

\newcommand{\rrm}{{\rrflex{m}}}

\newcommand{\xxx}{\mathcal{X}}
\newcommand{\yyy}{\mathcal{Y}}

\newcommand{\ee}{{\mathbb{E}}}
\newcommand{\nn}{{\mathbb{N}}}
\newcommand{\zz}{{\mathbb{Z}}}

\NewDocumentCommand{\ppp}{mo}{
                \ensuremath{
                            \mathcal{P}\IfValueT{#2}{_{{#2}}}\IfValueF{#2}{_1}
                            \left(
                                {#1}
                            \right)
                }
                    }

\newcommand{\www}{{\mathscr{W}}}

\NewDocumentCommand{\NN}{oo}{
    \ensuremath{
        \mathcal{NN}
        \IfValueT{#1}{_{#1}}\IfValueF{#1}{_{d:D}}
        \IfValueT{#2}{^{#2}}\IfValueF{#2}{^{\sigma}}
    }
}
\newtheorem{previewtheorem}{Informal Theorem}[section]
\newtheorem{example}[previewtheorem]{Example}
\newtheorem{remark}[previewtheorem]{Remark}
\newtheorem*{notee}{Note}

\newtheorem{corollary}[previewtheorem]{Corollary}
\newtheorem{theorem}[previewtheorem]{Theorem}
\newtheorem{lemma}[previewtheorem]{Lemma}

\newtheorem{definition}[previewtheorem]{Definition}

\newtheorem{assumption}[previewtheorem]{Assumption}

\let\tmp\newinsert
\let\newinsert\newbox
\let\newinsert\tmp
\usepackage{adjustbox}

\usepackage{xcolor}
\definecolor{TealBlue}{RGB}{54,117,136}
\definecolor{darkcerulean}{rgb}{0.03, 0.27, 0.49}
\definecolor{darkmidnightblue}{rgb}{0.0, 0.2, 0.4}
\definecolor{darkcyan}{rgb}{0.0, 0.55, 0.55}
\definecolor{forestgreen}{RGB}{34,139,34}
\definecolor{darkgreen}{rgb}{0.0, 0.2, 0.13}
\definecolor{deepjunglegreen}{rgb}{0.0, 0.29, 0.29}
\definecolor{BurntOrange}{RGB}{204,85,0}
\definecolor{darkcandyapplered}{rgb}{0.64, 0.0, 0.0}
\definecolor{darkred}{rgb}{0.55, 0.0, 0.0}
\definecolor{darkscarlet}{rgb}{0.34, 0.01, 0.1}
\definecolor{jasper}{rgb}{0.84, 0.23, 0.24}
\definecolor{darkjazzberryjam}{rgb}{0.45, 0.04, 0.37}
\definecolor{darkpurpleheart}{rgb}{0.36, 0.17, 0.56}
\definecolor{lightpurpleheart}{rgb}{0.66, 0.47, 0.86}
\definecolor{MidnightBlue}{RGB}{25,25,112}
\definecolor{MidnightBlueComplementingGreen}{RGB}{25,112,25}
\definecolor{MidnightBlueComplementingPurple}{RGB}{112,25,112}
\definecolor{MidnightBlueComplementingRed}{RGB}{112,25,69}
\begin{document}
\maketitle
\begin{abstract}
Many practical problems need the output of a machine learning model to satisfy a set of constraints, $K$. There are, however, no known guarantees that classical neural networks can exactly encode constraints while simultaneously achieving universality.  We provide a quantitative constrained universal approximation theorem which guarantees that for any convex or non-convex compact set $K$ and any continuous function $f:\mathbb{R}^n\rightarrow K$, there is a probabilistic transformer $\hat{F}$ whose randomized outputs all lie in $K$ and whose expected output uniformly approximates $f$.  Our second main result is a ``deep neural version'' of \citet{BergeMaximumOG_Articel1963}'s Maximum Theorem.  The result guarantees that given an objective function $L$, a constraint set $K$, and a family of soft constraint sets, there is a probabilistic transformer $\hat{F}$ that approximately minimizes $L$ and whose outputs belong to $K$; moreover, $\hat{F}$ approximately satisfies the soft constraints.  Our results imply the first universal approximation theorem for classical transformers with exact convex constraint satisfaction, and a chart-free universal approximation theorem for Riemannian manifold-valued functions subject to geodesically-convex constraints.
\end{abstract}
\hfill\\
%
\section{Introduction}\label{s_intro}
In supervised learning, we select a parameterized model $\hat{f}:\rr^n\rightarrow \rr^m$ by optimizing a real-valued loss function%
\footnote{For example, in a regression problem one can set $L(x, y) = \| f(x) - y \|$ for an unknown function $f$ or in regression problems one sets $L(x,y)=\sum_{i=1}^m [C(x)]_i\log(y_i)$ for an unknown classifier $C$.}%
$L$ over training data from an input-output domain $\xxx\times \yyy\subseteq \rr^n\times \rr^m$.  A necessary property for a model class to produce asymptotically optimal results, for any continuous loss $L$, is the {\color{darkcerulean}{universal approximation property}}.  However, often more {\color{darkgreen}{structure}} (beyond vectorial $\rr^m$) is present in a learning problem and this structure must be encoded into the trained model $\hat{f}$ to obtain meaningful or feasible predictions.  This additional structure is typically described by a constraint set $K\subseteq \rr^m$ and the condition $\hat{f}(\xxx)\subseteq K$.  For example, in classification $K=\{y\in [0,1]^m:\,\sum_{i=1}^m y_i =1\}$ \citep{ShalevShai22014UnderstandingML}, in Stackelberg games \citep{holters2018playing,jin2020local,li2021complexity} $K$ is the set of utility-maximizing actions of an opponent, in integer programming $K$ is the integer lattice $\zz^m$ \citep{conforti2014integer}, in financial risk-management $K$ is a set of positions meeting the minimum solvency requirements imposed by international regularity bodies \citep{FRTB,MinCapReq,PaulQRM}, in covariance matrix prediction $K\subseteq \rr^{m\times m}$ is the set of $m\times m$ matrices which are symmetric and positive semi-definite \citep{bonnabel2013rank,BonnabelPSD,baes2021lowrank}, in geometric deep learning $K$ is typically a manifold {\color{black}{
(e.g. a pose manifold in computer vision and robotics \citep{IEEEding2014multilayer} or a manifold of distance matrices \citep{IEEEEuclideanDistanceMatrices})}}, a graph, or an orbit of a group action \citep{bronstein2017geometric,bronstein2021geometric,kratsios2020non}.  Therefore, we ask: 
\[
\textit{Is exact constraint satisfaction possible with universal deep learning models?}
\]

The answer to this question begins by examining the classical universal approximation theorems for deep feedforward networks.   If $L$ and $K$ are mildly regular, the universal approximation theorems of \cite{hornik1989multilayer,Cybenko,pinkus1999approximation,Yarotski,kidger2019universal,park2020minimum} guarantee that for any ``good activation function $\sigma$'' and for every tolerance level $\epsilon>0$, there is a deep feedforward network with activation function $\sigma$, such that $\inf_{y \in K}\, L(x,y)$ and $L(x,\hat{f}(x))$ are uniformly at most $\epsilon$ apart.  Written in terms of the optimality set,
\begin{equation}
\sup_{x\in \mathcal{X}}
\|\hat{f}(x) -
\underset{
{\color{darkgreen}{y \in K}}
}{\operatorname{argmin}}
\,
L(x,y)
\|
{
\color{darkcerulean}{
        \leq 
\epsilon
}
}
\label{eq_motivation_1_constrained_approx}
,
\end{equation}
where the distance of a point $y \in \rr^m$ to a set $A\subseteq \rr^m$ is defined by $\|y-A\|\triangleq \inf_{a\in A}\|y-a\|$.  
Since $\operatorname{argmin}_{y \in K} L(x,y)\subseteq K$, then~\eqref{eq_motivation_1_constrained_approx} only implies that $\|\hat{f}(x)-K\|\leq \epsilon$ and there is no reason to believe that the {\color{darkgreen}{constraint $
\hat{f}(x)\in K
$}} is exactly satisfied, for every $x \in \xxx$.  

This kind of \textit{approximate} constraint satisfaction is not always appropriate.  In the following examples constraint violation causes either practical or theoretical concerns:
\begin{enumerate}[nolistsep]
    \item[(i)] In post-financial crisis risk management, international regulatory bodies mandate that any financial actor should maintain solubility proportional to the risk of their investments \citep{FRTB,MinCapReq}.  To prevent future financial crises, any violation of these risk constraints, no matter the size, incurs large and immediate fines.  
    \item[(ii)] In geometric deep learning, we often need to encode complicated non-vectorial structure present in a dataset, by viewing it as a $K$ valued function \citep{fletcher2011geodesic,BonnabelPSD,baes2021lowrank}.  However, if $K$ is non-convex then \citet{motzkin1935quelques} confirms that there is no unique way to map predictions $\hat{f}(x)\not\in K$ to a closest point in $K$.  Thus, we are faced with the dilemma: either make an ad-hoc \textit{choice} of a $ k$ in $K$ with $k\approx \hat{f}(x)$ (ex.: an arbitrary choice scheme when $K=\zz^m$) or have meaningless predictions (ex: non-integer values to integer programs, or symmetry breaking \citep{weinberg1976implications}\footnote{As discussed in \cite{rosset2021replab} this is problematic since respecting symmetries can often massively reduce the computational burden of a learning task.}).  
\end{enumerate}

Constrained learning was recognized as an effective framework for fairness and robustness by \cite{chamon2020probably} who study empirical risk minimization under constraints. Many emerging topics in machine learning lead to constrained learning formulations. A case in point is model-based domain generalization \citep{robey2021model}.  
Despite the importance of (deep) learning with constraints, there are no related approximation-theoretic results to the best of our knowledge.

In this paper, we bridge this theoretical gap by showing that {\color{darkcerulean}{universal approximation}} with {\color{darkgreen}{exact constraint satisfaction}} is always possible for deep \textit{(probabilistic) transformer networks} with a single attention mechanism as output layer.  Our contribution is three-fold:
\begin{enumerate}[nolistsep]
    \item We derive the first universal approximation theorem with exact constraint satisfaction;
    \item Our transformer network's encoder and decoder adapt to the dimension of the constraint set and thus beat the curse of dimensionality for low-dimensional constraints;
    \item Our models leverage a probabilistic attention mechanism that can encode non-convex constraints. This probabilistic approach is key to bypass the topological obstructions to non-Euclidean universal approximation \citep{kratsios2021_GDL}.
\end{enumerate}
Our analysis provides perspective on the empirical success of attention and adds to the recent line of work on approximation theory for transformer networks, \citep{Yun2020TransformersUniversalSequencesICRL,Yun2020SparseTransformersUniversalNeurIPS2020}, which roughly considers the unconstrained case (with $K$ in~\eqref{eq_motivation_1_constrained_approx} replaced by $\rr^m$) in the special case of $L(x,y)=\|f(x)-y\|$ for a suitable target function $f:\rr^n\rightarrow \rr^m$.  Our probabilistic perspective on transformer networks fits with the representations of \cite{vuckovic2021regularity} and of \cite{kratsios2021_GCDs}.

Our results can be regarded as an approximation-theoretic counterpart to the constrained statistical learning theory of \cite{chamon2020probably}.  Further, they put forward a perspective on randomness in neural networks that is complementary to the work of  \citet{RandomRNN_RandomMatrixPerspective_AnnStat2018,GononJPLyudmila_RandomJMLR2020,GononJPLyudmila_RandomRNNs2021}. We look at the same problem focusing on constraint satisfaction instead of training efficiency.  Finally, our proof methods are novel, and build on contemporary tools from metric geometry \citep{AmbriosioPuglisi2020RandomProjections,Brue2021Extension}.  

\vspace{-.5em}
\subsection{The Probabilistic Attention Mechanism}\label{s_Intro_ss_Attention}
We now give a high-level explanation of our results; the detailed formulations are in Section~\ref{s_Main_Results}.  

Introduced in \citep{Bahdanauetal2015ICLRAttentionIsInvented} and later used to define the transformer architecture \citep{vaswani2017attention}, in the NLP context, \textit{attention} maps a matrix of queries $Q$, a matrix of keys $K$, and a matrix of values $V$ to the quantity $\operatorname{Softmax}(QK^{\top})V$, where the softmax function (defined below) is applied row-wise to $QK^{\top}$. Just as the authors of \citep{PetersenVoigtaender2020Convnetquievalencetoffnn,Zhou2020UniversalConvNetsAppliedHarmonicAnalysis} focus on the simplified versions of practically implementable ConvNets in the study of approximation theory of deep ConvNets (e.g. omitting pooling layers), we find it sufficient to study the following simplified attention mechanism to obtain universal approximation results:
\begin{equation}
    \operatorname{Attention}\left(w,Y\right)\triangleq 
\operatorname{Softmax}_N\left(w\right)^{\top}Y
= \sum_{n=1}^N [\operatorname{Softmax}_N(w)_n]Y_n
\label{eq_definition_attention}
,
\end{equation}
where $w\in \rr^N$, $\operatorname{Softmax}_N:\rr^N\ni w \mapsto (\frac{e^{w_k}}{\sum_{j=1}^N e^{w_j}})_{k=1}^N$, and $Y$ is an $N\times m$ matrix.  The attention mechanism~\eqref{eq_definition_attention} can be interpreted as ``paying attention'' to a set of particles $Y_1,\dots,Y_N\in \rr^m$ defined by $Y$'s rows.  This simplified form of attention is sufficient to demonstrate that transformer networks can approximate a function while respecting a constraint set, $K$, whether convex or non-convex. 
\begin{previewtheorem}[Deep Maximum Theorem for Transformers]\label{previewtheorem_Convex_Case}
If {\color{darkgreen}{$K$ is convex}} and the quantities defining~\eqref{eq_motivation_1_constrained_approx} are regular then, for any $\epsilon\in (0,1]$, there is a feedforward network $\hat{f}$, an ${\xxx_{\epsilon}}\subset \rr^n$ of probability $1\mbox{-}\epsilon$, and a matrix $Y$ such that the transformer $\operatorname{Attention}(\hat{f}(x),Y)$ satisfies:
\begin{enumerate}[nolistsep]
    \item[(i)] \textbf{Exact Constraint Satisfaction:} For each $x\in \rr^n$, $\operatorname{Attention}(\hat{f}(x),Y){\color{darkgreen}{\in K}}$,
    \item[(ii)] \textbf{Universal Approximation:} $\sup_{x\in \xxx_{\epsilon}}\,
\|\operatorname{Attention}(\hat{f}(x),Y) -
\underset{
{\color{darkgreen}{y^{\star}\in K}}
}{\operatorname{argmin}}
\,
L(x,y^{\star})
\|
{
\color{darkcerulean}{
        \leq 
\epsilon
}}
$
\end{enumerate}
\end{previewtheorem}
Informal Theorem~\ref{previewtheorem_Convex_Case} guarantees that simple transformer networks can minimize any loss function while exactly satisfying the set of {\color{darkgreen}{convex constraints}}.  As illustrated by Figure~\ref{fig_convex_classical_transformers} and Figure~\ref{fig_nonconvex_classical_transformers}, $K$'s convexity is critical here, since without it the transformer's prediction may fail to lie in $K$.  This is because any transformer network's output is a convex combinations of the \textit{particles} $Y_1,Y_2,Y_3$; thus, any transformer network's predictions must belong to these particles' convex hull.

\begin{figure}[H]
\centering
\begin{minipage}[b]{0.4\linewidth}
\centering
\includegraphics[width=0.4 \linewidth]{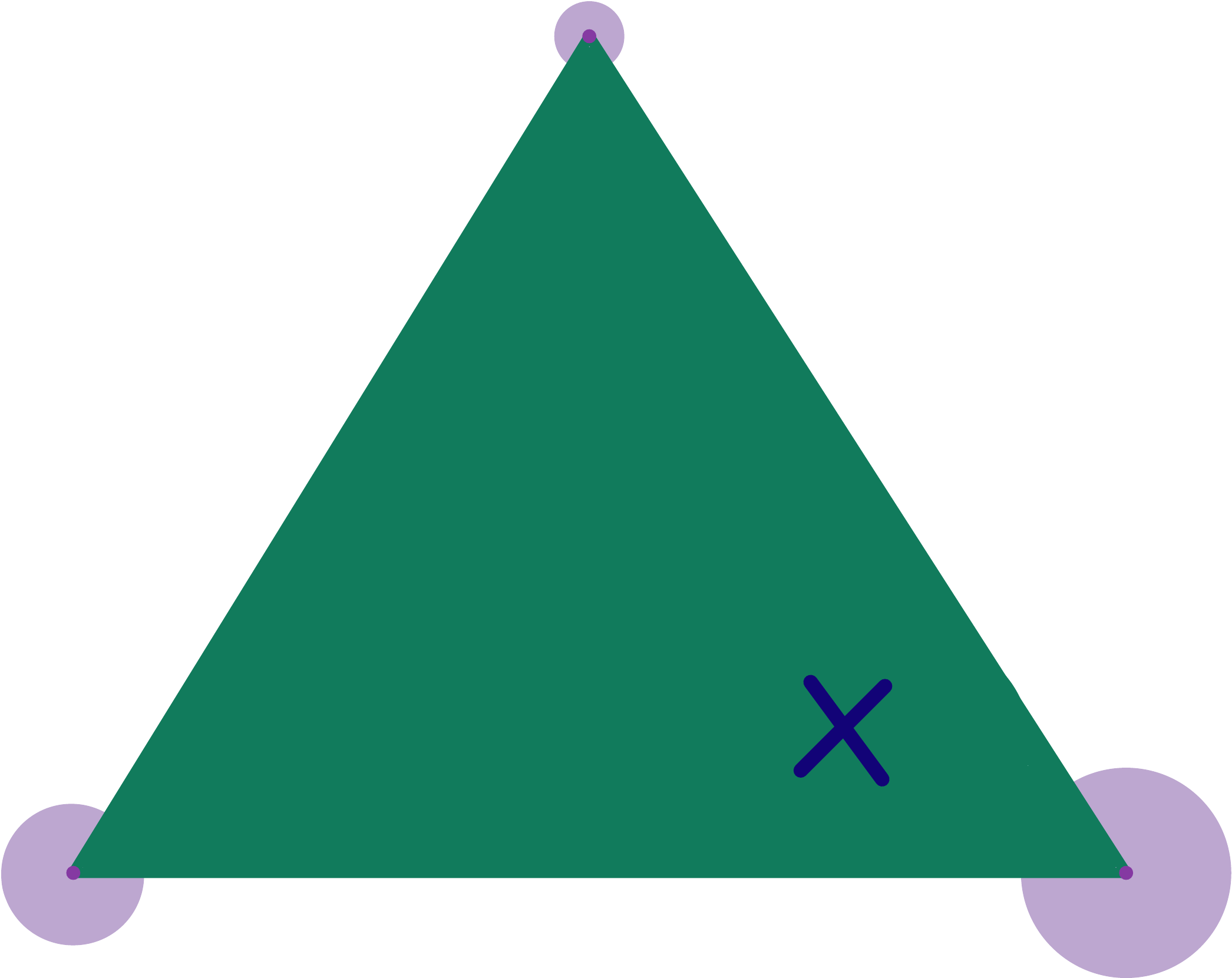}
\captionof{figure}{Convex Constraints}
\label{fig_convex_classical_transformers}
\end{minipage} \hfill
\begin{minipage}[b]{0.5\linewidth}
\centering
\includegraphics[width=0.32 \linewidth]{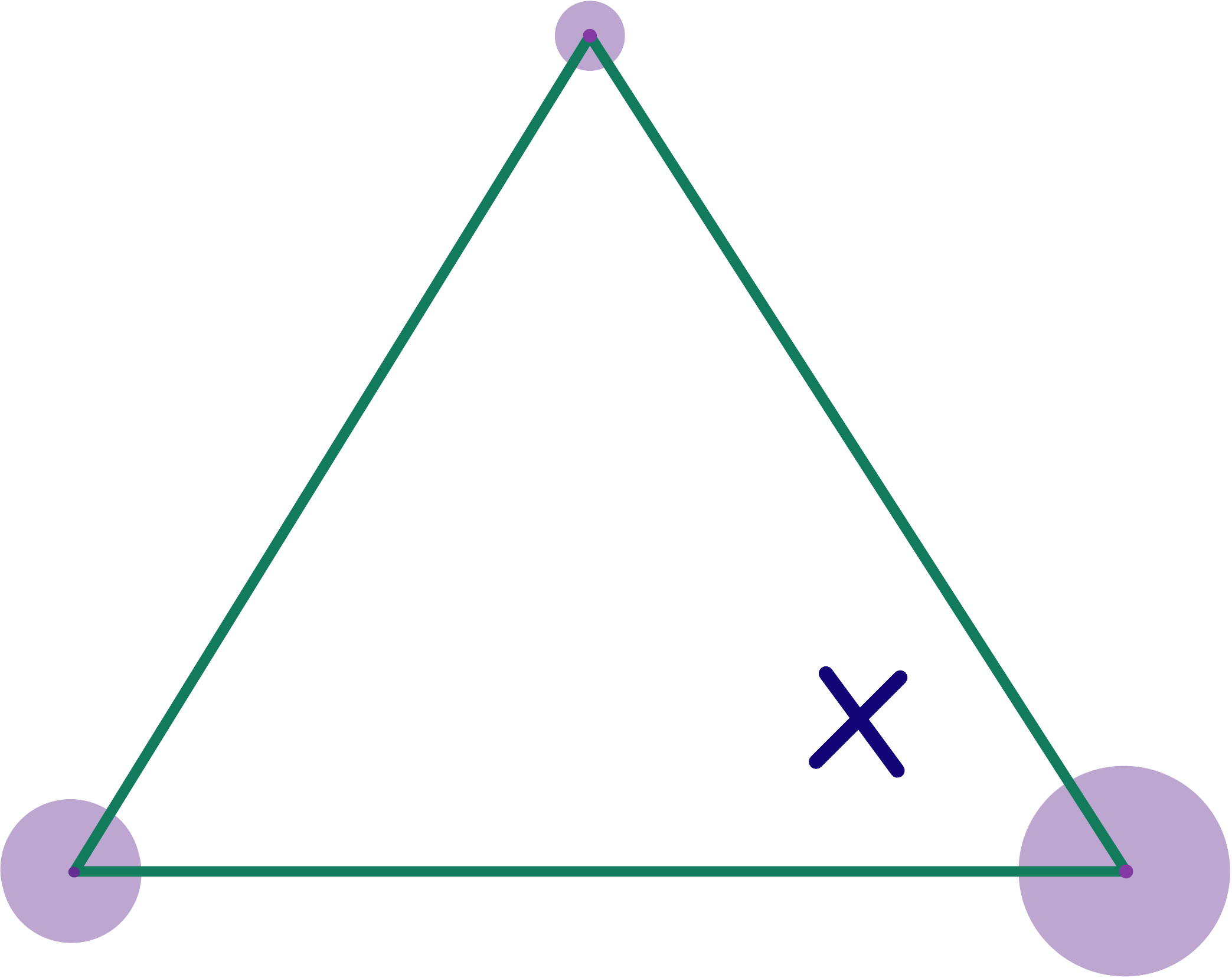}
\captionof{figure}{Non-Convex Constraints}
\label{fig_nonconvex_classical_transformers}
\end{minipage}
\end{figure}
\vspace{-1.5em}
{\color{black}{
In Figures~\ref{fig_convex_classical_transformers} and~\ref{fig_nonconvex_classical_transformers}, $Y$'s columns, i.e. the particles $Y_1,Y_2$, and $Y_3$, are each illustrated by a {\color{darkpurpleheart}{$\bullet$}} at {\color{darkgreen}{the constraint set ($K$)}} vertices.  The {\color{lightpurpleheart}{bubble}} around each each $Y_i$ illustrates the predicted probability, for a given input, that $f(x)$ is nearest to that $Y_i$.  The {\color{blue}{$\boldsymbol{\times}$}} is the transformer's prediction which is, by construction, a convex combination of the $Y_i$ weighted by the aforementioned probabilities and therefore they lie in the {\color{darkgreen}{$K$}} if it is convex (Figure~\ref{fig_convex_classical_transformers}) but not if {\color{darkgreen}{$K$}} is non-convex (Figure~\ref{fig_nonconvex_classical_transformers}).
}}

Naturally, we arrive at the question: 
\textit{How can (i) and (ii) simultaneously hold when $K$ is non-convex?}

Returning to \citet{vaswani2017attention} and using the introduced terminology, we note that the role of the $\operatorname{Softmax}_N$ layer is to rank the importance of the particles $\{Y_n\}_{n=1}^N$ when optimizing $L$, at any given input: the weights $[\operatorname{Softmax}_N(w)]_n$ in~\eqref{eq_definition_attention} can be interpreted as charging their respective \textit{point masses} $\{\delta_{Y_n}\}_{n=1}^N$ with probabilities of being optimal for $L$ (relative to the other particles)%
\footnote{Following \cite{Villani2009optimal}, $\delta_{Y_n}$ is the Borel probability measure on $\rr^m$ assigning full probability to any Borel subset of $\rr^m$ containing the particle $Y_n$ and $0$ otherwise.  }%
.  This suggests the following \textit{probabilistic reinterpretation of attention} {\color{black}{(which we denote by p-attention)}}:
\begin{equation}
    \operatorname{P-attention}(w,Y)
        \triangleq 
    \sum_{n=1}^N [\operatorname{Softmax}_N(w)]_n\delta_{Y_n}
    \label{eq_probabilistic_attention}
    .
\end{equation}
Crudely put, $\operatorname{P-attention}(\cdot,Y)$ \textit{``pays relative attention to the particles''} $Y_1,\dots,Y_n\in \rr^m$.  

A simple computation shows that the mean prediction of our probabilistic attention mechanism, exactly implements ``classical'' $\operatorname{Attention}$ of \cite{vaswani2017attention}, as defined in~\eqref{eq_definition_attention},
\begin{equation}
    \operatorname{Attention}(w,Y)
    =
\mathbb{E}_{X\sim \operatorname{P-attention}(w,Y)}[X]
\label{eq_THE_KEY_IDENTITY___attention_implementation_via_expecation}
,
\end{equation}
where $\mathbb{E}_{X\sim \operatorname{P-attention}(w,Y)}[X]$ denotes the (vector-valued) expectation of a random-vector $X$ distributed according to $\operatorname{P-attention}(w,Y)$.  Hence,~\eqref{eq_probabilistic_attention} is no less general than~\eqref{eq_definition_attention}.  
The advantage of~\eqref{eq_probabilistic_attention} is that, if each particle $Y_n$ belongs to $K$ (even if $K$ is non-convex) then,  any sample drawn from the probability measure $\operatorname{P-attention}(w,Y)$ necessarily belongs to $K$.  
\vspace{-.5em}
\subsection{Qualitative Results: Deep Maximum Theorem}
\label{s_Intro_ss_Qualitative_deep_berge_theorem}
\vspace{-.5em}
Probabilistic attention \eqref{eq_probabilistic_attention} yields the following non-convex generalization of Informal Theorem~\ref{previewtheorem_Convex_Case}.  The result is a \textit{qualitative universal approximation theorem} as well as a deep neural version of the Maximum Theorem%
\footnote{
More precisely, our result is a deep neural version of the measure-theoretic counterpart to Berge's Maximum Theorem; see \citep[(Measurable Maximum Theorem) - Theorem 18.19]{InfiniteHitchhiker2006}.
} \citep{BergeMaximumOG_Articel1963}, which states that under mild regularity conditions, given any well-behaved family of input dependent ``soft constraint sets''  $\{C_x\}_{x\in \rr^n}$ compatible with $K$, there is a measurable function mapping each $x \in \rr^n$ to a minimizer of $L(x,y)$ on $K\cap C_x$.

We use $\www_1$ to denote the Wasserstein-1 distance between probability measures on $K$.  The results also give the flexibility to the user to enforce an input-dependent family of ``soft constraints'' $\{C_x\}_{x\in \rr^n}$ which only need to hold approximately; definitions are provided in Section~\ref{s_Intro_ss_BandNotation}.  
\begin{previewtheorem}[Deep Maximum Theorem: Non-Convex Case]\label{previewtheorem_General}
If the quantities defining~\ref{eq_motivation_1_constrained_approx} are regular, $K$ is a compact set of ``exact constraints'', and $\{C_x\}_{x\in \rr^n}$ a set of ``soft constraints'', then, for any approximation quality $0<\epsilon\leq 1$, there is a deep feedforward network $\hat{f}$ and a matrix $Y$ satisfying:
\begin{enumerate}[nolistsep,leftmargin=2em]
\item[(i)] \textbf{Exact Constraint Satisfaction:} For each $x \in \rr^n$, $\operatorname{P-attention}(\hat{f}(x),Y)$ is supported in $K$,
\item[(ii)] \textbf{Universal Approximation:} 
    $
    \pp(
      \www_1
        (
            \operatorname{P-attention}(\hat{f}(x),Y)
        ,
    \underset{{y^{\star}\in C_x\cap K}}{\operatorname{argmin}}
    \, L(x,y^{\star})
        )
    \leq \epsilon
)\geq 1-\epsilon;
    $
    \end{enumerate}
\vspace{-1em}
where for a probability measure $\pp$ on $\rr^m$ and a $B\subseteq \rr^m$ we define $\www_1(\pp,B)\triangleq \inf_{b\in B}\,\www_1(\pp,\delta_{b})$.
\end{previewtheorem}
\begin{example}[Reduction to Classical Point-to-Set Distance]\label{ex_redux_classical}
In particular, when $\pp$ is a point-mass $\pp=\delta_y$ for some $y\in \rr^m$, then one recovers the familiar \textit{Euclidean distance to the set $B$} via:
$$
\www_1(\delta_y,B)
    \overset{\operatorname{(def)}}{=}
\inf_{b\in B}\, \www_1(\delta_y,\delta_b)
    =
\inf_{b\in B}\,\|y-b\| 
    \overset{\operatorname{(def)}}{=}
\|y-B\|;
$$
where the first and second equality follows from \citep[(5) - page 99]{Villani2009optimal}, and the last equality is the definition of $\|y-B\|$ (as in \citep[Definition 1.1.1]{SetValuedAnalysisAubinFrankowska2009}).
\end{example}
Another  important class of non-convex constraints arising from geometric deep learning where $K$ is a non-Euclidean ball in a Riemannian submanifold of $\rr^m$.  In this broad case, we may extract mean predictions from $\operatorname{P-attention}(\hat{f},Y)$, by applying the \textit{Fr\'{e}chet mean} introduced in \cite{FrechetOGPaperFrechetMeansIntroduced1948}.  Such ``geometric means'' are well-understood theoretically \citep{bhattacharya2003large} and easily handled numerically \cite{JMLR_Manopt_MiolanGuilguiLeBriantetcGeomSTATsPython,lou2020differentiating}.
\subsection{Quantitative Results: Constrained Universal Approximation Theorem}\label{s_Intro_ss_Quantitative_constrained_UAT}
In its current form, the objective function $L$ is too general to derive quantitative approximation rates\footnote{For instance, $L$ can describe anything from a regression, to a clustering problem.}.      
Nevertheless, as with most universal approximation theorems \citep{hornik1989multilayer,pinkus1999approximation,kidger2019universal}, if each soft constraint $C_x$ is set to $\rr^m$ and $L$ quantifies the uniform distance to an \textit{unknown continuous function} $f:\rr^n\rightarrow K$ in the Euclidean sense,
\[
L(x,y)\triangleq \|f(x)-y\|
,
\]
then, Informal Theorem~\ref{previewtheorem_General} reduces to a (qualitative) {\color{darkcerulean}{universal approximation}} for transformer networks with {\color{darkgreen}{exact constraint satisfaction}}.  In fact, this additional structure is enough for us to derive quantitative versions of the aforementioned results.  
We permit ourselves the general situation, where $K$ is contained in an unknown \textit{$d$-dimensional} submanifold (where $d\in \Theta(m^{\frac1{s}})$ for some $s>0$).  Our approximation rates scale favourably in the ratio $s\approx \frac{\log(m)}{\log(d)}$; i.e., we avoid the curse of dimensionality for low-dimensional constraint sets.  This additional structure translates into the familiar encoder-decoder structure deployed in most transformer network implementations.  

\begin{figure}[ht]
    \centering
\begin{minipage}[b]{0.4\linewidth}
\centering
\includegraphics[width=0.4 \linewidth]{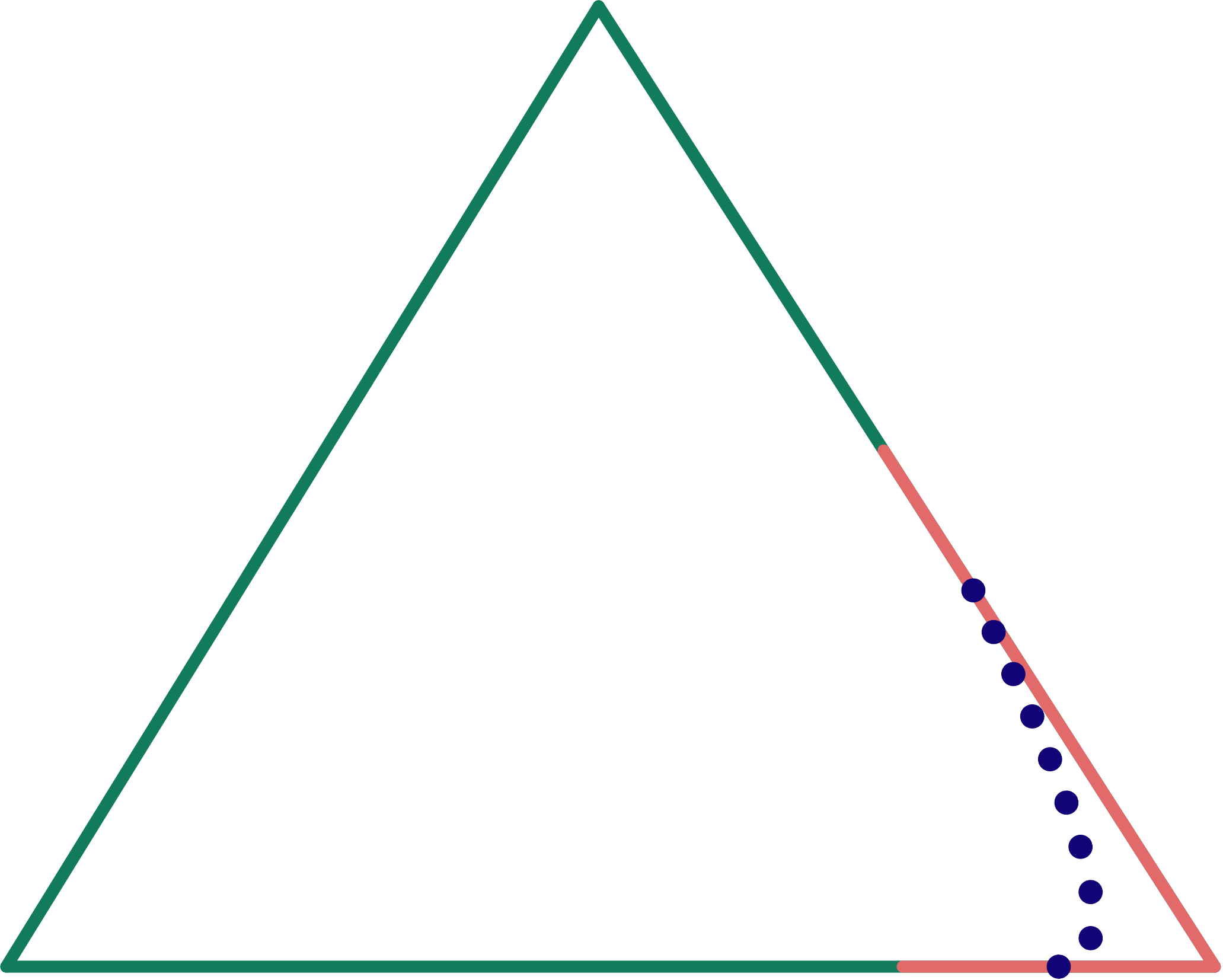}
\captionof{figure}{
{\color{orange}{
    Encoder
}}
: $\approx$ {\color{blue}{$f$}}}
\label{fig_encoder}
\end{minipage} \hfill
\begin{minipage}[b]{0.5\linewidth}
\centering
\includegraphics[width=0.32 \linewidth]{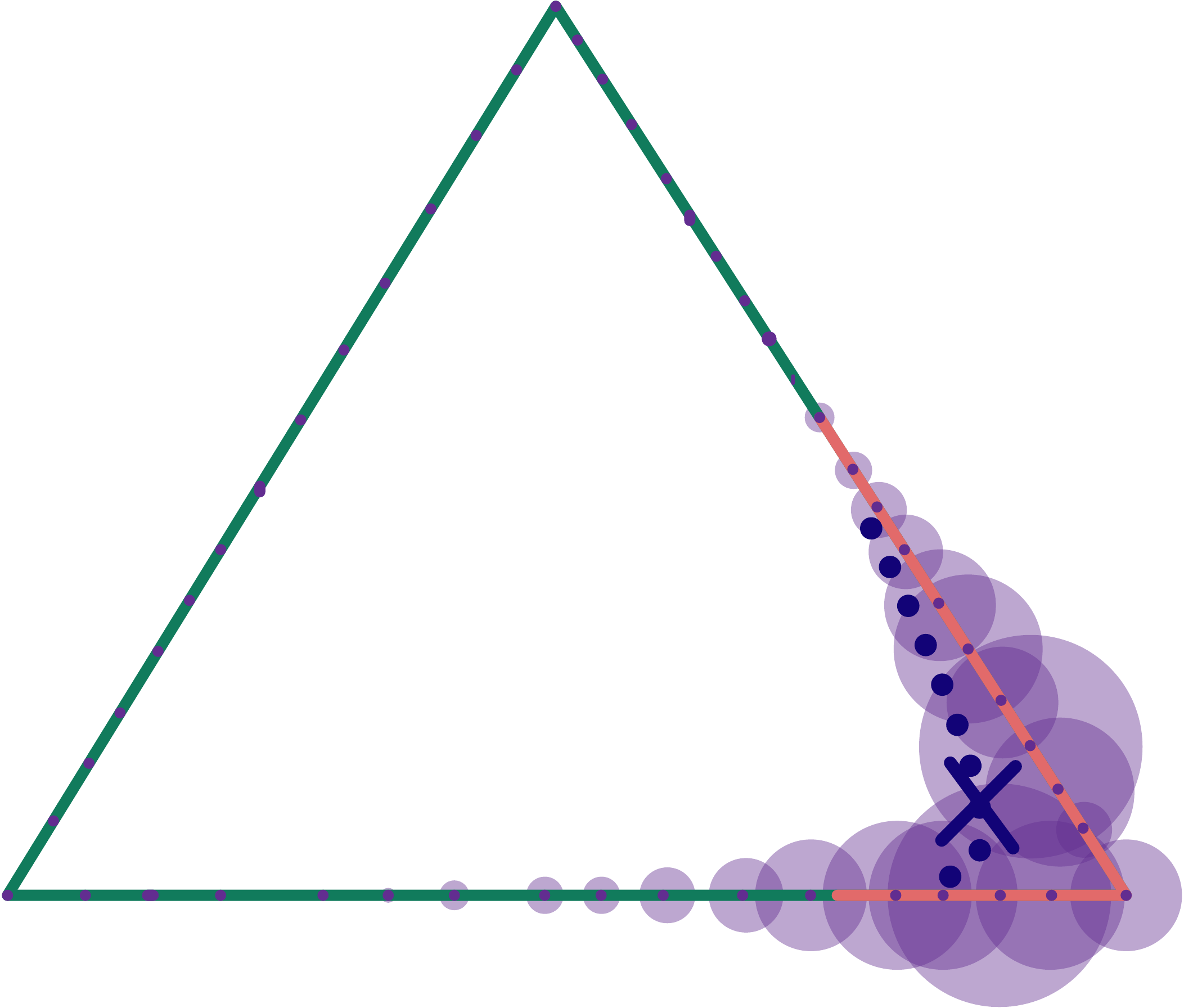}
\captionof{figure}{
{\color{darkpurpleheart}{Decoder}}
: $\approx$ {\color{darkgreen}{Random Projection to $K$}}}
\label{fig_decoder}
\end{minipage}
\end{figure}
\vspace{-1em}
Figure~\ref{fig_encoder} illustrates the {\color{orange}{\textit{encoder network} $\hat{\mathcal{E}}:\rr^n\rightarrow \rr^m$}}, whose role is to perform a (classical) unconstrained approximation of the target function, $f$.  {\color{black}{Since $\hat{\mathcal{E}}$ is a classical feedforward network then its {\color{orange}{approximation}} of the {\color{blue}{target function}} can be arbitrarily close to {\color{darkgreen}{the constraint set $K$}} but it need not lie in it.  The next step is to ``map the {\color{orange}{\textit{encoder network}}}'s output onto {\color{darkgreen}{$K$}} with low distortion.''}} 

\vspace{-.5em}
The role of the {\color{darkpurpleheart}{\textit{decoder network} $\hat{\mathcal{D}}$}} is to correct any constraint violation made by {\color{orange}{\textit{encoder network}}} by ``projecting them back on to {\color{darkgreen}{$K$}}''.  However, such a projection does not exist if {\color{darkgreen}{$K$}} is not convex since there must be more than one closest point in $K$ to some $y\in \rr^m$ \citep{motzkin1935quelques}.  Nevertheless, if the ``projection'' were capable of mapping any $y\in \rr^m$ to \textit{multiple} points on $K$, ranked by their proximity to $y$, then there would be no trouble.  
The {\color{darkpurpleheart}{\textit{decoder network}}} accomplishes precisely this, as illustrated in Figure~\ref{fig_decoder}, where the {\color{darkpurpleheart}{bubbles}} illustrate the probability of any particle in $K$ being closest to ${\color{blue}{y}}$, illustrated by the size of the {\color{darkpurpleheart}{bubbles}} in Figure~\ref{fig_decoder}.   Mathematically,\footnote{$\mathcal{P}_1(K)$ denotes the Wasserstein space on $K$, and is defined formally below.} {\color{darkpurpleheart}{$\hat{\mathcal{D}}:\rr^m\rightarrow \mathcal{P}_1(K)$}} approximates a \textit{(non-affine) random projection}, in the sense of \citet{Ohta2009ExtendingLipandHold,AmbriosioPuglisi2020RandomProjections,Brue2021Extension}; i.e.: a $1$-Lipschitz map $\Pi:\rr^m\rightarrow \mathcal{P}_1(K)$ satisfying the \textit{random projection property}: \textit{for all $y \in K$}
\[
\Pi_y = \delta_y
.
\]
Thus, $\Pi$'s random projection property means that it fixes any output already satisfying the constraint $K$, and its Lipschitz regularity implies that it is stable. Thus, sampling from $\Pi(y_1) $ is similar to sampling from $\Pi(y_2)$ whenever the points $y_1,y_2\in \rr^m$ are near to one another.  
\begin{remark}
\label{remark_NLP_randP}
Random projections are closely tied to the (random) partitions of unity of \cite{LeeNaor2005RandomProjectionsforLipExtensions} (see \citep[Theorem 2.8]{AmbriosioPuglisi2020RandomProjections}).  These  random projections generalize the random projections of \cite{JohnsonLindenstraussLemma1984}, beyond the case where $K$ is affine.  
\end{remark}
\begin{remark}
The special case of random projections onto affine spaces has recently been used when constructing universal neural models \citep{cuchiero2020discretetime,puthawala2020globally}.  
\end{remark}
\vspace{-.5em}
We record the complexity of both the decoder and encoder networks constructed in our quantitative results in Table~\ref{tab_Big_O_Rates}.  Here $A,B,C,D\geq 0$ are constants independent of $\epsilon$ and $k$, where $k \in \nn_+$ is the number of continuous derivatives which $f$ admits (when viewed as a function into $\rr^m$).  
\vspace{-.5em}
\begin{table}[h]
    \centering
    \begin{tabular}{c|cc}
\toprule
        Network & $\hat{\mathcal{E}}$ & $\hat{\mathcal{D}}$
\\
\midrule
        \hline
        Depth & 
        $\mathscr{O}(
            m^{\frac1{s}}(
                    1+\epsilon^{\frac{2n}{3(kn+1)}-\frac{2n}{kn+1}}
                )
        )$
        & 
        $
        \mathscr{O}\left(
        (N^{\frac{3}{2}}(A + 2\epsilon)(4-\epsilon^{-1})^2)^{\frac{2m}{s}}
        \right)
        $
         \\
        Width 
        & $m^{\frac1{s}}(4n+10)$ & $m^{\frac1{s}}+N+2$ \\
        $N$ & 
        - & $
        \mathscr{O}\left(
        (\epsilon^{-1}A + B)^{\frac{m}{2}}
    \right)$ \\
    $Q$ & - & $\mathscr{O}\left(
    \epsilon^{\frac{-m}{s}}
    \right)$\\
\bottomrule
    \end{tabular}
    \vspace{-.5em}
    \caption{Complexity of simple transformer network $\hat{f}=\hat{\mathcal{D}}\circ \hat{\mathcal{E}}$ approximating $f$.}
    \label{tab_Big_O_Rates}
\end{table}
\vspace{-1em}
From Table~\ref{tab_Big_O_Rates}, we see that if $m^{\frac{1}{s}}\ll m$ then, $s>0$ is large; hence, $\epsilon^\frac{m}{s}$, $(1-4\epsilon^{-1})^{\frac{2m}{s}}$, and $N^{\frac{m}{s}}$ are small.  
\vspace{-1em}
\subsection{Notation and Background}\label{s_Intro_ss_BandNotation}
\vspace{-.5em}
\paragraph{Optimal Transportd} 
Given any non-empty subset $K\subseteq \rr^m$, the set of all Borel probability measures $\pp$ on $K$ with a finite mean; i.e.:
$
  \mathbb{E}_{X\sim \pp}[\|X\|]<\infty
$, is denoted by $\mathcal{P}_1(K)$.   \textit{Wasserstein distance} $\www_1$ is defined for any $\pp,\qq\in \mathcal{P}_1(K)$ by the minimal energy needed to transport all mass from $\pp$ to $\qq$.  Following \cite{Villani2009optimal}, $\www_1(\pp,\qq)$ is defined by:
\[
\www_1(\pp,\qq)\triangleq \inf_{\pi} \mathbb{E}_{(X_1,X_2)\sim\pi}[\|X_1-X_2\|],
\]
where the infimum is taken over all Borel probability measures $\pi$ on $K^2$ with marginals $\pp$ and $\qq$.  The metric space $(\ppp{K},\www_1)$ is named the \textit{Wasserstein space over $K$}; we abbreviate it by $\ppp{K}$.  

\vspace{-1em}
\paragraph{Smooth Function Spaces} 
The set of real-valued continuous functions on $\rr^n$ is denoted by $C(\rr^n)$.  
Let $k\in \nn_+$ and $\xxx\subseteq [0,1]^n$ be non-empty.  The set of functions $f:\xxx\rightarrow K$ for which there is a $k$-times continuously differentiable $\boldsymbol{f}:\rr^n\rightarrow \rr^m$ extending $f$; i.e.: $\boldsymbol{f}|_{\xxx}=f$, is denoted by $C_{tr}^k(\xxx,K)$.  Our interest in $C_{tr}^k(\xxx,K)$ does not stem from the fact that it contains all smooth functions mapping $[0,1]^n$ to $K$, but rather that it allows us to speak about the uniform approximation of discontinuous $K$-valued functions on regions in $[0,1]^n$ where they are ``regular''.  This is noteworthy for pathological constraint sets, such as integer constraints%
\footnote{
For example, there is no non-constant continuous function $f:[0,1]\rightarrow \zz$.  However, for any $\lambda\in (0,\frac1{2})$ and any $y_1,y_2\in \zz$, $f=y_1 I_{[0,\lambda]} + y_2I_{[\lambda + \frac1{2},1]} $ belongs to $C_{tr}^k([0,\lambda]\cup [\lambda + \frac1{2},1],\zz)$ for each $k\in \nn_+$.  
}%
.
For details on $C_{tr}^k(\xxx,K)$, see \citep{Brudnyix2_Monograph_WhitneyProblem_and_Traceclasses_Vol1_2021,Brudnyix2_Monograph_WhitneyProblem_and_Traceclasses_Vol2_2021} and the extension theorems of \cite{ExtensionTheorem_Whitney_1934,Fefferman2005_Breakthrough_WhitneyTheoremSolved_AnnMath}.   

\vspace{-1em}
\paragraph{Neural Networks}
It has recently been observed that deep feedforward networks with multiple activation functions, or more generally parametric families of activation functions, achieved significantly more efficient approximation rates than classical feedforward networks with a single activation function
\citep{Yarotsky2020NEuripsPhaseDiagram,Yarotsky2021SuperPressiveICML,Shen2021ThreeLayersSuperExpressive,ShenYangZhang2021}.  Practically deployed examples of parametric activation functions are the PReLU activation function of \cite{PreLU_he2015delving}, the Sigmoid-weighted Linear Unit (SiLU) of \cite{elfwing2018sigmoid}, and the Swish activation function of \cite{Swish2018ICLR}.  We also observe a similar phenomenon, and therefore our quantitative results consider deep feedforward networks whose activation functions belongs to a  $1$-parameter family $\sigma_{\star}\triangleq \{\sigma_t\}_{t\in [0,1]}\subseteq C(\rr)$.  The set of all such networks is denoted by $\NN[n,N][\sigma_{\star}]$ and it includes all $\hat{f}:\rr^n\rightarrow \rr^N$ with iterative representation:
\begin{equation}
        \hat{f}(x)\triangleq A^{(J)} x^{(J)}
,
\qquad
x^{(j+1)}_{i_j}\triangleq \sigma_{t_{i_j}}((A^{(j)}x)_{i_j} + b^{(j)}_{i_j})
,
\qquad
x^{(0)}\triangleq x
\label{eq_definition_deepNN_iterative_representation}
,
\end{equation}
where $x\in \rr^n$, $j=1,\dots,J-1$, each $A^{(j)}$ is a $d_j\times d_{j+1}$-matrix, each $b^{(j)}\in \rr^{d_{j+1}}$, $d_{J+1}=N$, $d_1=0$, $t_{1,1},\dots,t_{J,N_J}\in [0,1]$, for each $j$.  
The integer $J$ is $\hat{f}$'s \textit{depth} and $\underset{j=1,\dots,{J+1}}{\max} d_j$ is $\hat{f}$'s \textit{width}.  
\vspace{-.5em}
\begin{example}[Networks with Untrainable Nonlinearity]\label{remark_nonparametric_activation}
Denote $\sigma\triangleq \sigma_0$.  The subset of classical feedforward networks consisting of all $\hat{f}\in \NN[n,N][\sigma_{\star}]$ with each $\sigma_{t_{i_j}}=\sigma$ in~\eqref{eq_definition_deepNN_iterative_representation} is denoted $\NN[n,N]$.  
\end{example}
\vspace{-.5em}
It is approximation theoretically advantageous to generalize the proposed definition of probabilistic attention in the introduction~\eqref{eq_probabilistic_attention} by replacing $Y$ with a $3$-dimensional array (elementary $3$-tensor).  
\begin{definition}[Probabilistic Attention]
Let $N,Q,m\in \nn_+$, and $Y$ be an $N\times Q\times m$-array with $Y_{n,q}\in K$ for $n=1,\dots,N$, $q=1,\dots,Q$.  Probabilistic attention is the function:
\[
\rr^n\ni w\mapsto \operatorname{P-attention}(w,Y)\triangleq 
\frac1{Q}\sum_{n=1}^N\sum_{q=1}^{Q} \operatorname{Softmax}_N(w)_n \delta_{Y_{n,q}}
\in \ppp{K}
.
\]
\end{definition}
\vspace{-1em}
If $Y$ is an $N\times m$-matrix, as in~\eqref{eq_probabilistic_attention}, then we identify $Y$ as the $N\times m\times 1$-array in the obvious manner.  
\vspace{-2em}
\paragraph{Set-Valued Analysis:} A family of  non-empty subsets $\{C_x\}_{x\in \rr^n}$ of $K$ is said to be a \textit{weakly measurable correspondence}, denoted $C:\rr^n \rightrightarrows \rr^m$, if for every open subset%
\footnote{
Since $K$ is equipped with its subspace topology, then an open subset $U$ of $K$ is any subset of $\rr^m$ of the form $U=U_1\cap K$ where $U_1$ is an open subset of $\rr^m$ (see \citep[Lemma 16.1]{munkres2014topology} for further details).
}%
$U\subseteq K$, $\{
x\in \rr^n:\, C_x\cap U\neq \emptyset
\}$ is a non-empty Borel subset of $\rr^n$ \citep[pages 557, 592]{InfiniteHitchhiker2006}.  
\vspace{-1em}
\section{Main Results}
\label{s_Main_Results}
\vspace{-1em}
We now present our main results in detail.  All proofs are relegated to the paper's appendix.  
\vspace{-1em}
\subsection{Qualitative Approximation: Deep Maximum Theorem}\label{s_Main_ss_Deep_Maximum_Principle}
Our main qualitative result is the following deep neural version of \cite{BergeMaximumOG_Articel1963}'s Maximum Theorem where, the measurable selector is approximately implemented by a probabilistic transformer network.  We first present the general qualitative result which gives a concrete description of a measurable selector of~\eqref{eq_motivation_1_constrained_approx}, with high-probability, which has the key property that all its predictions satisfy the required constraints defined by $K$.  
\begin{assumption}[{\cite{kidger2019universal}}]\label{ass_KL}
$\sigma:\rr\rightarrow\rr$ is continuous, $\sigma$ is differentiable at some $x_0\in \rr$, and its derivative satisfies $\sigma'(x_0)\neq 0$.
\end{assumption}
\begin{theorem}[Deep Maximum Theorem]\label{theorem_DeepBMT}
Let $\sigma$ satisfy Assumption~\ref{ass_KL}.  
Let $K\subseteq \rr^n$ be a non-empty compact set, $C:\rr^n\rightrightarrows \rr^m$ be a weakly-measurable correspondence with closed values such that $C_x\cap K\neq \emptyset$ for each $x \in \rr^n$, $L\in C(\rr^m)$, and $\mathbb{P}$ be a Borel probability measure on $\rr^n$.  For each $0< \epsilon \leq 1$, there is an $N\in\nn_+$, an $\hat{f}\in \NN[n,N]$ of width at most $2+n+N$, and an $N\times m$-matrix $Y$ such that:
\begin{equation}
\smash{
        \hat{F}:\mathbb{R}^n\ni x\mapsto 
        \operatorname{P-attention}\left(
        \hat{f}(x),Y
        \right)
        \in \mathcal{P}_1(\rr^m)
    ,
}
\label{eq_theorem_deepBMT_Model_form}
\end{equation}
\vspace{-.25em}
satisfies the following:
\begin{enumerate}[nolistsep,leftmargin=2em]
    \item[(i)] \textbf{Exact Constrain Satisfaction:} $\cup_{x\in \rr^n}\, \operatorname{supp}(\hat{F}(x))\subseteq K$,
    \item[(ii)] \textbf{Probably Approximately Optimality:} There is a compact ${\xxx_{\epsilon}}\subseteq \rr^n$ satisfying:
    \begin{enumerate}[nolistsep,leftmargin=2em]
    \item[(a)] $\underset{x\in {\xxx_{\epsilon}}}{\max}\,
    \www_1(
        \hat{F}(x)
            ,
        \underset{{y\in C_x\cap K}}{\operatorname{argmin}}
    \, L(x,y)
    )
    \leq \epsilon,
    $
    \item[(b)] $1-\pp({\xxx_{\epsilon}})
    \leq \epsilon.$
    \end{enumerate}
\end{enumerate}
\end{theorem}
\vspace{-.5em}
Theorem~\ref{theorem_DeepBMT} implies that for any random field $(Y^x)_{x \in \rrn}$ on $\rrm$ (i.e. a family of $\rrm$-valued random vectors indexed by $\rr^n$) with $Y^x\sim \hat{F}(x)$:
1. samples drawn from $Y^x$ are in $K$ (by (i)) 
and 
2. samples drawn from each $Y^x$ are near to the optimality set $\operatorname{argmin}_{y\in C_x\cap K}L(x,y)$ (by (ii)).
\begin{corollary}[{$\hat{F}$'s Mean Prediction}]\label{cor_simplified_version_of_theorem_DeepBMT}
Assume the setting of Theorem~\ref{theorem_DeepBMT}. Let $\{Y^x\}_{x\in \rr^n}$ be a $K$-valued random field with $Y^x\sim \hat{F}(x)$ for each $x \in \rr^n$ then, $1-\pp({\xxx_{\epsilon}})\leq \epsilon$ and
\[
    \underset{x\in {\xxx_{\epsilon}}}{\max}\,
      \mathbb{E}_{
      }
        [
        \|
            Y^x
    -
    \underset{{y^{\star}\in C_x\cap K}}{\operatorname{argmin}}
    \, L(x,y^{\star})
        \|
        ]
    \leq  \epsilon
    .
\]
\end{corollary}
\vspace{-1em}
Appendix~\ref{a_ss_Additional_Results} contains additional consequences of the Deep Maximum Theorem, such as the special case of classical transformers when $K$ is convex. 
Next, we complement our qualitative results by their quantitative analogues, within the context of \textit{universal approximation} under \textit{constraints}.  
\vspace{-1em}
\subsection{Quantitative Approximation: Constrained Universal Approximation}\label{s_Main_ss_Quantitative_Results}
\vspace{-.5em}
In order to derive a \textit{quantitative} constrained universal approximation theorem, we require the loss function to be tied to the Euclidean norm in the following manner.  
\begin{assumption}[Norm-Controllable Loss]\label{ass_normcontrolled_loss}
There is a continuous $f:\rr^n\rightarrow \rr^m$ with $f(\rr^n)\subseteq K$ and a continuous $l:[0,\infty)\rightarrow [0,\infty)$ with $l(0)=0$, satisfying:
$
L(x,y)\leq l(\|f(x)-y\|)
.
$
\end{assumption}
\vspace{-.5em}
Just as with transformer networks, our ``constrained universal approximation theorem'' approximates a suitably regular function $f:\rr^n\rightarrow K\subseteq \rr^m$ while exactly respecting the constraints $K$ by implementing an \textit{encoder-decoder} network architecture.  Thus, our model is a composition of an \textit{encoder network} $\hat{\mathcal{E}}:\rr^n\rightarrow \rr^d$ whose role is to approximate $f$ in a classical ``unconstrained fashion'' and a \textit{decoder network} (with probabilistic attention layers at its output) $\hat{\mathcal{D}}:\rr^d\rightarrow \mathcal{P}_1(K)$ whose role is to enforce the constraints $K$ while preserving the approximation performed by $\hat{\mathcal{E}}$, where $d\lll m$.

To take advantage of the encoder-decoder framework present in most transformer networks,  we formalize what is often called a \textit{``latent low-dimensional manifold''} hypothesis.  Briefly, this means that, the hard constraints in set $K$ are \textit{contained} in a ``low dimensional'' subspace.  
\begin{assumption}[Low-Dimensional Manifold]\label{ass_low_dimensional_mainfold}
There is an $0<s$ and a smooth bijection $\Phi$ from $\rr^n$ to itself with smooth inverse\footnote{
Here smooth means that $\Phi$ is continuously differentiable any number of times.  NB, smooth bijections with smooth inverses are called \textit{diffeomorphisms} in the differential geometry and differential topology literature.  
}, such that
$\Phi(K)\subseteq \rr^d$; where $2\leq d$ and $d \in \Theta(m^{\frac1{s}})$.
\end{assumption}
Assumption~\ref{ass_low_dimensional_mainfold} does not postulate that $K$ is itself a single-chart low-dimensional manifold, or even a manifold.  Rather, $K$ need only be contained in a low-dimensional manifold.  For the fast rates we use activation functions generalizing the swish function \citep{Swish2018ICLR} as follows. 
\begin{assumption}[Swish-Like Activation Function]\label{ass_PWLA}
The map $\sigma:[0,1]\times \rr\ni (\alpha,t)\mapsto \sigma_{\alpha}(t)\in \rr$ is continuous; $\sigma_0$ is non-affine and piecewise-linear, and $\sigma_1$ is smooth\footnote{
Following \cite{jost2008riemannian}, a function $\sigma:\rr\rightarrow \rr$ is called smooth (or class $C^{\infty}$) if $\partial^k \sigma$ exists for each $k \in \nn_+$.  
} and non-polynomial.
\end{assumption}
\vspace{-0.5em}
\begin{theorem}[Constrained Universal Approximation]\label{thrm_quantitativer_version_NonConvex}
Let $k\in \nn_+$ and $\xxx\subseteq [0,1]^n$ be non-empty.  Suppose that $\sigma$ satisfies~\ref{ass_PWLA}, $L$ satisfies Assumption~\ref{ass_normcontrolled_loss}, $K\subseteq \rr^n$ is non-empty, compact and satisfies Assumption~\ref{ass_low_dimensional_mainfold}.  For any $f\in C_{tr}^k(\xxx,K)$, every constraining quality $\epsilon_K>0$, and every approximation error $\epsilon_f>0$, there exist $N,Q\in \nn_+$, an encoder $\hat{\mathcal{E}}\in \NN[n,d]$, and a decoder:
\begin{equation}
\hat{\mathcal{D}}:\rr^d\ni x \mapsto \sum_{k=1}^N
\operatorname{P-attention}\left(
    \hat{D}(x)
,
Y
 \right)\in \mathcal{P}_1(K)
\label{eq_decoder_network_definition_for_parameter_count}
\end{equation}
where $\hat{D}\in \NN[d,N]$ and $Y$ is an $N\times Q\times m$-array with $Y_{1,1},\dots,Y_{N,Q}\in K$ such that:
\begin{enumerate}[nolistsep,leftmargin=2em]
    \item[(i)] \textbf{Exact Constrain Satisfaction:} For each $x \in \rr^n$:
    $
    \operatorname{supp}(\hat{\mathcal{D}}\circ \hat{\mathcal{E}}(x))\subseteq K
    ,
    $
    \item[(ii)] \textbf{Universal Approximation:} The estimate holds
    \footnote{
        In fact, we actually prove that the slightly stronger statement: $
        \sup_{x \in [0,1]^n}\,
            \www_1\left(
            \hat{\mathcal{D}}\circ \hat{\mathcal{E}}(x)
                    ,
                \delta_{f(x)}
            \right)
            \leq 
            \epsilon_K
            +
            k \operatorname{Lip}(\Phi^{-1})d
            \epsilon_f
        $.  Both formulations align when $l$ has a unique minimum at $0$, as is the case when $L(x,y)=\|f(x)-y\|_{\star}$ and $\|\cdot\|_{\star}$ is any norm on $\rr^m$.
    }
    :
\[
\sup_{x \in [0,1]^n}\,
    \www_1
    (
    \hat{\mathcal{D}}\circ \hat{\mathcal{E}}(x)
            ,
        \underset{y\in K}{\operatorname{argmin}}\,
        L(x,y)
    )
    \leq 
    \epsilon_K
    +
    k \operatorname{Lip}(\Phi^{-1})d
    \epsilon_f
    ;
\]
where,
$0<k$ is an absolute constant independent of $n$, $m$, $d$,
$f$, and of $\epsilon$ and $\operatorname{Lip}(\Phi^{-1})$ denotes the Lipschitz constant of $\Phi^{-1}$ on the compact set $\{z\in \rr^d:\, \|z-\Phi(K)\|\leq \epsilon_f\}$.  
\end{enumerate}
Furthermore, the ``complexities'' of $\hat{\mathcal{D}}$ and $\hat{\mathcal{E}}$ are recorded in Table%
\footnote{
Explicit constants are recorded in 
Table~\ref{tab_model_complexities} within the paper's appendix; there, $\epsilon_K$ and $\epsilon_f$ may differ.
}
~\ref{tab_Big_O_Rates} for $\frac{\epsilon}{2}=\epsilon_k=\epsilon_f$.
\end{theorem}
\vspace{-.75em}
In practice, we can only \textit{sample} from each measure $\hat{\mathcal{D}}\circ \hat{\mathcal{E}}(x)$.  In this case, we may ask how the typical sample drawn from a random-vector $Y^x$ distributed according to our learned measure $\hat{\mathcal{D}}\circ \hat{\mathcal{E}}(x)$ performs when minimizing $L(x,y)$.  The next result relates the estimates in Theorem~\ref{thrm_quantitativer_version_NonConvex} (ii) to the typical (in $Y^x$) worst-case (in $x$) gap between a sample from $Y^x$ and $f(x)$, as quantified by $L(x,\cdot)$. 
\begin{corollary}[{Average Worst-Case Loss}]\label{cor_random_uniform_UAT}
Assume the setting of Theorem~\ref{thrm_quantitativer_version_NonConvex} and suppose that the ``modulus'' $l$ in Assumption~\ref{ass_normcontrolled_loss} is strictly increasing and concave.  Let $\hat{\mathcal{D}}$ and $\hat{\mathcal{E}}$ be as in Theorem~\ref{thrm_quantitativer_version_NonConvex} and let $\{Y^x\}_{x\in \xxx}$ be an $\rr^m$-valued random field with $Y^x\sim  \hat{\mathcal{D}}\circ \hat{\mathcal{E}}(x)$, for each $x \in \rr^n$.  
Then:
\[
    \underset{x\in \xxx}{\max}\,
      \mathbb{E}_{
      Y^x \sim \hat{\mathcal{D}}\circ \hat{\mathcal{E}}(x)
      }
        \left[
        L(x,Y^x)
    \right]
        \leq 
    l\left(
         \epsilon_K
        +
        k \operatorname{Lip}(\Phi^{-1})d
        \epsilon_f
    \right)
    .
\]
\end{corollary}
\vspace{-1em}
Corollary~\ref{cor_random_uniform_UAT} quantifies the expected performance of a sample from our probabilistic transformer model, as expressed by $L$, whereas Theorem~\ref{thrm_quantitativer_version_NonConvex} (ii) quantifies the difference from the transformer's prediction to the optimal prediction value.  Next, we consider implications of our main results.  
\vspace{-1em}
\subsection{Applications}\label{s_Applications}
\vspace{-.5em}
We apply our theory to obtain a {\color{darkcerulean}{universal approximation}} theorem for classical transformer networks with {\color{darkgreen}{exact convex constraint satisfaction}} and to derive a version of the non-Euclidean {\color{darkcerulean}{universal approximation}} theorems of \cite{kratsios2020non,kratsios2021_GDL} for {\color{darkgreen}{Riemannian-manifold}} valued functions which does not need explicit charts.  As with most quantitative (uniform) universal approximation theorems \citep{Yarotski,kidger2019universal,Shen2021ThreeLayersSuperExpressive}, we henceforth consider $L(x,y)=\|f(x)-y\|$.  We also fix $f\in C_{tr}^k([0,1]^n,K)$.
\vspace{-.5em}
\subsubsection{Transformers are Convex-Constrained Universal Approximators}
\vspace{-.5em}
\label{s_Applications_ss_Convex_Optim}
We return to the familiar transformer networks of \cite{vaswani2017attention}.  The next result shows that transformer networks can balance {\color{darkcerulean}{universal approximation}} and {\color{darkgreen}{exact convex constraint satisfaction}}.  
This is because when $K$ is convex, then the mean of the random field $\{Y^x\}_{x\in \rr^n}$ of Corollary~\ref{cor_simplified_version_of_theorem_DeepBMT} must belong to $K$.  Consequently, the identity~\eqref{eq_THE_KEY_IDENTITY___attention_implementation_via_expecation} implies that $\operatorname{Attention}(\hat{\mathcal{D}}\circ\hat{\mathcal{E}}(\cdot),Y)\approx f$.
\begin{corollary}[Constrained Universal Approximation: Convex Constraints]\label{cor_convex_case}
Consider the setting and notation of Corollary~\ref{cor_random_uniform_UAT}.  Suppose that $K$ is convex and let $L(x,y)=\|f(x)-y\|$.  Then:
\begin{equation}
    \smash{
        \rr^n\ni x \mapsto \mathbb{E}[Y^x] = \operatorname{Attention}(\hat{\mathcal{D}}\circ\hat{\mathcal{E}}(x),Y) \in K
        ;
    }
\label{cor_Deep_BMT_Convex_Case}
\end{equation}
\begin{enumerate}[nolistsep,leftmargin=2em]
  \item[(i)]\textbf{Exact Constraint Satisfaction:} $\mathbb{E}_{Y^x\sim \hat{\mathcal{D}}\circ \hat{\mathcal{E}}(x)}[Y^x]\in K$, for each $x\in \rr^n$,
    \item[(ii)]\textbf{Universal Approximation:}
$\sup_{[0,1]^n}\, \|f(x)-\mathbb{E}_{Y^x\sim \hat{\mathcal{D}}\circ \hat{\mathcal{E}}(x)}[Y^x]\|<
\epsilon_K
    +
    k d
    \epsilon_f
    .
$
\end{enumerate}
\vspace{-.5em}
The ``complexities'' of the networks $\hat{\mathcal{D}}$ and $\hat{\mathcal{E}}$ are recorded in Table
\footnote{
Explicit constants are recorded in 
Table~\ref{tab_model_complexities} within the paper's appendix; there, $\epsilon_K$ and $\epsilon_f$ may differ.
}
~\ref{tab_Big_O_Rates} for $\frac{\epsilon}{2}=\epsilon_k=\epsilon_f$.
\end{corollary}
\vspace{-1em}
\subsubsection{Chart-Free Riemannian Manifold-Valued Universal Approximation}
\label{s_Applications_ss_GDL}
\vspace{-.5em}
We explore how additional non-convex structure of the constraint set $K$ can be encoded by the \textit{probabilistic transformer networks} of Theorems~\ref{theorem_DeepBMT} and~\ref{thrm_quantitativer_version_NonConvex} and be used to build new types of (deterministic) transformer networks.  These results highlight that the standard transformer networks of~\eqref{cor_Deep_BMT_Convex_Case} are specialized for convex constraints and that by instead using an intrinsic variant of expectation, we build can new types of ``geometric transformer networks'' customized to $K$'s geometry.  This section makes use of Riemannian geometry; for an overview see \cite{jost2008riemannian}.

Let $(M,g)$ be a connected $d$-dimensional Riemannian submanifold of $\rr^m$ with distance function by $d_g$. We only require the following mild assumption introduced in \cite{Afsari_RiemannianCentersOfMassAMS2011}.  We recall that the \textit{injectivity radius} at $y_0$, denoted by $\operatorname{inf}_g(y_0)$,  (see \citep[Definition 1.4.6]{jost2008riemannian}) is the minimum length of a \textit{geodesic} (or minimal length curve) in $M$ with starting point $y_0$.  We also recall that the \textit{sectional curvature} (see \citep[Definition 4.3.2]{jost2008riemannian} for a formal statement) quantifies the curvature of $(M,g)$ as compared the geometry of its flat counterpart $\rr^d$.  We focus on a broad class of non-convex constrains, namely \textit{geodesically convex constraints}, which generalize convex constraint and have received recent attention in the optimization literature \citep{SurvitCOLT2016_GeodesicConvexityOptimization,GeodesicConvexityOptimizationAcceleration}. 
\begin{assumption}[Geodesically Convex Constraints]\label{ass_geodesic_general}
The Riemannian manifold $(M,g)$ is connected, it is complete as a metric space, and all its sectional curvatures of $(M,g)$ are all bounded above by a constant $C\geq 0$.  The non-empty constrain set $K$ satisfies:
\vspace{-.25em}
\begin{enumerate}[nolistsep,leftmargin=2em]
    \item $K$ is contained in the geodesic ball $B(y_0,\rho)\triangleq \{
y \in M:\, d_g(y_0,y)< \rho
\}$ for some point $y_0\in M$ and some radius $\rho$ satisfying\footnote{
Following \cite{Afsari_RiemannianCentersOfMassAMS2011}, we make the convention that if $C\leq 0$ then, $\frac1{\sqrt{C}}$ is interpreted as $\infty$.
}:
$
0 <\rho < 2^{-1}\min\{
    \operatorname{inj}_{g}(y_0)
        ,
    \frac{\pi}{\sqrt{C}}
\}
,
$
\item For each $y_0,y_1\in K$ there exists a unique geodesic $\gamma:[0,1]\rightarrow K$ joining $y_0$ to $y_1$.
\end{enumerate}
\end{assumption}
\vspace{-.5em}
Our latent probabilistic representation grants us the flexibility of replacing the usual ``extrinsic mean'' used in~\eqref{cor_Deep_BMT_Convex_Case} to extract deterministic predictions from our probabilistic transformer networks via an additional \textit{Fr\'{e}chet mean} layer at their readout.  This intrinsic notion of a mean, was introduced independently in \cite{FrechetOGPaperFrechetMeansIntroduced1948} and in \cite{KarcherOGPaperKarcherMeans1977}, and is defined on any $\mathbb{P}\in \mathcal{P}_1(K)$ by:
\vspace{-.5em}
\begin{equation}
    \bar{\mathbb{P}}\triangleq 
    \underset{
        k \in K
    }{
        \operatorname{argmin}
        }\,
    \int\, d_g^2(k,u)\, \mathbb{P}(du)
.
\label{eq_frechet_mean_def}
\end{equation}

\vspace{-1em}
With this ``geometric readout layer'' added to our model, we obtain the following variants of our main results in this non-convex, but geometrically regular, setting.
\begin{corollary}[Constrained Universal Approximation: Riemannian Case]\label{cor_nonconvex_geometric_case}
Consider the setting and notation of Corollary~\ref{cor_random_uniform_UAT}.  Let $L(x,y)=\|f(x)-y\|$.  If Assumption~\ref{ass_geodesic_general} holds then:
\begin{equation}
    \rr^n\ni x \mapsto  
    \overline{
        \hat{\mathcal{D}}\circ
        \hat{\mathcal{E}}(x)
    }
    \in K
    \label{cor_Deep_BMT_Non_Convex_and_Geometric_Case}
    ,
\end{equation}
is a well-defined Lipschitz-continuous function, and the following hold:
\begin{enumerate}[nolistsep,leftmargin=2em]
  \item[(i)]\textbf{Exact Constraint Satisfaction:} $
  \overline{ 
    \hat{\mathcal{D}}\circ \hat{\mathcal{E}}(x)
  }
  \in K$, for each $x\in \xxx$,
    \item[(ii)]\textbf{Universal Approximation:}
$\sup_{\xxx}\, d_g(f(x)
    ,
 \overline{ 
    \hat{\mathcal{D}}\circ \hat{\mathcal{E}}(x)
  }
)<
\epsilon_K
    +
    k d
    \epsilon_f
    .
$
\end{enumerate}
The ``complexities'' of $\hat{\mathcal{D}}$ and $\hat{\mathcal{E}}$ are recorded in Table%
\footnote{
Explicit constants are recorded in 
Table~\ref{tab_model_complexities} within the paper's appendix; there, $\epsilon_K$ and $\epsilon_f$ may differ.
}
~\ref{tab_Big_O_Rates} for $\frac{\epsilon}{2}=\epsilon_k=\epsilon_f$.
\end{corollary}
\vspace{-1.5em}
\section{Discussion}\label{s_Discussion}
\vspace{-1em}
In this paper, we derived the first \textit{constrained universal approximation} theorems using probabilistic reformation of \cite{vaswani2017attention}'s transformer networks.  The results assumed both a quantitative form (Theorem~\ref{thrm_quantitativer_version_NonConvex}) and a qualitative form in the more general case of an arbitrary loss functions $L$ and additional compatible soft constraints in (Theorem~\ref{theorem_DeepBMT}).  Our results provide (generic) direction to end-users designing deep learning models processing non-vectorial structures and constraints.  

As this is the first approximation theoretic result in this direction, there are naturally as many questions raised as have been answered.  In particular, it is natural to ask: ``\textit{Are the probabilistic transformer networks trainable in practice; especially when $K$ is non-convex%
?''}.  In Appendix~\ref{a_Training_the_Prob_Transformer}, we show that the answer is indeed: \textit{``Yes!''}, by proposing a training algorithm in that direction and showing that we outperform an MLP model and a classical transformer network in terms of a joint MSE and distance to the constraint set.  The evaluation is performed on a large number of randomly generated experiments, whose objective is to reduce the MSE to a randomly generated function mapping a high-dimensional Euclidean space to there sphere $\rr^3$ with outputs constrained to the sphere.
\section*{Acknowledgments}
Anastasis Kratsios and Ivan Dokmani\'{c} were supported by the European Research Council (ERC) Starting Grant 852821---SWING.  The authors thank Wahid Khosrawi-Sardroudi, Phillip Casgrain, and Hanna Sophia Wutte from ETH Z\"{u}rich, Valentin Debarnot from the University of Basel for their helpful feedback, and Sven Seuken from the University of Z\"{u}rich for his helpful feedback in the rebuttal phase.
\bibliography{iclr2021_conference}
\bibliographystyle{iclr2022_conference}
\section{Precise Transformer Complexities and Approximation Rates}
\label{a_Precise_rates}
This section records the exact approximation rates, or equivalently the precise model complexities, of the transformer networks implemented in our quantitative constrained universal approximation results.  The rates are simply those recorded in Table~\ref{tab_Big_O_Rates} but with explicit constants.  

Table~\ref{tab_model_complexities} makes use of the following notation.  
We denote the diameter of the compact set $K$ by $\operatorname{diam}(K)\triangleq \max_{y_1,y_2\in K}\, \|y_1-y_2\|$.  Furthermore, $k>0$ in Table~\ref{tab_model_complexities} is a universal constant independent of $\epsilon_K$, $\epsilon_f$, $f$, $K$, $n$, $m$, and of $d$.  The big $\mathscr{O}$ notation used in Table~\ref{tab_model_complexities} masks any constants not depending on these quantities.  
\begin{table}[h]
    \centering
    \begin{adjustbox}{width=\columnwidth,center}
    \begin{tabular}{c|cc}
\toprule
        Network & $\hat{\mathcal{E}}$ & $\hat{\mathcal{D}}$\\
\midrule    
        \hline
       Depth 
        & 
        $
        \mathscr{O}(
                m^{\frac1{s}}(
                  1
                    +
                \epsilon_f^{\frac{2n}{3(kn+1)}-\frac{2n}{kn+1}}
                )
            )
        $ 
        & 
        $
        \mathscr{O}\left(
        (
m^{\frac{1}{s}} N^{\frac{3}{2}}( \operatorname{Lip}(\Phi)\operatorname{diam}(K) + 2\epsilon_f) (1-\frac{\epsilon_K^{-1}}{4})^2 (1+\frac{m^{\frac{1}{s}}}{4})
)^{\frac{2m}{s}}
        \right)
        $
        \\
        Width  
        & $\mathscr{O}(m^{\frac1{s}}(4n+10))$ 
        & $
        \mathscr{O}(m^{\frac1{s}}+N+2)$
        \\
        $N$  & - 
        & $
        \mathscr{O}\left(
          \left(
    \frac{k m^{\frac{2}{s}}
    2^{\frac{9}{2}}
    \operatorname{Lip}(\Phi)(\operatorname{diam}(K)+\epsilon_f)
    }{
    \sqrt{m^{\frac1{s}}+1}\epsilon_K
    }
    \right)^{\frac{m}{s}}
    \right)$
        \\
    $Q$  & -
    & $\mathscr{O}\left(
    (\epsilon_K^{-1} \operatorname{Lip}(\Phi)\operatorname{diam}(K)m^{\frac{5}{2s}})^{\frac{m}{s}}
    \right)$
    \\
\bottomrule
    \end{tabular}
    \end{adjustbox}
    \caption{Complexities of encoder-decoder Network with P-attention for $0<\epsilon_K\leq 1$.}
    \label{tab_model_complexities}
\end{table}
\begin{remark}
If $\epsilon_K>1$ then the above rates hold but the term $(1-\frac{\epsilon_K^{-1}}{4})^2$ in $\hat{\mathscr{D}}$'s depth estimate must be replaced by $(1-\frac{\epsilon_K^{-1}}{4})(1-\epsilon_K^{-1})$ in order for the rates to remain valid.
\end{remark}


\section{Can the Probabilistic Transformer Network be Trained?}
\label{a_Training_the_Prob_Transformer}
The purpose of this appendix is to answer the following questions:
\begin{enumerate}
    \item[(i)] \textit{Are our probabilistic transformer networks trainable and, if so, how?}
    \item[(ii)] \textit{How do probabilistic transformer networks perform on a toy non-convex problem?}
\end{enumerate}
We first affirm (i) by describing a potential training algorithm for our model.  Then, we address (ii) on a toy non-convex problem whose objective is to learn (randomly generated) functions taking values on the standard $2$-sphere in $\rr^3$.  Our code is available at code is available at~\cite{TransformerCode}.

\subsection{A Training Algorithm}\label{a_Training_the_Prob_Transformer___aa_training_algo}
Our analysis only has practical implications as we can affirmatively answer the following question:
\[
\mbox{\textit{``Is the probabilistic transformer network $\hat{F}$ of Theorem~\ref{thrm_quantitativer_version_NonConvex} trainable?"}}
\]
Our theoretical analysis motivates the following training procedure, whose steps we briefly explain.

\paragraph{Step 1 - Get Particles:} We assume that the user is has access to a subroutine $\operatorname{Generate}$ which generates particles on $K$.  This is always possible, for example, by randomly by sampling the available training outputs $\{y_t\}_{t=1}^T$; for example, this is what is done in our toy implementations.  However, if one has access to additional structure, such as a $K$-supported probability measure \citep{JMLR_Manopt_MiolanGuilguiLeBriantetcGeomSTATsPython} then, samples can be drawn therefrom, or if $K$ is equipped with a meaningful metric, then the randomized partition procedure such as \citep{Bartalmetricapprox,MaxflowmMincutTheoremImprovedPulat1989} can be deployed.  

\paragraph{Step 2 - Train Model:}
When $m>1$, the Wasserstein distance is costly to evaluate numerically \citep{Computational_optimal_transport_IEEE_PeleWerman,Computational_optimal_transport_NIPS2013_Cuturi,Computational_optimal_transport_NeurIPS_KolourNadjahiSimselkiBadeaRohde,Computational_optimal_transport_JMLR2019SommerfeldScrieverZemelMunk}.  A variety of approximate or regularized transport distances have been introduced to manage this problem but only approximately.  However, in the context of Theorem~\ref{thrm_quantitativer_version_NonConvex} and Algorithm~\ref{algo_train} we are always interested in distances to the pointmass and therefore, $\www_1$ has the following exceptional \textit{closed-form} expression which bypasses these computational issues:
\begin{equation}
    \sum_{t=1}^T \mathcal{W}_1(\delta_{f(x_t)},\operatorname{P-attention}(\hat{f}(x_t),Y))
=
\sum_{t=1}^T \sum_{n=1}^N
\left\|
y_t
-
Y_n
    \right\|
    [\operatorname{Softmax}_N\circ \hat{f}(x_t)]_n
\label{eq_closed_form_wasserstein}
.
\end{equation} 
\begin{remark}[{Exceptional Closed-Form for $\mathcal{W}_1(\delta_{f(x_t)},\operatorname{P-attention}(\hat{f}(x_t),Y))$ in~\eqref{eq_closed_form_wasserstein}}]\label{remark_derivation}
A derivation of the closed-form identity~\ref{eq_closed_form_wasserstein} is in Lemma~\ref{lem_closedform_wasserstein_dirac}.
\end{remark}

\paragraph{Step 2 - Prediction:} In the case where $K$ is a Riemannian manifold, 
Since Fr\'{e}chet means are readily implemented in a variety of packages \citep{JMLR_Manopt_MiolanGuilguiLeBriantetcGeomSTATsPython}, we assume that the user has access to a subroutine $\operatorname{Fr\mbox{\'{e}}chet\,Mean}$ which takes an $N\times Q$-matrix of weights $(w_{n,q})_{n,q=1}^{N,Q}$ and an $N\times Q\times 1$-array and computes the Fr\'{e}chet mean~\eqref{eq_frechet_mean_def}.  
By Corollary~\ref{cor_nonconvex_geometric_case}, once the network $\hat{F}$ is trained, its outputs generate points on $K$ via the Fr\'{e}chet mean~\eqref{eq_frechet_mean_def}; i.e.:
$
\overline{\hat{F}(x)}
= 
\operatorname{argmin}_{y \in K} \sum_{n,q=1}^{N,Q} w_{n,q} d_g^2(y,y_{n,q})
.
$

\SetAlgoNoLine
\LinesNotNumbered
\begin{algorithm}[ht]
\DontPrintSemicolon
\KwIn{Training Data $\{(x_t,y_t)\}_{t=1}^T\subseteq \rr^n\times K$
}
\KwOut{
Probabilistic Transformer Network:
$
\hat{F}\triangleq 
\sum_{n=1}^N \sum_{q=1}^Q [\operatorname{Softmax}_N\circ f(\cdot)]_n 
w_{n,q}
    \delta_{y_{n,q}}
$
}
\nextnr
\textbf{Get Particles:} 
Use $\operatorname{Generate}$ $K$ to generate $y_1,\dots,y_S\in K$
\;
\For{$n=1,\dots,N$}{
    $\{s_q\}_{q=1}^Q 
        \leftarrow
         \operatorname{argsort}_Q \,\{\|y_s-Y_n\|\}_{s=1}^S$
    \\
    $
    \{y_{n,q}\}_{q=1}^Q 
        \leftarrow
        \{y_{s_q}\}_{q=1}^Q
    $
    }
\nextnr
\textbf{Train Model:}\;Get Labels:
\For{$t\leq T$}{
\For{$n\leq N$}{
$(L_t)_n \leftarrow I(\|y_t - Y_n\|\leq \min_{m=1,\dots,N}\|y_t-Y_m\|$)
}
}
\nextnr
$
\hat{f}
    \leftarrow
\operatorname{argmin}_{\hat{f}}
\sum_{t=1}^T \sum_{n=1}^N
\|(L_t)_n 
    -[\operatorname{Softmax}_N\circ \hat{f}(x_t)]_n
\|^2
$
\;
\Return{
$
\hat{F}(\cdot)
    \triangleq 
    \sum_{n=1}^N
    \operatorname{Softmax}_N\circ \hat{f}(\cdot)_n\delta_{Y_n}
$
}
\label{algo_train}
\caption{Training Probabilistic Transformers for Exact Constraint Satisfaction}
\end{algorithm}
\begin{remark}[Prediction]\label{remark_prediction}
Predictions can be made using $\hat{F}$ by either applying an expectation, in which case classical tranformer networks of \cite{vaswani2017attention} are recovered, using the Fr\'{e}chet mean as a final layer if $K$ is a geodesically convex subset of a Riemannian submanifold of $\mathbb{R}^m$, or taking the most-likely particle if noting more is known of $K$ other than its point-set.  
\end{remark}

We now address question (ii).  
\subsection{Performance on a Toy Non-Convex Problem}
\label{a_Training_the_Prob_Transformer___aa_Performance_on_Toy_Example}
Let $K\subseteq \rr^{2}$ be a $2$-dimensional sphere in $\rr^{3}$.  
Let $a,b,c$ be independently drawn from a uniform distribution on $[0,1]$ and let $A$ be a $2\times {10^3}$ random matrix with i.i.d. standard Gaussian entries.  Let
$f=\tilde{f}(Ax)$ be the random $K$-valued function where
$u_i=a x^2_i+b x_i +c$, for $i=1,2$, and 
$\tilde{f}(u)\triangleq (
  \cos(\tilde{f}(u)_1)\sin(\tilde{f}(u)_2)),
  \sin(\tilde{f}(u)_1)\sin(\tilde{f}(u)_2)),
  \cos(\tilde{f}(u)_1)\cos(\tilde{f}(u)_2))
  )
$.  
Therefore, $A$ projects $\rr^{{10^3}}$ onto the latent low-dimensional space $\rr^2$ and $\tilde{f}$ sends data in $\rr^2$ to a point on the sphere obtain by a random polynomial transformation of its spherical coordinates (which is a non-convex constraint set).  

We independently repeat this experiment $500$-times, generating a random $f$ each time and generating $1k$ training inputs $\{x_t\}_{t=1}^{{10^3}} \subseteq [0,1]^{{10^3}}$ (resp. $100$ testing inputs) by independently and uniformly sampling from $[0,1]^{{10^3}}$ and producing $1k$ corresponding training (resp. $100$ testing) outputs $\{f(x_t)\}_{t=1}^{{10^3}}\subset K$.   For each independent experiment, a probabilistic transformer network {\color{forestgreen}{$\pp\mbox{-Trans.}$}}) is trained using Algorithm~\ref{algo_train}, and benchmarked against a deep feedforward network (MLP) and a classical transformer network ({\color{BurntOrange}{Trans.}}).  Table~\ref{tab_experiments} reports the average and standard deviation, across all experiments, of the test-set MSE and the distance to the constraint set ($d_K$) of the test-set predictions for each learning model.  

Figure~\ref{fig_image} shows that, high emphasis is placed on constraint satisfaction ($\lambda\in [0,0.75]$) then the $\pp$-Trans.+Fr\'{e}chet  model outperforms the benchmark models.  As the emphasis parameter $\lambda$ approaches the critical value of $\approx.75$ then, the MSE dominates the constraint satisfaction metric $d_K$ and the $\pp$-Trans.+Fr\'{e}chet's larger average test MSE is larger than that of the MLP and Trans. models.  This validates the error terms $\epsilon_K$ and the factor $k\operatorname{Lip}(\Phi^{-1})d$ in Theorem~\ref{thrm_quantitativer_version_NonConvex} (ii), reflected in Table~\ref{tab_experiments}, which is due to the decoder network $\hat{\mathcal{D}}$ in $f$ approximating a random projection of $\rr^{3}$ onto $K$.  

\begin{figure}[ht]
\centering
\includegraphics[width=0.5\linewidth]{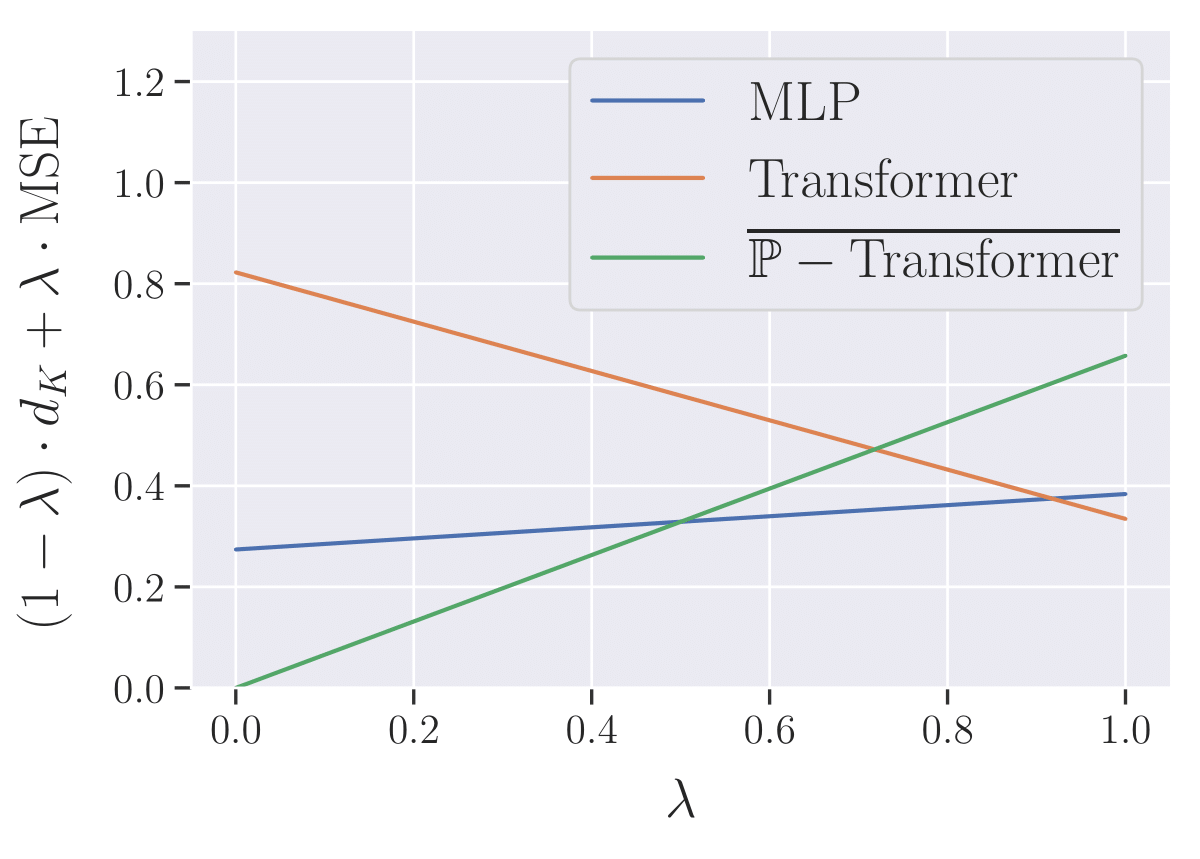}
\caption{Performance for varying importance on constraint satisfaction vs. MSE.}
\label{fig_image}
\end{figure}

Therefore, we find that our probabilistic transformer network is both implementable, and that, as expected, it offers good predictive performance even while enforcing non-convex constrained.  In other words, we have obtained positive answers to the natural questions (i) and (ii) posed at the start of Appendix~\ref{a_Training_the_Prob_Transformer}.  

\begin{table}[ht]
\centering
\begin{tabular}{lrrrr}
\toprule
& \multicolumn{2}{c}{$d_K$} & \multicolumn{2}{c}{MSE} \\ \cmidrule(r){2-3} \cmidrule(r){4-5}
&  Mean  &  Std & Mean & Std \\
\midrule
{\color{darkcerulean}{MLP}}                &  0.274 &   0.106 &   0.384 & 0.034 \\
{\color{BurntOrange}{Transformer}}             &   0.822 &   0.100   &   0.334  &0.009  \\
{\color{forestgreen}{$\overline{\pp\mbox{-Transformer}}$}} &   0.000 &  0.000 &   0.657 &  0.059 \\
\bottomrule
\end{tabular}
\caption{Performance metrics across all $500$ experiments.}
\label{tab_experiments}
\end{table}

Table~\ref{tab_experiments} emphasizes that classical transformer networks are not built to handle non-convex constraints.  Indeed, the poor performance of the transformer network, with respect to the $d_K$, is due to most its predictions lying inside the sphere (which is hollow).

Further study into training algorithms for our model, and detailed ablation of the model parameters are topics of focus in forthcoming research.  Nevertheless, we have obtained an affirmative answer to both our questions (i) and (ii).  Namely, we have shown that our probabilistic transformer morel is trainable via a simple procedure such as Algorithm~\ref{algo_train}.

{\color{black}{
\subsection{Examining the impact of {\color{darkgreen}{$K$}}'s Geometry on Transformer Networks}
\label{a_Training_the_Prob_Transformer__ss_Further_Ablation}
To gain further insight into how probabilistic transformers encode geometric priors, we will examine the impact of perturbations to {\color{darkgreen}{$K$}} on the probabilistic transformer's approximation capabilities.  We consider toy illustrations beginning with the convex setting before moving on to the fully non-convex setting where no projections, charts, or even a Riemannian structure is available.  

We further underline that step $3$ in Algorithm~\ref{algo_train} may be performed in a variety of ways and, unlike the previous experiments, all networks in this section are obtained by randomizing their hidden weights and only training their final layer.  Theoretical guarantees for this approach has become relatively well understood \citep{RandomRNN_RandomMatrixPerspective_AnnStat2018,GononJPLyudmila_RandomJMLR2020,GononJPLyudmila_RandomRNNs2021}.  Implicitly, our examples also show that probabilistic transformers can equally be integrated into domains where randomized models such as extreme learning machines (ELMs) are typical; e.g. in the reservoir computing \cite{LUKOSEVICIUS2009127,JPLG2018,JPLG2019}. 

This section's primary goal is to experimentally validate the main quantitative claim made implicitly in our main result; i.e. Theorem~\ref{thrm_quantitativer_version_NonConvex}.  Namely, we verify that: 
\[
\mbox{
\textit{``The model complexities in Table~\ref{tab_Big_O_Rates} are independent of {\color{darkgreen}{$K$}}'s geometry.'' }
}
\]
That is, the approximation quality of any optimized probabilistic transformer network only depends on the involved dimensions.  Expressed another way, we empirically validate our result that the probabilistic transformer networks can encode any geometric prior with the model complexity agnostic of {\color{darkgreen}{$K$}}'s geometry.  

Accordingly, all model architectures' hyperparameters are kept fixed across all experiments.  Each experiment reports the probabilistic transformer's MSE relative to the benchmark MLP model {\tiny{$\left(\frac{\mbox{MSE}}{\mbox{MSE-MLP}}\right)$}}.  

Our result is validated upon observing that the probabilistic transformer model's {\tiny{$\frac{\mbox{MSE}}{\mbox{MSE-MLP}}$}} is of the same order across all experiments.  In other words: \textit{probabilistic transformers can approximate a {\color{darkgreen}{$K$}}-valued function while simultaneously encoding {\color{darkgreen}{$K$}}'s geometry with the same efficiency as an MLP trained only to approximate $f$ that ignores {\color{darkgreen}{$K$}}'s geometry.  }  

Each toy experiment is trained on a dataset of $900$ instances and tested on a dataset of $100$ instances.  We maintain the coloring scheme of Figure~\ref{fig_image} in all our subsequent plots, namely the {\color{darkcerulean}{MLP is colored in blue}}, the
{\color{BurntOrange}{Transformer is colored in orange}}, and the {\color{forestgreen}{$\pp\mbox{-Transformer}$ is colored in green.}}

\subsubsection{Convex Constraints}
\label{a_Training_the_Prob_Transformer__ss_Further_Ablation____sss_Convex_Constraints}

Our first set of examples concern the case where $K$ is a \textit{convex} constraint set, as studied in Corollaries~\ref{cor_convex_case}.  In each instance, we generate a random target function $f$, mapping $\rr^2$ to $K$.  The function $f$ is defined via the following two-step procedure:
\[
f : \mathbb{R} \overset{\tilde{f}}{\rightarrow} \mathbb{R}^2 \overset{P_K}{\rightarrow} {\color{darkgreen}{K}}
;
\]
where $\tilde{f}:\rr^2\ni x\mapsto Ax$ is a random rotation matrix re-scaled by a factor of $1.5$ with the angle sampled uniformly from $[0,2\pi]$ and where $P_K:\rr^2\ni x\mapsto \operatorname{argmin}_{y\in K}\,\|y-x\|$ is the \textit{metric projection} onto the $K$, which exist in this context by \citep{motzkin1935quelques}.  Figures~\ref{fig_x_fx_Square} and ~\ref{fig_x_fx_Circle} illustrate this two step transformation by first representing the uniformly generated input data in {\color{violet}{violet}} then, illustrating their images under $\tilde{f}$ in {\color{forestgreen}{light green}}, and finally plotting their value under $f$ in {\color{darkgreen}{dark green}}.

We perform our illustrations in the visualizable two dimensional case where $K$ is either the square $[-1,1]^2$ or the disk $\{y\in \rr^2:\, \|y\|\leq 1\|$.  This is because the projection operators ($P_K$) are readily interpretable from their closed-form formulations \citep{CombettesBauschkeConvexMontoneCMS2011}.  Respectively, these are given by $P_K(x) =(\min\{\max\{x_i,-1\},1\})_{i=1}^2$ and $P_K(x):=\frac{x}{\max\{1,\|x\|\}}$.

Figures~\ref{fig_Geometry_and_Particels_Square} and~\ref{fig_Geometry_and_Particels_Circle} demonstrate the constraint set ($K$) in {\color{darkgreen}{green}} and the particles, which populate $Y$'s rows, in {\color{violet}{violet}}.  These are generated randomly by first sampling uniformly from $[-2,2]^2$ and then projecting each sample onto $K$ via $P_K$.  Figures~\ref{fig_Geometry_and_Particels_Square} and~\ref{fig_Geometry_and_Particels_Circle} illustrate the role of the particles defining the probabilistic attention mechanism, defined in~\eqref{eq_probabilistic_attention}; namely, they identify the points in $K$ on which any output may \textit{approximately lie}.  Thus, the role of the encoder and decoder networks can be summarized as learning to classify which input is closest to which {\color{violet}{particle}}.

{\centering
\begin{minipage}{0.5\textwidth}
  \centering
  \centerline{\includegraphics[scale=0.15]{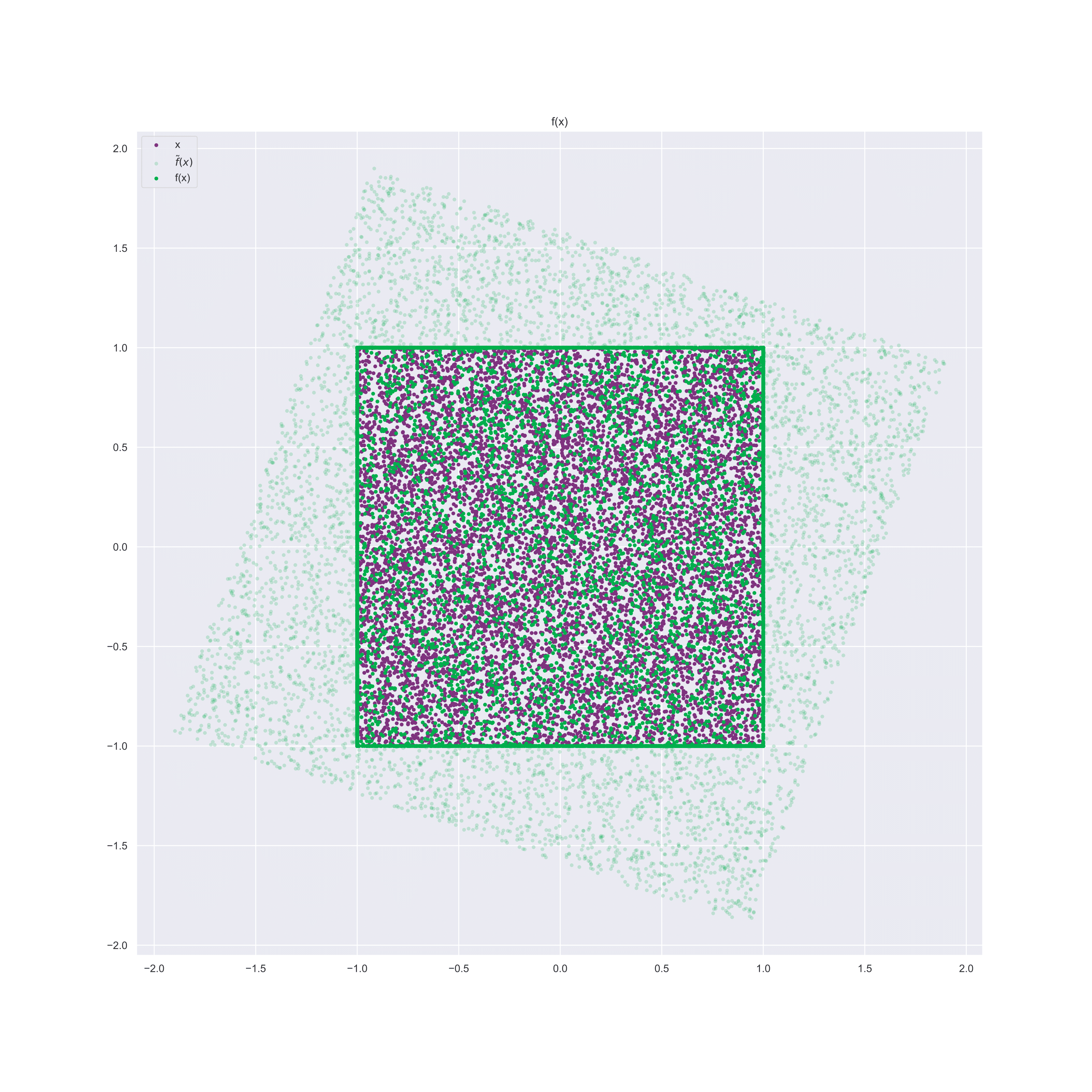}}
  \captionof{figure}{$x\mapsto f(x)$.}
  \label{fig_x_fx_Square}
\end{minipage}
\begin{minipage}{0.5\textwidth}
  \centering
   \centerline{\includegraphics[scale=0.15]{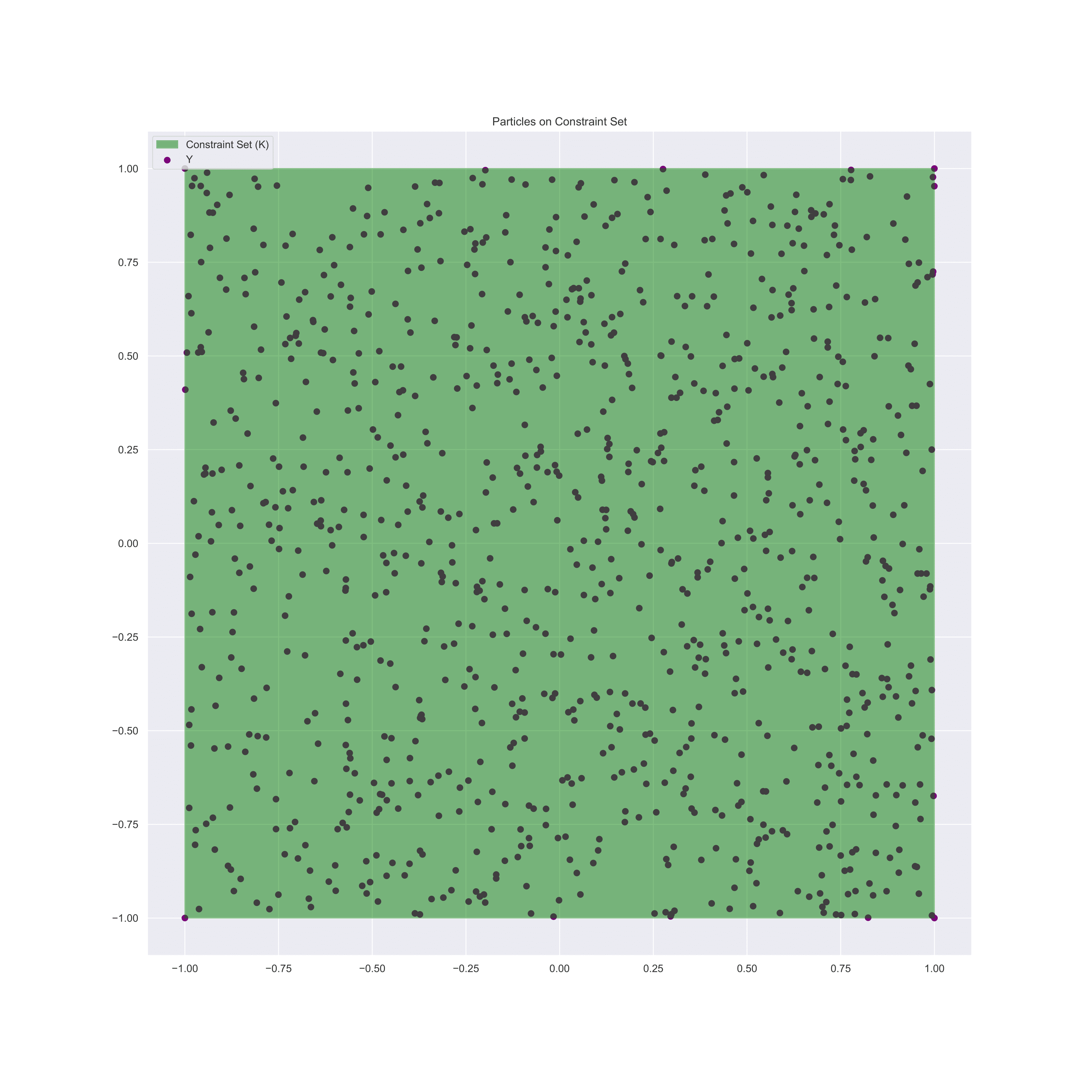}}
  \captionof{figure}{Particles $(Y)$ on Constraint Set $(K)$.}
  \label{fig_Geometry_and_Particels_Square}
\end{minipage}
}


Therefore, at high-level, the probabilistic attention mechanism~\ref{eq_probabilistic_attention} quantizes the constraint set {\color{darkgreen}{$K$}}.  The (simplified) classical Attention mechanism of~\eqref{eq_THE_KEY_IDENTITY___attention_implementation_via_expecation} implements a (convex) interpolation between the {\color{violet}{particles}} quantizing {\color{darkgreen}{$K$}} and an analogous statement is true of Riemannian analogue (Section~\ref{s_Applications_ss_GDL}).  For general {\color{darkgreen}{$K$}}, however, such interpolations within {\color{darkgreen}{$K$}} can be impossible or unclear how to implement them.  Nevertheless, the probabilistic attention mechanism never faces such a difficulty since it explicitly ``interpolates'' in $\mathcal{P}_1(K)$ and not on {\color{darkgreen}{$K$}} directly.


{\centering
\begin{minipage}{0.45\textwidth}
  \centering
  \centerline{\includegraphics[scale=0.11]{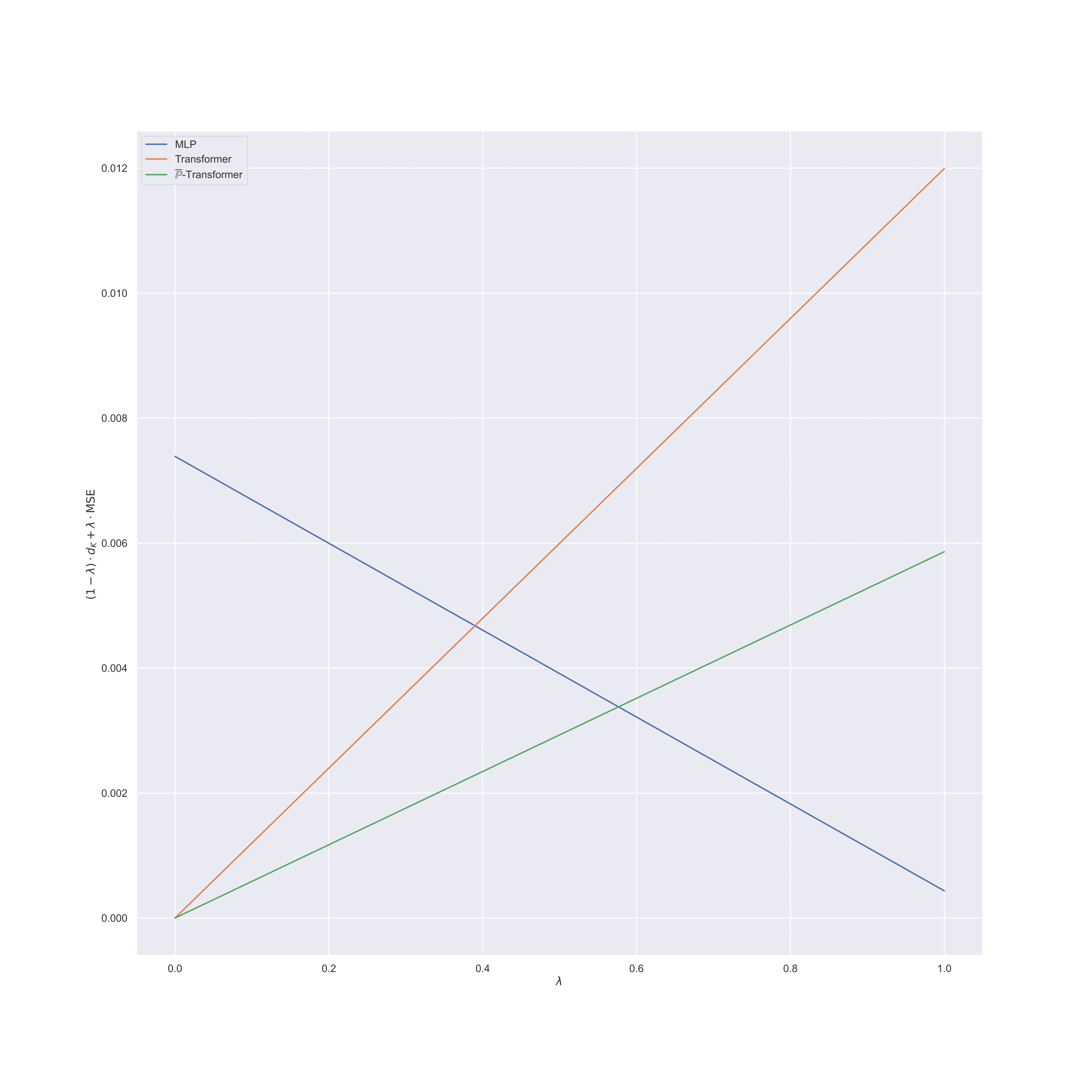}}
  \captionof{figure}{Performance: MSE vs. $d_K$.}
  \label{fig_Frontier_Square}
\end{minipage}
\begin{minipage}{0.55\textwidth}
  \centering
  \begin{adjustbox}{width=\columnwidth,center}
  \begin{tabular}{lrrr}
    \toprule
    {} &      $\frac{\mbox{MSE}}{\mbox{MSE MLP}}$ &      MSE &      $d_K$ \\
    \midrule
    {\color{darkcerulean}{MLP}}        & 1 & 4.35e-04 & 7.39e-03 \\
    {\color{BurntOrange}{Transformer}}     & 4.81 & 1.20e-02 & 0.00e+00 \\
    {\color{forestgreen}{$\overline{\pp\mbox{-Transformer}}$}} & 4.01 & 5.86e-03 & 0.00e+00 \\
    \bottomrule
    \end{tabular}
    \label{tab_Square}
    \end{adjustbox}
  \captionof{table}{Performance Metrics}
\end{minipage}
}

We obtain analogous results to the $500$ experiments performed in the case where $K$ is geodesically-convex in Section~\ref{a_Training_the_Prob_Transformer___aa_Performance_on_Toy_Example}.  Just as in Figure~\ref{fig_image}, Figures~\ref{fig_Frontier_Square} and~\ref{fig_Frontier_Circle} show that the transformer can simultaneously encode $K$ and approximate $f$, wheras the MLP cannot.  More precisely, in each case, if at-least roughly equal importance is placed on predictive accuracy (MSE) and constraint satisfaction ($d_K$) then, the transformer models offer the best performance.  This is equally reflected in the test set performance metrics of Tables~\ref{tab_Circle} and~\ref{tab_Square} which are consistent with the findings of Table~\ref{tab_experiments}.  

{\centering
\begin{minipage}{0.5\textwidth}
  \centering
  \centerline{\includegraphics[scale=0.15]{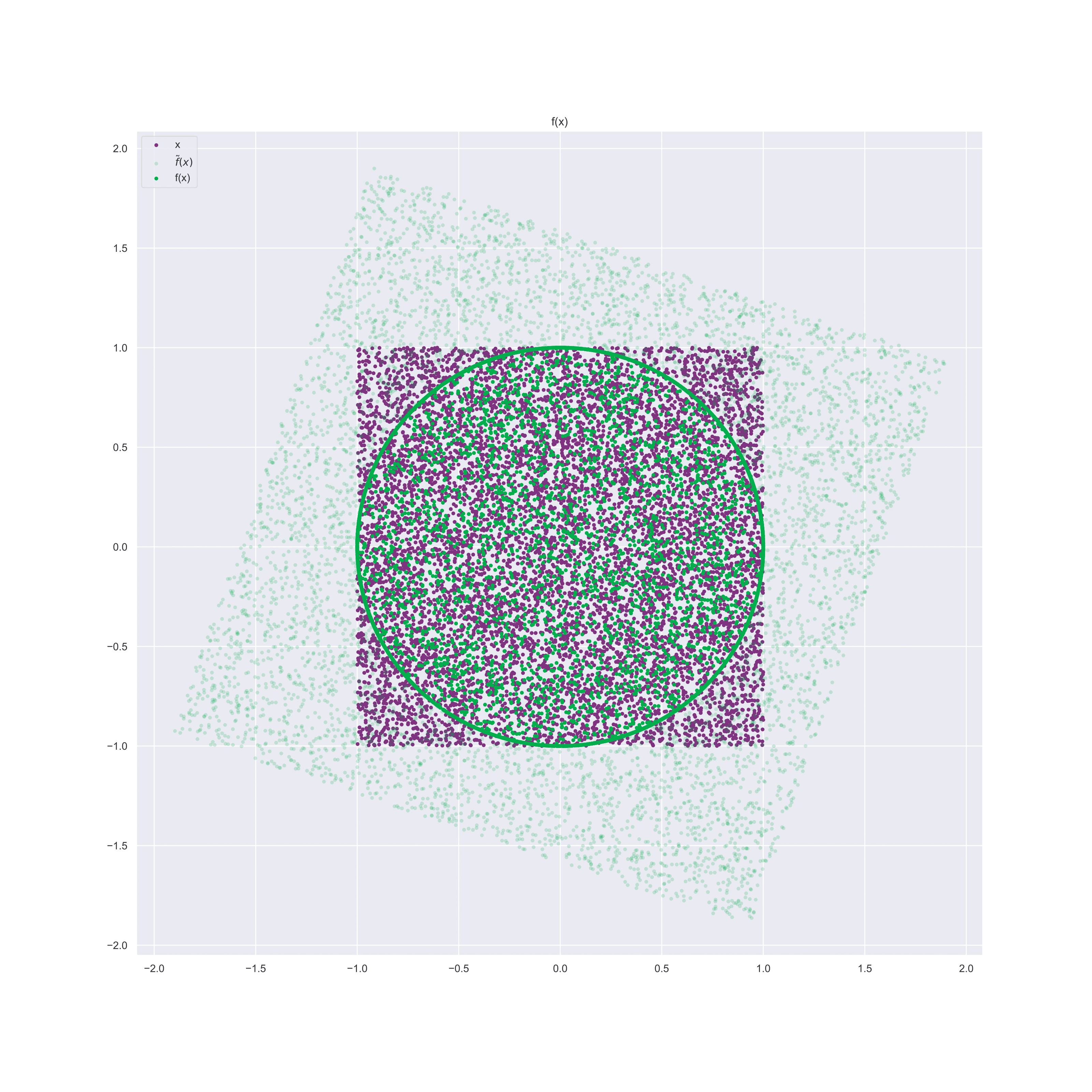}}
  \captionof{figure}{$x\mapsto f(x)$.}
  \label{fig_x_fx_Circle}
\end{minipage}
\begin{minipage}{0.5\textwidth}
  \centering
   \centerline{\includegraphics[scale=0.15]{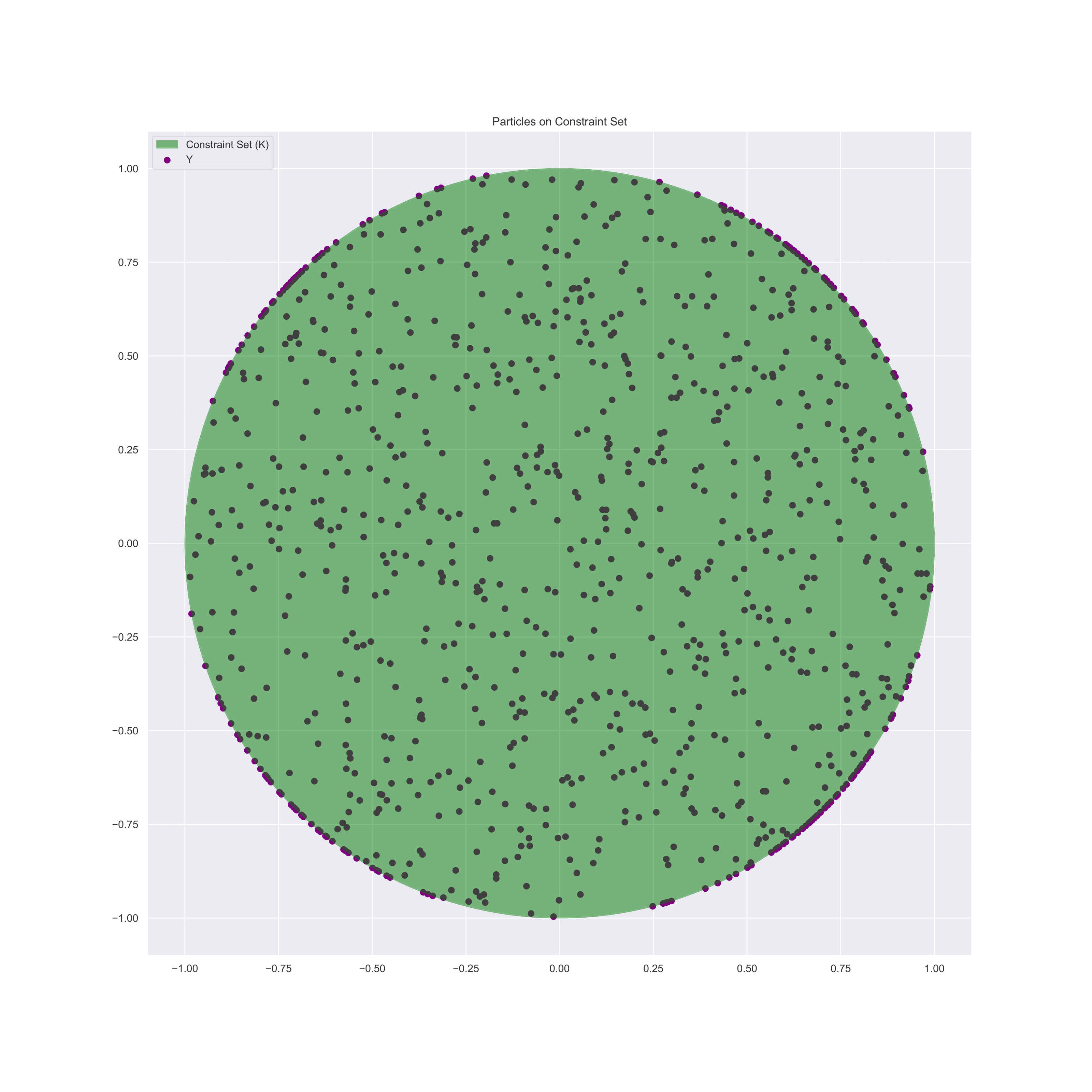}}
  \captionof{figure}{Particles $(Y)$ on the constraint set $(K)$.}
  \label{fig_Geometry_and_Particels_Circle}
\end{minipage}
}

At times, when $K$'s geometric is sufficiently simple we the transformer can outperform the probabilistic transformer.  This is not surprising, since  Corollary~\ref{cor_convex_case} guaranteed that the transformer universal in this setting an exactly implements the {\color{darkgreen}{$K$}}'s geometry.  Nevertheless, in both instances, the MLP cannot compete when encoding the geometric prior defined by the constraint set {\color{darkgreen}{$K$}}.

{\centering
\begin{minipage}{0.45\textwidth}
  \centering
  \centerline{\includegraphics[scale=0.11]{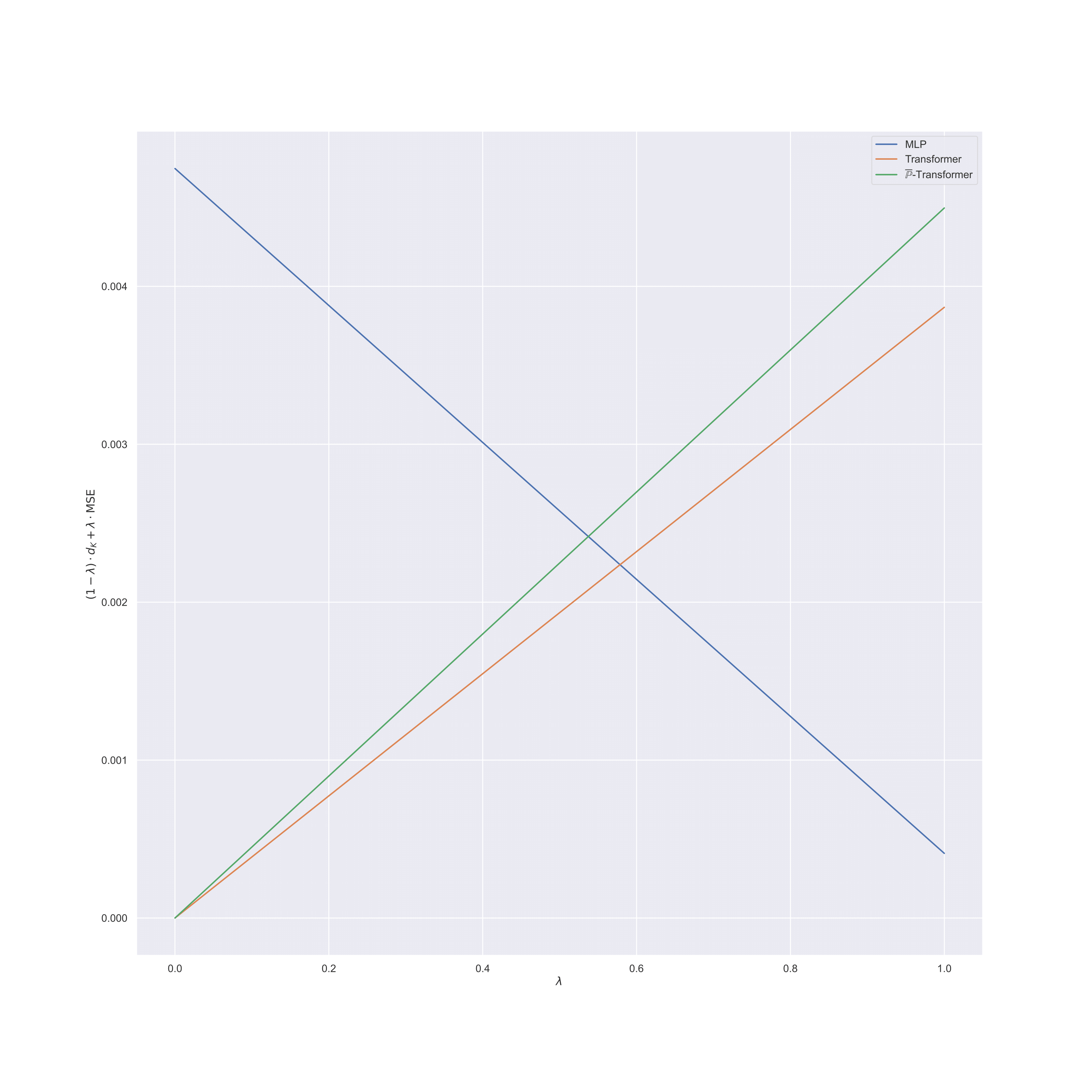}}
  \captionof{figure}{Performance: MSE vs. $d_K$.}
  \label{fig_Frontier_Circle}
\end{minipage}
\begin{minipage}{0.55\textwidth}
  \centering
  \begin{adjustbox}{width=\columnwidth,center}
  \begin{tabular}{lrrr}
    \toprule
    {} &      $\frac{\mbox{MSE}}{\mbox{MSE MLP}}$ &      MSE &      $d_K$ \\
    \midrule
    {\color{darkcerulean}{MLP}}        & 1 & 4.09e-04 & 4.75e-03 \\
    {\color{BurntOrange}{Transformer}}     & 2.98 & 3.87e-03 & 0.00e+00 \\
    {\color{forestgreen}{$\overline{\pp\mbox{-Transformer}}$}} & 3.41 & 4.50e-03 & 6.11e-19 \\
    \bottomrule
    \end{tabular}
    \label{tab_Circle}
    \end{adjustbox}
  \captionof{table}{Performance Metrics}
\end{minipage}
}

\subsubsection{Fully Non-Convex Constraints}
\label{a_Training_the_Prob_Transformer__ss_Further_Ablation____sss_Non_convex_Constraints}
Let us study the milieu in which probabilistic transformer is the only universal approximator capable of constraint satisfaction (unlike the case where {\color{darkgreen}{$K$}} is convex and we showed that the transformer filled this role).  Specifically, we consider the case where $K$ does not admit a single chart (since it has non-trivial homotopy), nor is there a well-defined projection operator of some $\rr^m$ onto {\color{darkgreen}{$K$}}.

Analogously to the convex situation investigated in Section~\ref{a_Training_the_Prob_Transformer__ss_Further_Ablation____sss_Convex_Constraints}, we define
\[
f : \mathbb{R} \overset{\tilde{f}}{\rightarrow} \mathbb{R} \overset{\rho}{\rightarrow} {\color{darkgreen}{K}}
;
\]
where $\tilde{f}(x) \triangleq \sum_{i=0}^5 \beta_i x^i $ is a (random) quintic polynomial function with $\beta_i\sim N(0,1)$, and the constraint set's geometry is defined by {\color{darkgreen}{$K$}}; where $\rho:\rr\rightarrow \rr^2$.
In this experiment, we also allow the training data to be perturbed by multivariate Gaussian noise with variance $10^{-1}$.  

Similarly to Figures~\ref{fig_x_fx_Square} and~\ref{fig_x_fx_Circle}, in Figures~\ref{fig_x_fx_NC_Rose} and~\ref{fig_x_fx_NC_Variety}, we use a color coded visualization method to understand $f$.  Sample points from $[-10,10]$ and label them with a color gradient ranging from {\color{pink}{pink}} to {\color{black}{blue}} such that {\color{pink}{pinkish}} points are close to $-10$ and {\color{black}{blueish}} points are a nearer to $10$.  The image ($f(x)$ of each input ($x$) on {\color{darkgreen}{$K$}} is illustrated using the same colour as $x$ did.  This coloring helps us visualize the winding of $f$ around $K$.

Nevertheless, as in the convex case, we can generate {\color{purple}{particles}} on {\color{darkgreen}{$K$}} by first sampling from $[-10,10]$ and then mapping them onto $K$ using $\rho$.  Thus, even if no chart or projection operator is available, we can easily build probabilistic attention mechanisms.

{\centering
\begin{minipage}{0.5\textwidth}
  \centering
  \centerline{\includegraphics[scale=0.15]{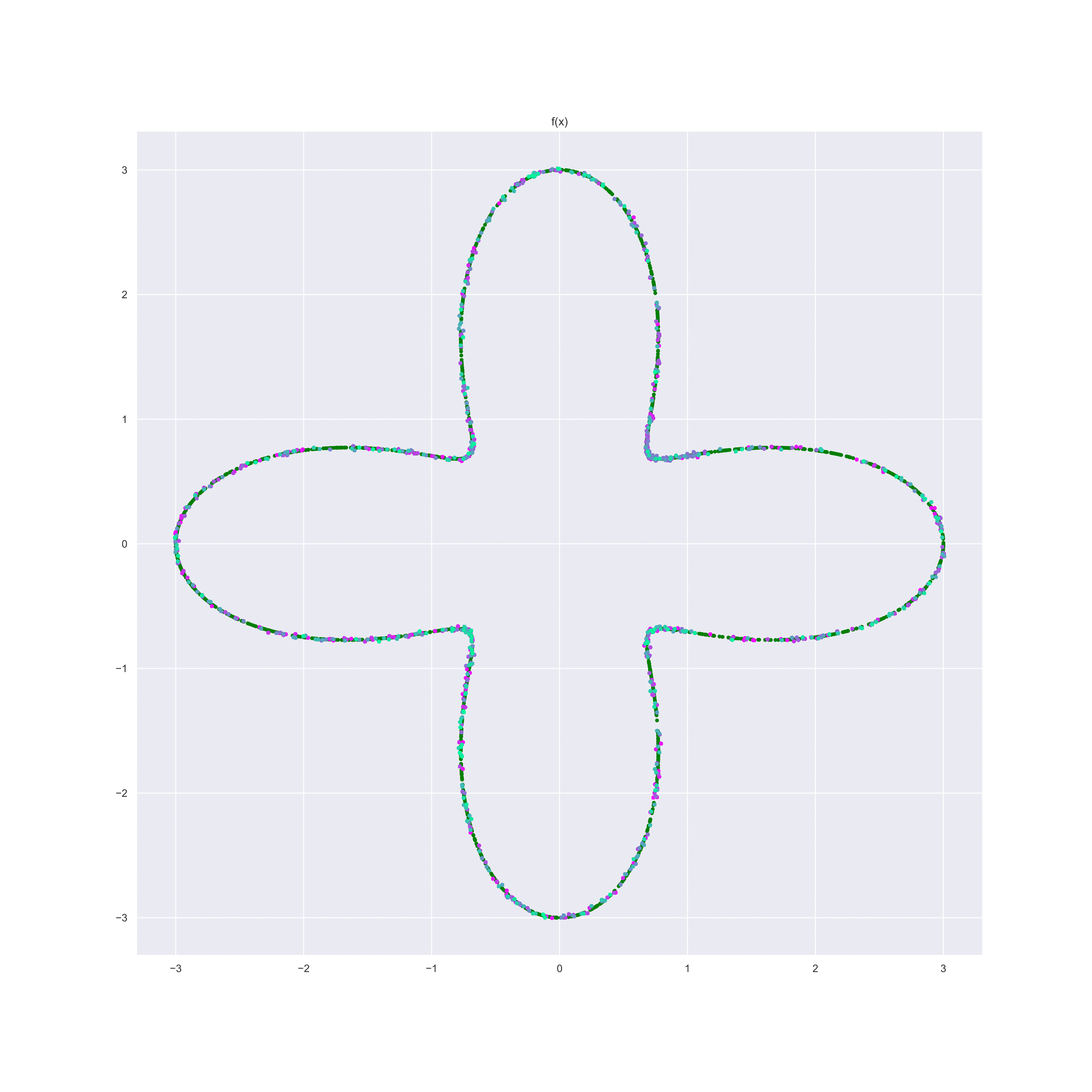}}
  \captionof{figure}{$x\mapsto f(x)$.}
  \label{fig_x_fx_NC_Rose}
\end{minipage}
\begin{minipage}{0.5\textwidth}
  \centering
  \centerline{\includegraphics[scale=0.15]{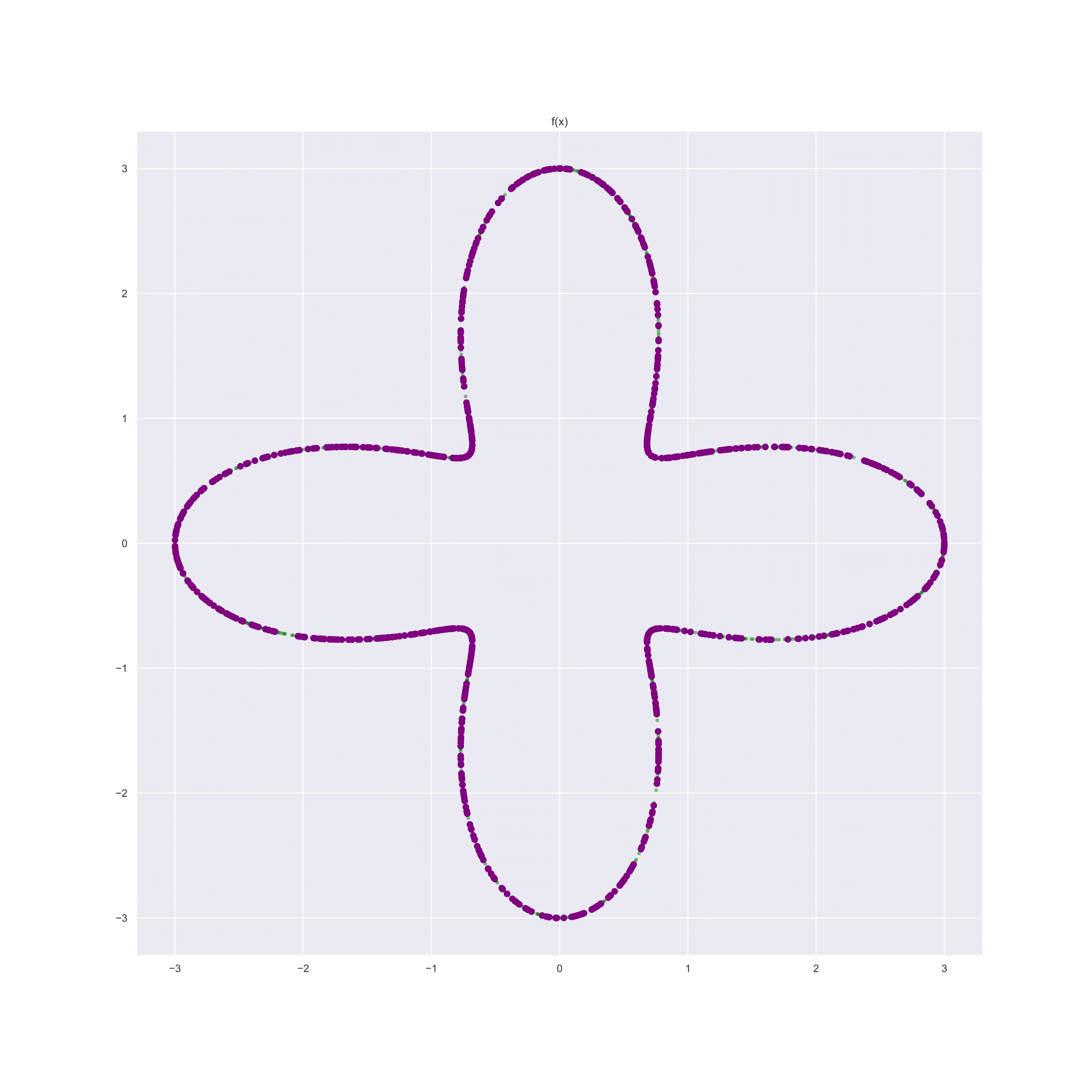}}
  \captionof{figure}{Particles $(Y)$ on Constraint Set $(K)$.}
  \label{fig_Geometry_and_Particels_NC_Rose}
\end{minipage}
}

In Figures~\ref{fig_x_fx_NC_Rose} and~\ref{fig_Geometry_and_Particels_NC_Rose}, the map defining {\color{darkgreen}{$K$}}'s geometry is $\rho(y)\triangleq   (2\cos(y)^2+1)\cdot (\cos(y/3),\sin(y/3))$.  Figure~\ref{fig_Frontier_NC_Rose} and Table~\ref{tab_NC_Rose} show that our probabilistic transformer network's performance is ``robust to changes of geometric priors'', in the sense that the relative performance of our models is entirely analogous to the above experiments where {\color{darkgreen}{$K$}} was convex or it was a geodesically convex patch on a Riemannian manifold.

{\centering
\begin{minipage}{0.45\textwidth}
  \centering
  \centerline{\includegraphics[scale=0.11]{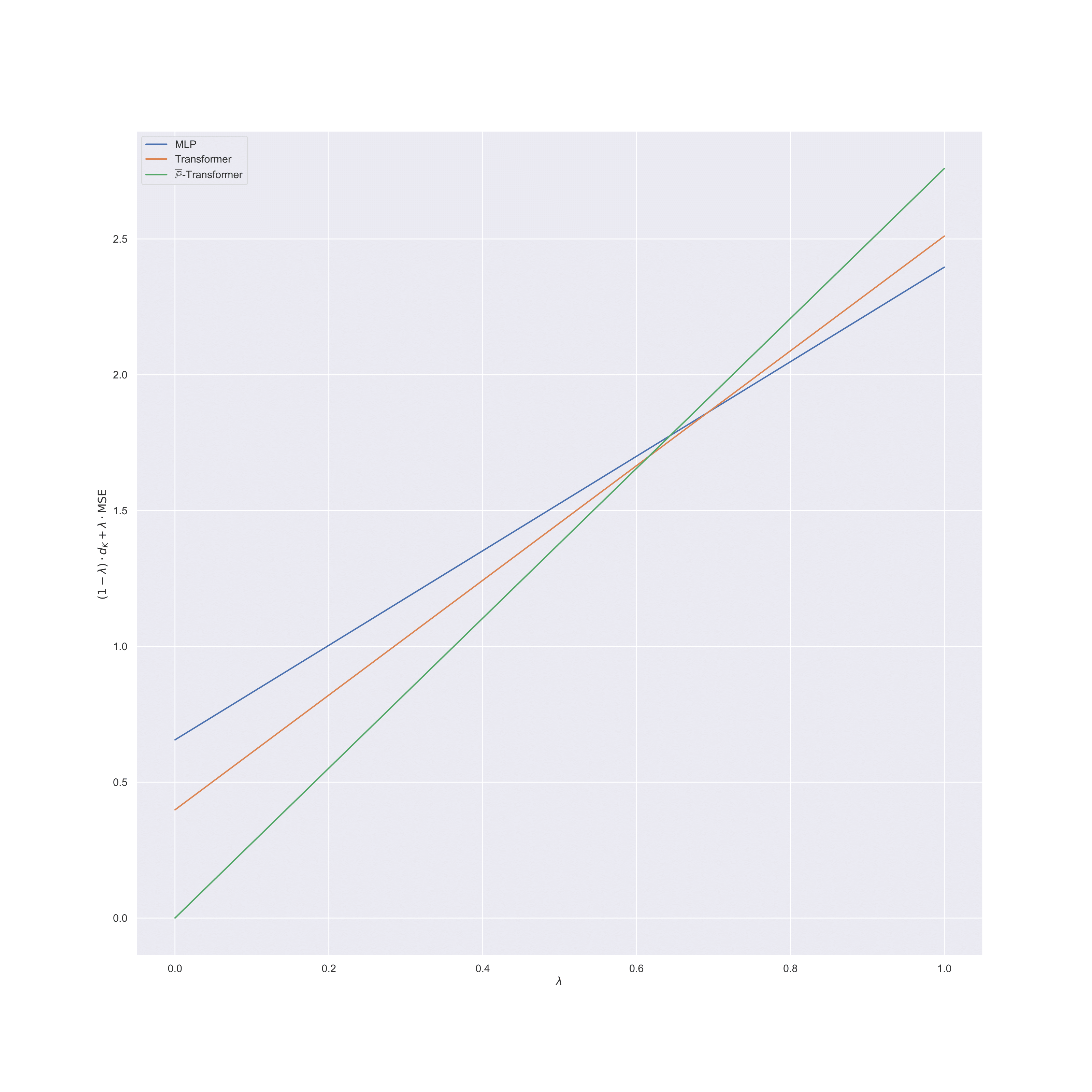}}
  \captionof{figure}{Performance: MSE vs. $d_K$.}
  \label{fig_Frontier_NC_Rose}
\end{minipage}
\begin{minipage}{0.55\textwidth}
  \centering
  \begin{adjustbox}{width=\columnwidth,center}
  \begin{tabular}{lrrr}
    \toprule
    {} &      $\frac{\mbox{MSE}}{\mbox{MSE MLP}}$ &      MSE &      $d_K$ \\
    \midrule
    {\color{darkcerulean}{MLP}}        & 1 & 2.40e+00 & 6.56e-01 \\
    {\color{BurntOrange}{Trans.}}     & 1.01 & 2.51e+00 & 3.98e-01 \\
    {\color{forestgreen}{$\pp\mbox{-Trans.}$}} & 1.05 & 2.76e+00 & 0.00e+00 \\
    \bottomrule
    \end{tabular}
    \label{tab_NC_Rose}
    \end{adjustbox}
  \captionof{table}{Performance Metrics}
\end{minipage}
}

We complete our discussion by considering an instance where {\color{darkgreen}{$K$}}'s geometry is both non-convex and it is not a differentiable manifold (due to the self-intersecting point).  This last toy example is illustrated in Figures~\ref{fig_x_fx_NC_Variety} and~\ref{fig_Geometry_and_Particels_NC_Variety} in which case {\color{darkgreen}{$K$}}'s geometry is the image of the map $\rho(y)\triangleq \Phi(\operatorname{sinc}(y+1)(\cos(y/2),\sin(y/2)))$ where $\Phi$ is a randomly generated invertible feedforward network with invertible square weight matrices and $\tanh$ activation function (i.e.: a random homeomorphism on $\rr^2$).  

{\centering
\begin{minipage}{0.5\textwidth}
  \centering
  \centerline{\includegraphics[scale=0.15]{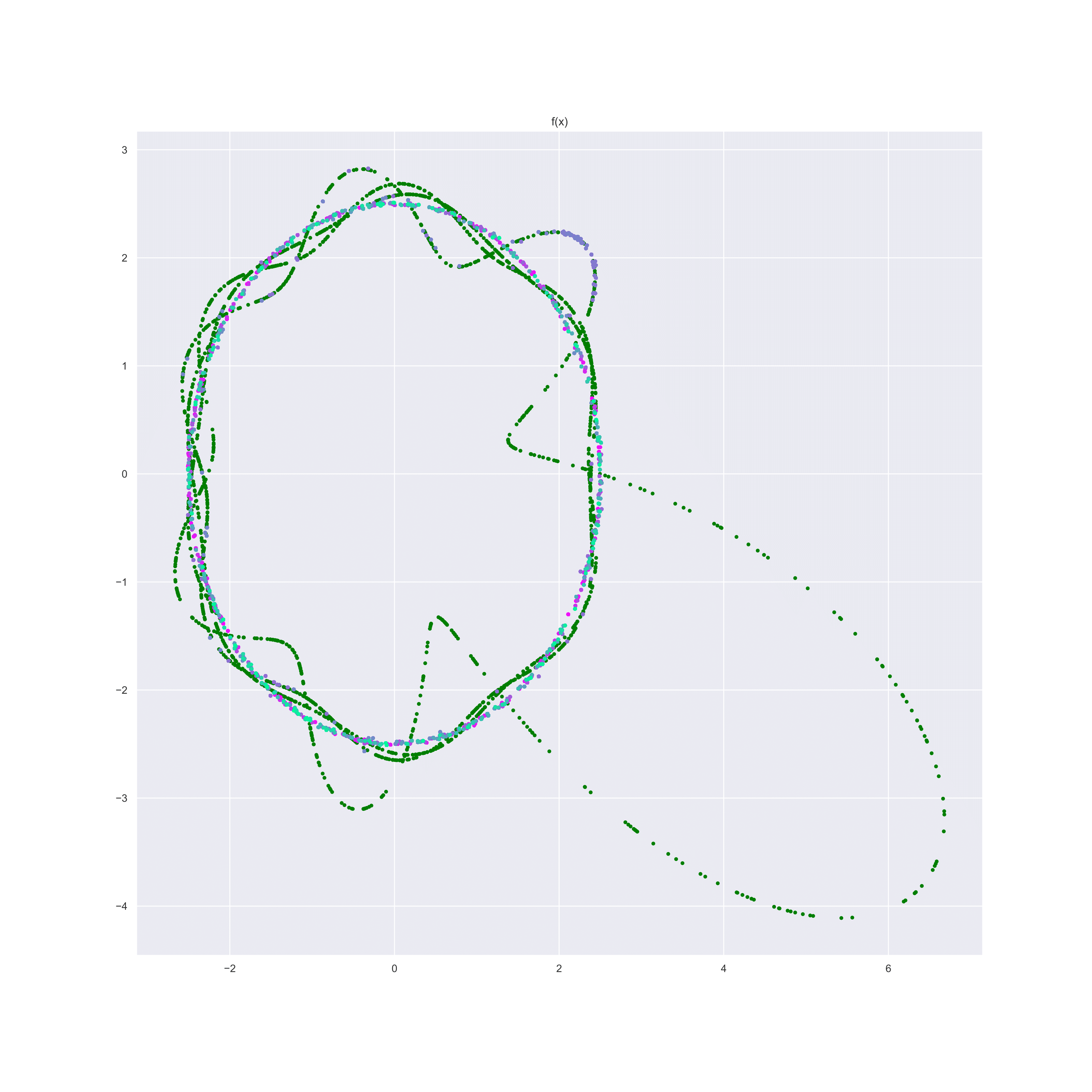}}
  \captionof{figure}{$x\mapsto f(x)$.}
  \label{fig_x_fx_NC_Variety}
\end{minipage}
\begin{minipage}{0.5\textwidth}
  \centering
  \centerline{\includegraphics[scale=0.15]{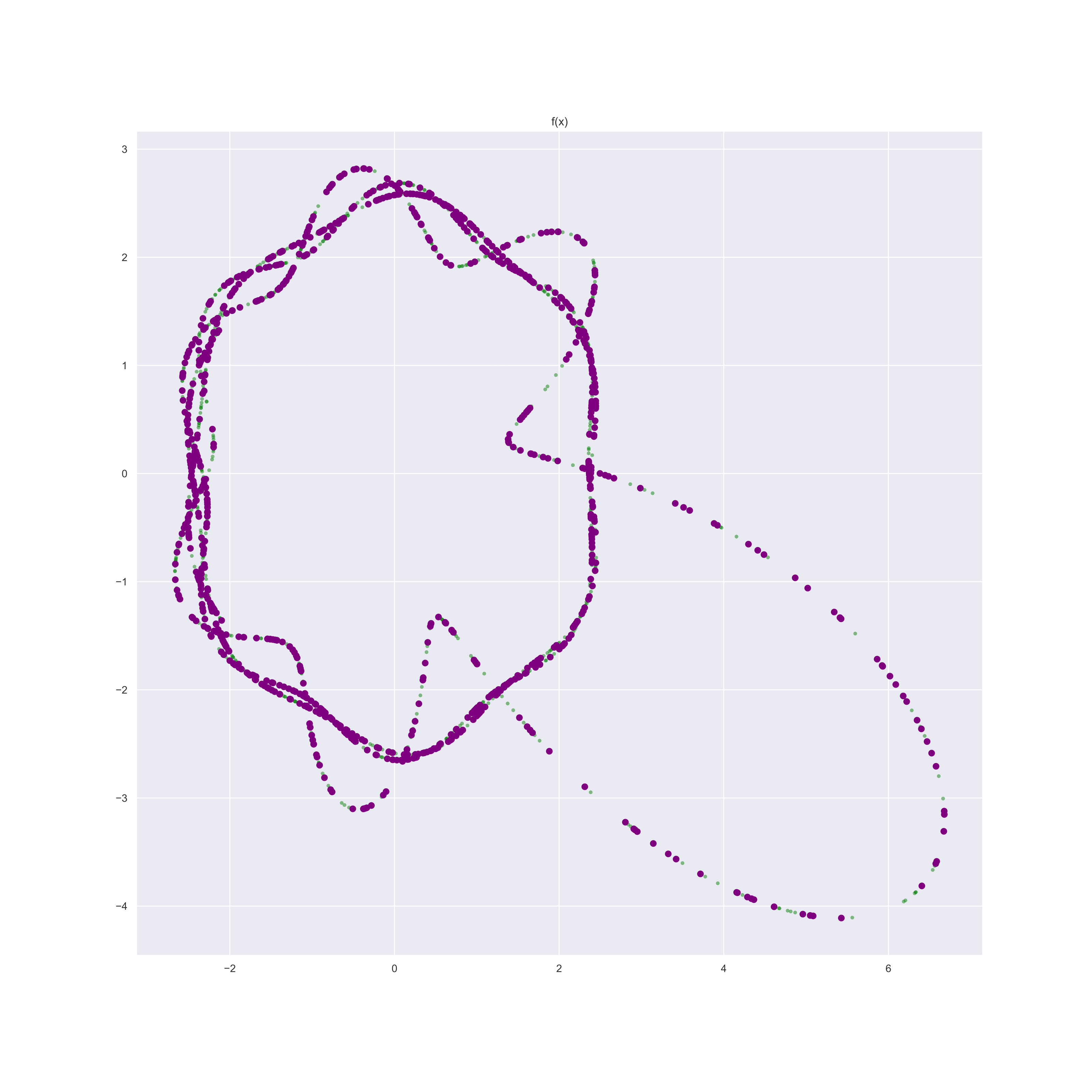}}
  \captionof{figure}{Particles $(Y)$ on Constraint Set $(K)$.}
  \label{fig_Geometry_and_Particels_NC_Variety}
\end{minipage}
}

We conclude our study examining the impact of {\color{darkgreen}{$K$}}'s geometry on the probabilistic transformer's performance by noting that the probabilistic transformer's relative performance is analogous to its performance in the previous experiments. Figure~\ref{fig_Frontier_NC_Variety} and Table~\ref{tab_NC_Variety} reaffirm that the probabilistic transformer outperforms the MLP and the transformer network when the mixed objective of optimizing the MSE and the distance to the {\color{darkgreen}{constraint set}}.

{\centering
\begin{minipage}{0.45\textwidth}
  \centering
  \centerline{\includegraphics[scale=0.11]{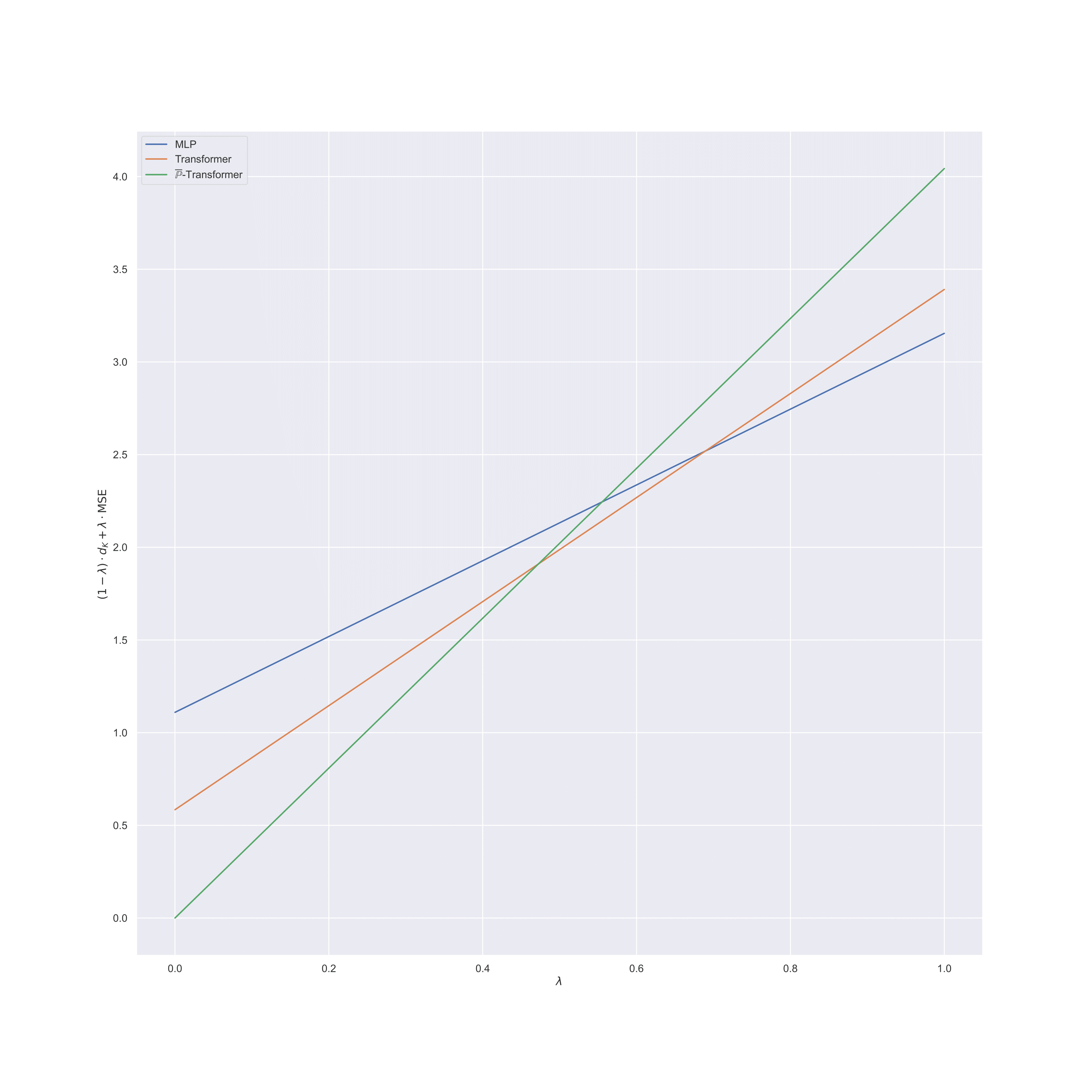}}
  \captionof{figure}{Performance: MSE vs. $d_K$.}
  \label{fig_Frontier_NC_Variety}
\end{minipage}
\begin{minipage}{0.55\textwidth}
  \centering
  \begin{adjustbox}{width=\columnwidth,center}
  \begin{tabular}{lrrr}
    \toprule
    {} &      $\frac{\mbox{MSE}}{\mbox{MSE MLP}}$ &      MSE &      $d_K$ \\
    \midrule
    {\color{darkcerulean}{MLP}}        & 1 & 3.15e+00 & 1.11e+00 \\
    {\color{BurntOrange}{Trans.}}     & 1.01 & 3.39e+00 & 5.84e-01 \\
    {\color{forestgreen}{$\pp\mbox{-Trans.}$}} & 1.09 & 4.04e+00 & 0.00e+00 \\
    \bottomrule
    \end{tabular}
    \label{tab_NC_Variety}
    \end{adjustbox}
  \captionof{table}{Performance Metrics}
\end{minipage}
}

This appendix showed that the probabilistic transformer is implementable, that it can indeed approximate functions while exactly encoding {\color{darkgreen}{constraints}}, and that its performance doesn't degrade for more complicated geometries.  In conclusion, probabilistic transformers can generically and canonically encode geometric priors without sacrificing the expressibility of more familiar deep learning models.

}} 

\section{Proofs}\label{s_Appendix_Proofs}
In what follows, we denote the set of couplings of two probability measures $\mu,\nu\in \mathcal{P}_1(\rr^n)$ on $\rr^n$ by $\operatorname{Cpl}(\mu,\nu)$.  I.e. these are Borel product measures $\pi$ on $\rr^n\times \rr^n$ with respective marginals $\mu$ and $\nu$.  We begin by deriving some useful lemmata.
\subsection{Lemmata}\label{a_ss_lemmata}
This section records lemmata that will be frequently be used throughout this paper's proofs.  The lemmata's proofs are deferred until Section~\ref{a_sss_proofs_of_lemmata} of this Appendix.  
\begin{notee}
\label{remark_GDL_Skip_2}
{\textbf{For the reader interested in convex constraints:}}
We recognize that the results where $K$ is convex follow the more general results where $K$ is a geodesically convex subset of some embedded submanifold of $\rr^m$.  Nevertheless, so as to provide a self-contained reading to those focused on classical transformers or on convex constraints, independent proofs for both of these two cases.
\end{notee}
\subsubsection{{Lemmata in the case where K is convex}}
\label{a_ss_lemmata_sss_convex}
The results are especially useful for results pertaining to convex constraint sets.  
\begin{lemma}[Collapsing Measure-Valued Estimates for Convex Constraint Sets]\label{lem_convex_barycenter_map}
Let $K\subseteq \rr^m$ be non-empty, compact, and convex. Let $F\in C(\rr^n,\mathcal{P}_1(K))$ and $f\in C(\rr^n,K)$.  For every $x \in \rr^n$, the following hold:
\begin{enumerate}
    \item[(i)]\textbf{Convex Constraints Hold:}
    $
    \mathbb{E}_{Y\sim F(x)}[Y]\in K
    ,
    $
    \item[(ii)]\textbf{Non-Expansive Distance:} $
    \left\|
        f(x)
            -
        \mathbb{E}_{Y\sim F(x)}[Y]
    \right\|
        \leq 
    \www_1\left(\delta_{f(x)},F(x)\right)
    .
    $
\end{enumerate}
Moreover, let $\epsilon>0$ and some non-empty compact $C\subset \rr^n$ be non-empty and compact.  If $\max_{x\in C}\,\www_1(\delta_{f(x)},F(x))\leq \epsilon$ then, in addition we have that:
\begin{equation}
    \underset{x\in C}{\max}\,
    \left\|
        f(x)
            -
        \mathbb{E}_{Y\sim F(x)}[Y]
    \right\|
        \leq
    \epsilon
\label{eq_lem_convex_barycenter_map_estimate_derivation}
.
\end{equation}
\end{lemma}
The next lemma, though immediate, is still helpful to write down explicitly as it clearly relates $\operatorname{P-attention}$ to $\operatorname{Attention}$.  
For any $N\in \nn_+$, we denote the standard $N$-simplex by $\Delta_N$; i.e.:
\[
\Delta_N\triangleq \left\{w\in [0,1]^N:\, \sum_{n=1}^N w_n =1\right\}
.
\]
\begin{lemma}[An Identity: P-attention as implicit Attention]\label{lem_attention_to_Attention}
Let $\{y_{n,q}\}_{n=1,\dots,N,q=1,\dots,Q}\subseteq K\subset \rr^m$, let $Y$ be an $N\times 1\times m$-array with $Y_{n}=Q^{-1}\sum_{q=1}^Q y_{n,q}$, and let $f(x)\in C(\rr^n,\Delta_N)$. Then:
\[
\operatorname{Attention}\left(
f(x)|Y
\right)
 =
\mathbb{E}_{X\sim \operatorname{P-attention}(F(x),
Y)}\left[
X
\right]
.
\]
\end{lemma}
\subsubsection{{Lemmata in the case where K is a closed geodesic ball}}
\label{a_ss_lemmata_sss_geodesicallyconvex}
We now consider the analogue of Lemma~\ref{lem_convex_barycenter_map}, in the case where $K$ is geodesically convex of controlled radius\footnote{
We use the terminology controlled in direct analogy with \citep[Theorem 10]{kratsios2021_GDL}.
}.  A $K$ subset of $(M,g)$ is called \textit{geodesically convex} if for every two points $y_1,y_2\in K$ there is a unique geodesic (Riemannian distance minimizing curve) joining $y_1$ to $y_2$.  
\begin{lemma}\label{lem_cbm}
Let $(M,g)$ be a connected Riemannian manifold with sectional curvatures uniformly bounded-above by $C\geq 0$ and which is complete as a metric space under $d_g$.  Fix $y_0\in M$, 
\begin{equation}
0<\rho<2^{-1}\min\left\{
    \operatorname{inj}_{g}(y_0)
            ,
    \frac{\pi}{\sqrt{C}}
\right\}
\label{eq_lem_cbm_geodesic_regularity_condition_Statement}
\end{equation}
(where, following \cite{Afsari_RiemannianCentersOfMassAMS2011}, $\frac1{\sqrt{C}}\triangleq \infty$ whenever $C\leq 0$), and let $K$ be a non-empty, compact, and geodesically convex subset of $\overline{B(y_0,\rho)}$.  Then, the ``Fr\'{e}chet mean'' function:
\begin{equation}
\begin{aligned}
\mathcal{P}_{1}\left(K\right) & \rightarrow \overline{
B(y_0,\rho)
}\\
\mathbb{P} & \mapsto \operatorname{argmin}_{y \in K}\, 
\mathbb{E}_{Y\sim \mathbb{P}}\left[
    d_g^{2}\left(
        y   
    ,
        Y
    \right)
\right],
\end{aligned}
\label{eq_lem_cbm_definition_of_intrinsic_mean_map}
\end{equation}
is a well-defined (i.e.: single-valued) and non-expansive (i.e.: $1$-Lipschitz) function.  Furthermore, if $\mathbb{P}$ is finitely-supported then:
\begin{equation}
    \overline{\mathbb{P}} \in K
    \label{eq_lem_cbm_inclusion}
    .
\end{equation}
\end{lemma}
\subsection{Exceptional Closed-Form for Wasserstein Distance}\label{a_Lemmata_ss_Auxiliary_lemmata}
The following result is folklore in the optimal transport community.  Since its statement is difficult to track down, we record the statement and derive its proof here, for a self-contained reading. 
\begin{lemma}[Closed-Form Expression for Wasserstein Distance to Pointmass]\label{lem_closedform_wasserstein_dirac}
Let $K\subseteq \rr^m$ be non-empty and compact, let $y$ be in $K$, and let $\pp\in \ppp{K}$.  Then:
\[
\www_1(\pp,\delta_y) = \ee_{Y\sim \pp}[\|Y-y\|]
.
\]
\end{lemma}
\subsubsection{Proofs of Lemmata}\label{a_sss_proofs_of_lemmata}
\begin{proof}[{Proof of Lemma~\ref{lem_convex_barycenter_map}}]
Fix $\mu\triangleq F(x)$.  We first show (ii). Since $\rr^n$ is a Banach space, then \citep{BruHeinicheLootgieter1993} implies that there exists a unique contracting barycenter map on $\mathcal{P}_1(\rr^n)$; i.e.: a Lipschitz map $\beta_{\rr^n}:\mathcal{P}_1(\rr^n)\rightarrow\rr^n$ satisfying $\beta_{\rr^n}(\delta_x)=x$ for all $x \in \rr^n$, whose Lipschitz constant $\operatorname{Lip}(\beta_{\rr^n})$ is at most $1$.  Moreover, the result guarantees that the barycenter map is linear an given by the Bochner integral (i.e. the usual vector-valued expectation of a $\rr^n$-valued random-vector):
\begin{equation}
    \beta_{\rr^n}:\mathcal{P}_1(\rr^n)\ni \mu
        \mapsto
    \mathbb{E}_{Y\sim \mu}\left[Y\right]
        \in \rr^n
    \label{eq_cor_Deep_BMT_Convex_Case_barycenter_map}
    .
\end{equation}
Therefore, we conclude that:
\begin{equation}
    \begin{aligned}
    \|f(x)-\mathbb{E}_{Y\sim \mu}\left[Y\right]\|
    =&
\|f(x)-\beta_{\rr^n}(F(x))\|
\\
= &\|\beta_{\rr^n}(\delta_{f(x)}-\beta_{\rr^n}(F(x))\|
\\
\leq & \www_1\left(\delta_{f(x)},F(x)\right)
.
\end{aligned}
\label{eq_lemma_convex_first_derivation}
\end{equation}
This gives (ii).  
Furthermore, if the right-hand side of~\eqref{eq_lemma_convex_first_derivation} is upper-bounded by a constant $\epsilon>0$, uniformly over $C$, then so must be the left-hand side.  This gives~\eqref{eq_lem_convex_barycenter_map_estimate_derivation}.  

We now show (i).  Since $K\subset \rr^n$, we may view $\mathcal{P}_1(K)$ as a subspace of $\mathcal{P}_1(\rr^n)$.  Thus, $\beta_{\rr^n}|_{\mathcal{P}_1(K)}$ satisfies $\beta_{\rr^n}(\delta_x)=x$ for all $x\in K$.  Moreover, we may 
view~\eqref{eq_cor_Deep_BMT_Convex_Case_barycenter_map} as a map on $\mathcal{P}_1(K)$.  Therefore, if $\mu \in \mathcal{P}_1(K)$ then, any $\rr^n$-valued random-vector $Y$ with law $\mu$, by definition, $\mu$-a.s. takes values in $K$.  
Since $K$ is convex, and $\mu\in \mathcal{P}_1(K)$ (i.e. $\mathbb{E}_{Y\sim \mu}[|Y|]<\infty$) then the formulation of Jensen's inequality given in \citep[Theorem 10.2.6]{DudleyRealAnalysisandPRobabilityBook1989Revised2002} guarantees that
\begin{equation}
    \mathbb{E}_{Y\sim \mu}[Y]\in K
    \label{eq_cor_Deep_BMT_Convex_Case_barycenter_map_constraints_satisfaction}
    .
\end{equation}
Hence, we may refine~\eqref{eq_cor_Deep_BMT_Convex_Case_barycenter_map} to state: $\beta_{\rr^n}|_{\mathcal{P}_1(K)}$ is $1$-Lipschitz and satisfies
\begin{equation}
    \beta_{\rr^n}|_{\mathcal{P}_1(K)}:\mathcal{P}_1(K)\ni \mu
        \mapsto
    \mathbb{E}_{Y\sim \mu}\left[Y\right]
        \in K
    \label{eq_cor_Deep_BMT_Convex_Case_barycenter_map_refined_convex_constraints}
    .
\end{equation}
Thus (i) holds.  
\end{proof}
\begin{proof}[{Proof of Lemma~\ref{lem_attention_to_Attention}}]
Follows directly from the linearity of integration and the fact that integration against a pointmass is just point-evaluation.  
\end{proof}
\begin{proof}[{Proof of Lemma~\ref{lem_closedform_wasserstein_dirac}}]
By definition of the Wasserstein distance between $\pp$ and $\delta_y$ we have that: 
\begin{equation}
    \www_1(\pp,\delta_y)=\inf_{\pi \in \operatorname{Cpl}(\pp,\delta_y)}\, \ee_{(Y_1,Y_2)\sim \pi}[\|Y_1-Y_2\|]
\label{PROOF_lem_closedform_wasserstein_dirac_definition_of_Wasserstein}
.
\end{equation}
Since $\pp\otimes \delta_y\in \operatorname{Cpl}(\pp,\delta_y)$ (e.g. see \citep[Page 6]{Villani2009optimal}) then it is enough to show that if $\pi$ is a coupling in $\operatorname{Cpl}(\pp,\delta_y)$ then $\pi=\pp\otimes \delta_{y}$.  We show this now.  

Let $B_1,B_2\subseteq K$ be Borel and let $\pi \in \operatorname{Cpl}(\pp,\delta_y)$.  If $y\in B_2$, then
$\pp(B_1)=\pi(B_1\times K)\geq \pi(B_1\times B_2)\geq \pi(B_1\times \{y\})$.  Therefore, $1-\pp(B_1)\geq \pi(K\times \{y\})-\pi(B\times \{y\})$; thus, $\pi(B_1\times \{y\})\leq \pp(B_1)$.  Therefore, 
$\pp(B_1)\leq \pi(B_1\times \{y\})\leq \pi(B_1\times B_2)$; whence, $\pi(B_1\times B_2)=\nu(B_2)\delta_{y}(B_1)\overset{\mbox{(def)}}{=}\nu\otimes\delta_y (B_2\times B_1)$. Now, suppose that $y\not\in B_2$ then, $\pi(B_1\times B_2)\leq \pi(K\times B_2)=\delta_y(B_2)=0$.  We have show that $\pi=\nu\otimes \delta_y$.  Hence,~\eqref{PROOF_lem_closedform_wasserstein_dirac_definition_of_Wasserstein} reduces to:
    \begin{align}
    \nonumber
         \www_1(\pp,\delta_y)
         = & 
         \inf_{\pi \in \operatorname{Cpl}(\pp,\delta_y)}\, \ee_{(Y_1,Y_2)\sim \pi}[\|Y_1-Y_2\|]
    \\
    \nonumber
    = &  
    \ee_{(Y_1,Y_2)\sim \pp\otimes \delta_y}[\|Y_1-Y_2\|]
    \\ 
    = & 
    \label{PROOF_lem_closedform_wasserstein_dirac_fubini}
    \ee_{Y_1\sim \pp}[\ee_{Y_2\sim \delta_y}[\|Y_1-Y_2\|]]
    \\ 
    = &
    \label{PROOF_lem_closedform_wasserstein_dirac_definition_pointmass}
    \ee_{Y_1\sim \pp}[\|Y_1-y\|]
    ;
\end{align}
where we have applied the Fubini-Tonelli Theorem (see \citep[Theorem 1.27]{FoundationsofModernProbabilityKallenberg3rd2021}) in~\eqref{PROOF_lem_closedform_wasserstein_dirac_fubini} and the definition of a pointmass to derive~\eqref{PROOF_lem_closedform_wasserstein_dirac_definition_pointmass}.  
\end{proof}
\begin{proof}[{Proof of Lemma~\ref{lem_cbm}}]
We first observe that, since $(M,g)$ is connected and complete as a metric space, then by the Hopf-Rinow Theorem (\citep[Theorem 1.7.1]{jost2008riemannian}) $(M,g)$ is a complete Riemannian manifold (sometimes also called a geodesically complete Riemannian manifold; see \citep[Definition 1.7.1]{jost2008riemannian}.  

Fix a $\mathbb{P}\in \mathcal{P}_1(K)$.  Since $K$ is compact, then $\mathbb{P}\in \mathcal{P}_2(K)$.  The completeness of $(M,g)$ (as a Riemannian manifold) and the facts that $K$ is a non-empty geodesically convex subset of $B(y_00,\rho)$ (where $\rho$ satisfies~\eqref{eq_lem_cbm_geodesic_regularity_condition_Statement}) implies that the conditions of \citep[Theorem 2.1]{Afsari_RiemannianCentersOfMassAMS2011} are met; whence, the ``Fr\'{e}chet mean function'' of~\eqref{eq_lem_cbm_definition_of_intrinsic_mean_map} is well-defined function from $\ppp{K}$ to $\overline{B(y_0,\rho)}$. 
It remains to show that it is $1$-Lipschitz.  Since $\rho<\operatorname{inj}_{g}(y_0)$, then the remark on \citep[Page 299]{jost2008riemannian} implies that \citep[Theorem 6.9.2]{jost2008riemannian} holds; therefore, for any $y_1,y_2,y_3\in B(p,\rho)$ the map:
\[
[0,1]\ni t \mapsto d^2(\gamma_{[y_1,y_2]}(t),\gamma_{[y_1,y_3]}(t)) \in [0,\infty)
,
\]
is convex. We may now conclude our proof by arguing analogously to
\citep[Theorem 6.3's proof]{SturmOGPaper2003ProbabilityonNPCSpaces}
.  Fix $\mathbb{P},\mathbb{Q}\in \mathcal{P}_{1}(K)$ and let $\pi\in \mathcal{P}(K\times K)$ with marginals $\mathbb{P}$ and $\mathbb{Q}$.  Then, applying Jensen's inequality, we have that:
\begin{equation}
    d_g(\bar{P},\bar{Q}) 
    \leq 
\int \, d^2(y_1,y_2)\, \pi(d(y_1,y_2))
\label{eq_lem_proof_wasserstein_bound}
.
\end{equation}
Since we have just showed that right-hand side of~\eqref{eq_lem_proof_wasserstein_bound} holds for any such $\pi$. Consequently, taking the infimum over all such $\pi$ implies that:
\[
d_g(\bar{P},\bar{Q}) 
    \leq 
\www_1(\mathbb{P},\mathbb{Q})
.
\]
Thus,~\eqref{eq_lem_cbm_definition_of_intrinsic_mean_map} is $1$-Lipschitz.  

For the last claim, suppose that $\mathbb{P}\in \mathcal{P}_1(K)$ is finitely supported.  Since $K$ is geodesically convex and since $\mathbb{P}$ is finitely supported then \citep[Theorem 3.4 (i)]{Afsari_RiemannianCentersOfMassAMS2011} implies that $\bar{\mathbb{P}}$ is an element of the smallest closed geodesically convex subset $C_{\mathbb{P}}$ containing the support of $\mathbb{P}$; since $K$ itself is itself closed and geodesically convex then we infer that $C_{\mathbb{P}}\subseteq K$.  Thus,~\eqref{eq_lem_cbm_definition_of_intrinsic_mean_map} takes values in $K$.  
\end{proof}
\subsection{{Proof of Theorem~\ref{theorem_DeepBMT}}}\label{s_Appendix_ss_proof_main_theorem}
\begin{proof}[{Proof of Theorem~\ref{theorem_DeepBMT}}]
Since $K\subseteq \rr^n$ is non-empty and compact, for each $x \in \rr^n$ the set $C_x$ is closed and has non-empty intersection with $K$, thus each $C_x\cap K$ is compact.  Thus, the map $\varphi:\rr^n\ni x \mapsto C_x\cap K\in 2^{\rr^m}$ is a non-empty and compact-valued multifunction.  Moreover, by \citep[18.4 Lemma]{InfiniteHitchhiker2006} $\varphi$ is weakly-measurable since $C$ is and so is the correspondence $\rr^n\ni x\mapsto K\in 2^{\rr^m}$.  Thus, the hemicontinuity of $\varphi$ and the assumptions made on $L$ are such that the Measurable Maximum Theorem (see \citep[Theorem 18.19]{InfiniteHitchhiker2006}) applies; whence, the ``optimality'' sets
\begin{equation}
    \mathcal{O}(x)\triangleq \operatorname{argmax}_{y\in\varphi(x)}\, -L(x,y) \in \rr
=
\operatorname{argmin}_{y\in C_x\cap K}\, L(x,y)
\label{PROOF_theorem_DeepBMT_eq_optimality_set}
,
\end{equation}
are a well-defined for each $x\in \mathbb{R}^n$ and, there exists a Borel measurable function $S:\rr^n\rightarrow \rr^m$ satisfying the ``optimal selection condition'':
\begin{equation}
    S(x)\in \mathcal{O}(x) \qquad (\forall x \in \rr^n)
    \label{eq_proof_theorem_DeepBMT_1_optimal_selector}
    .
\end{equation}

Since $\rr^n$ is a complete and separable metric space and since $\pp$ is a Borel probability measure on $\rr^n$, then by \citep[Theorem 13.6]{klenke2013probability}, $\pp$ is a Radon measure on $\rr^n$.  Since $S$ is Borel measurable, $\rr^n$ and $\rr^m$ are locally-compact and second-countable topological spaces, and since $\mathbb{P}$ is a Radon measure on $\rr^n$, then Lusin's Theorem (as formulated in \citep[Excersize 13.1.3]{klenke2013probability}) implies that, for every $\epsilon\in (0,1]$, there is a compact subset ${\xxx_{\epsilon}}\subseteq \rr^n$ on which 
$S|_{{\xxx_{\epsilon}}}$ is continuous and $\pp({\xxx_{\epsilon}})\geq 1-\epsilon$.  
By \citep[Point 5 - Page 99]{Villani2009optimal}, the map $\rr^m:y\mapsto \delta_y\in \mathcal{P}(\rr^m)$ is an isometry.  In particular, the map $\rr^m:y\mapsto \delta_y\in \mathcal{P}(\rr^m)$ is continuous.  Hence, $S^{\star}:{\xxx_{\epsilon}}\ni x \mapsto \delta_{S(x)}\in \mathcal{P}_1(\rr^m)$ is continuous.  However, by construction, $S(x)\in \varphi(x)\subseteq K$; thus, $S^{\star}$ defines a map with codomain $\mathcal{P}_1(K)$.  

Therefore, \citep[Theorem 3]{kratsios2021_GCDs} implies that there exists an $\hat{F}$ of the form
\begin{equation}
    \hat{F}:\mathbb{R}^n\ni x\mapsto 
\sum_{n=1}^N [\operatorname{Softmax}_N(\hat{f}(x))]_n \frac1{Q}\sum_{q=1}^Q \delta_{k_{n,q}}
\in \mathcal{P}_1(\rr^m)
\label{eq_proof_theorem_DeepBMT_1b_optimal_selector_reduced_form__OG_form}
,
\end{equation}
satisfying:
\begin{equation}
    \max_{x\in {\xxx_{\epsilon}}}\, 
        \www_1\left(
                \hat{F}(x)
            ,
                S^{\star}(x)
        \right)
            <
        \epsilon
    \label{eq_proof_theorem_DeepBMT_2_application_UAP}
    .
\end{equation}
Grouping the sums $\sum_{n=1}^N$ and $\sum_{q=1}^Q$ and the weights $Q^{-1}[\operatorname{Softmax}_N(\hat{f}(x))_n]$ in~\eqref{eq_proof_theorem_DeepBMT_1b_optimal_selector_reduced_form__OG_form}, we may rewrite~\eqref{eq_proof_theorem_DeepBMT_1b_optimal_selector_reduced_form__OG_form} in the form~\eqref{eq_theorem_deepBMT_Model_form}.

By construction, for each $x \in {\xxx_{\epsilon}}$ we have that $S(x)\in
\mathcal{O}(x)$.  Thus,~\eqref{eq_proof_theorem_DeepBMT_1_optimal_selector} implies that: 
\begin{equation}
      \max_{x\in {\xxx_{\epsilon}}}\, 
        \www_1\left(
            \hat{F}(x)
                ,
        \underset{y^{\star}\in
\mathcal{O}(x)}{\inf}\,
        \delta_y^{\star}
        \right)
        \leq 
    \max_{x\in {\xxx_{\epsilon}}}\, 
        \www_1\left(
                \hat{F}(x)
            ,
                S^{\star}(x)
        \right)
        <\epsilon
        \label{eq_proof_theorem_DeepBMT_2_main_upperbound_abstract}
        .
\end{equation}
This gives (ii).  

Now, by construction, each $y_1,\dots,y_{N,Q}\in K$.  Therefore, for each $x \in \rr^n$, $\hat{F}(x)$ is supported in $K$ and, moreover, $\hat{F}(x)\in \ppp{K}$.  Thus, (i) holds.  
\end{proof}

\subsection{{Proof of Theorem~\ref{thrm_quantitativer_version_NonConvex}}}
We make use of the following notation during Theorem~\ref{thrm_quantitativer_version_NonConvex}'s proof.  For $d\leq m,\, d\in \nn_+$, we denote $p_d^m:\rr^m\ni (x_1,\dots,x_m)\rightarrow (x_1,\dots,x_n)\in \rr^d$ and similarly, $\iota_d^m:\rr^d\ni (x_1,\dots,x_n)\mapsto (x_1,\dots,x_n,0,\dots,0)\in \rr^m$.  Before proceeding, we also emphasize the following identities: if $x_1,\dots,x_n\in \rr$ then $\iota_d^m\circ p_d^m(x_1,\dots,x_n,0,\dots,0)=(x_1,\dots,x_n,0,\dots,0)$ and conversely, $p_m^d\circ \iota_d^m$ is the identity on $\rr^d$.  
\begin{proof}[{Proof of Theorem~\ref{thrm_quantitativer_version_NonConvex}}]
Let $k\in \nn_+$, let $\xxx\subseteq [0,1]^n$ be non-empty, and let $f\in C_{tr}^k(\xxx,K)$.  Since $f\in C_{tr}^k(\xxx,K)$ then, there exists a $k$-times continuously differentiable $\boldsymbol{f}:\rr^n\rightarrow \rr^m$ such that: for every $x \in \xxx$, we have that:
\begin{equation}
    \boldsymbol{f}(x)=f(x)
    \label{PROOF_thrm_quantitativer_version_NonConvex__eq_extension}
    .
\end{equation}
Note that $\xxx\subseteq [0,1]^n$.
We begin by building our encoder to approximate $p_d^m\circ \Phi^{-1}\circ \boldsymbol{f}\in C([0,1]^n,\rr^m)$.
By Assumption~\ref{ass_normcontrolled_loss} and the fact that $f(\xxx)\subseteq K$ we have that, for each $x\in \xxx$:
\[
\begin{aligned}
0\leq \inf_{y\in K}\, L(x,y)\leq 
L(x,f(x))
\leq 
l\left(
\|f(x)-f(x)\|
\right)
=l(0)=0.
\end{aligned}
\]
Therefore, by~\eqref{PROOF_thrm_quantitativer_version_NonConvex__eq_extension}, for each $x \in \xxx$, 
we know that $
    \{\boldsymbol{f}(x)\}
        \subseteq 
    \underset{y\in K}{\operatorname{argmin}}\,
        L(x,y)
    .
$  
In particular, for each $x \in \xxx$ we have that $\underset{y\in K}{\operatorname{argmin}}\, L(x,y)$ is non-empty and therefore, for any $\hat{F}:[0,1]^n\rightarrow \mathcal{P}_1(K)$ we may compute:
\begin{equation}
  \begin{aligned}
    \sup_{x\in \xxx}\,
        \www_1\left(
            \hat{F}(x)
                ,
            \underset{y\in K}{\operatorname{argmin}}\,
        L(x,y)
        \right)
            \overset{\operatorname{(def)}}{=} &
    \sup_{x\in \xxx}
    \inf_{y^{\star} \in \underset{y\in K}{\operatorname{argmin}}\,L(x,y)}
    \,
        \www_1\left(
            \hat{F}(x)
                ,
            \delta_{y^{\star}}
        \right)
            \\
            \leq  &
    \sup_{x\in \xxx}\,
        \www_1\left(
            \hat{F}(x)
                ,
            \delta_{f(x)}
        \right)
        \\
    = &
    \sup_{x\in \xxx}\,
        \www_1\left(
            \hat{F}(x)
                ,
            \delta_{\boldsymbol{f}(x)}
        \right)
        .
  \end{aligned}
    \label{eq_thrm_quantitativer_version_NonConvex____unique_minimizer_is_fx}
\end{equation}
Therefore, it is enough to construct models $\hat{\mathcal{D}}$ and $\hat{\mathcal{E}}$ such that the composite model $\hat{F}=\hat{\mathcal{D}}\circ \hat{\mathcal{E}}$ controls the approximation error on the right-hand side of~\eqref{eq_thrm_quantitativer_version_NonConvex____unique_minimizer_is_fx}.  The remainder and bulk of the proof is devoted to precisely this task.

NB, by Assumption~\ref{ass_low_dimensional_mainfold} we have that $\boldsymbol{f}(x)=
\Phi^{-1}\circ \iota_d^m \circ (p_d^m\circ \Phi \circ \boldsymbol{f})(x)$.  Now, since $p^m_d$ is a linear map between finite-dimensional normed spaces then, is is analytic, and therefore it is smooth.  Moreover, by hypothesis, both $\Phi$ and $\Phi^{-1}$ are both also smooth.  Thus, the map $p^m_d\circ \Phi\circ \boldsymbol{f}:\rr^n\rightarrow \rr^d$ is $k$-times continuously differentiable.  

Since $p^m_d\circ \boldsymbol{f}$ is $k$-times continuously differentiable then, by \citep[Proposition 17]{kratsios2021_GDL}, $[0,1]^n$ is efficient for $p^m_d\circ \Phi\circ \boldsymbol{f}$ (in the sense of \citep[Definition 16]{kratsios2021_GDL}); thus, \citep[Corollary 43]{kratsios2021_GDL} (activation function parameter $\alpha$ set to $\alpha=0$.  Thus, $\sigma_0$ is non-affine, continuous, and piecewise linear) implies that there is a $\hat{\mathcal{E}}\in \NN[n,d][\sigma_0]$ satisfying:
\begin{equation}
    \sup_{x\in [0,1]^n}\, 
    \|p_d^m\circ\Phi\circ \boldsymbol{f}(x)
        -
    \hat{\mathcal{E}}(x)\|<\epsilon_f
    \label{eq_thrm_quantitativer_version_NonConvex____eq_first_estimate}
    ,
\end{equation}
Furthermore, $\hat{\mathcal{E}}$ also satisfies the following quantitative estimates:
\begin{enumerate}
    \item[(i-$\mathscr{E}$)] $\hat{\mathcal{E}}$ has width $d\leq W\leq d(4n+10)$,
    \item[(ii-$\mathscr{E}$)] $\hat{\mathcal{E}}$ has depth of the order $\mathscr{O}(d+d\epsilon_f^{\frac{2n}{3(kn+1)}-\frac{2n}{kn+1}})$,
    \item[(iii-$\mathscr{E}$)] The number of trainable parameters determining $\hat{\mathcal{E}}$ are of the order $\mathscr{O}(
        d(d^2 +1)\epsilon_f^{-\frac{2n}{3(kn+1)}}
    )$.
\end{enumerate}
Since $\xxx\subseteq [0,1]^n$ and $f(\xxx)=\boldsymbol{f}(\xxx)\subseteq K$ then, Assumption~\ref{ass_low_dimensional_mainfold} implies that $p_d^m\circ \Phi\circ f(\xxx)\subseteq p_d^m(\Phi(K))\subseteq \rr^d$.  This together with~\eqref{eq_thrm_quantitativer_version_NonConvex____eq_first_estimate} implies that:
\begin{equation}
    \begin{aligned}
     \sup_{x\in \xxx}\,
    \|\hat{\mathcal{E}}(x)
        -
    p_d^m(\Phi(K))\|
    \overset{\operatorname{(def)}}{=} &
\sup_{x\in \xxx}
\inf_{y \in p_d^m(\Phi(K))}\,\|\hat{\mathcal{E}}(x)-y\|
\\
    \leq &
\sup_{x\in \xxx}\inf_{y \in p_d^m\circ\Phi\circ f(\xxx)}\,\|\hat{\mathcal{E}}(x)-y\|
\\
\leq &
\sup_{x\in \xxx}\,\|\hat{\mathcal{E}}(x)-p_d^m\circ\Phi\circ f(x)\|
\\
= &
\sup_{x\in \xxx}\,
\|
    \hat{\mathcal{E}}(x)
        -
    p_d^m\circ\Phi\circ \boldsymbol{f}(x)
\|
\\
\leq &
\epsilon_f
.
\end{aligned}
        \label{eq_thrm_quantitativer_version_NonConvex____epsilonthickeningof_K}
\end{equation}
Thus,~\eqref{eq_thrm_quantitativer_version_NonConvex____epsilonthickeningof_K} indicates that $\hat{\mathcal{E}}$ need not take values in $p_d^m(\Phi(K))$ but, it does take values in the following closed and bounded subset of $\rr^m$: 
\[
    \Phi(K)_{\epsilon_f}
        \triangleq 
    \left\{
        y \in \rr^m:\,\|y-p_d^m(\Phi(K))\|\leq \epsilon_f
    \right\}
\]
By \citep[Theorem 26.5]{munkres2014topology}, $\Phi(K)$ is compact since $K$ is compact and since $\Phi$ is continuous.  Thus, the Heine-Borel Theorem (see \citep[Theorem 27.3]{munkres2014topology}) implies that $\Phi(K)_{\epsilon_f}$ is compact as it is closed and bounded (because $\Phi(K)$ is closed and bounded).  Thus, we can approximate functions from $p_d^m(\Phi(K))_{\epsilon_f}$ to $\ppp{K}$ uniformly using the main result of \cite{kratsios2021_GCDs}.  Specifically, we will approximate a random project 
(in the sense %
    \footnote{The author of this first paper on the subject calls such maps Lipschitz stochastic retracts.  The terminology ``random projection'' was later adopted by other authors, such as \cite{AmbriosioPuglisi2020RandomProjections} and \cite{Brue2021Extension} in connection with the work of \cite{LeeNaor2005RandomProjectionsforLipExtensions} and the \cite{JohnsonLindenstraussLemma1984}'s Lemma.}%
of \citep[Definition 3.1]{Ohta2009ExtendingLipandHold}) 
of $\rr^d$ onto $p_d^m(\Phi(K))$, uniformly on the compact subset $\Phi(K)_{\epsilon_f}$ of $\rr^d$.  

To this end, we make the following observation on the bi-Lipschitz regularity of $\Phi$, when restricted to $K\subseteq \rr^m$.
Since $K$ is non-empty and compact, then $\Phi|_K:K\rightarrow \rr^d$ is Lipschitz, as it is at-least once continuously differentiable.  Since $K$ is compact, and $\Phi$ is continuous, then by \citep[Theorem 26.5]{munkres2014topology} $\Phi(K)$ is also compact.  Therefore, since $\Phi^{-1}$ is also at-least once continuously differentiable then, $\Phi^{-1}|_{\Phi(K)_{\epsilon_f}}:\Phi(K)_{\epsilon_f}\rightarrow \rr^m$ is Lipschitz.  Hence, $\Phi|_K:K\rightarrow \Phi(K)\subseteq \rr^d$ is bi-Lipschitz\footnote{
A map $f:\rr^n\rightarrow \rr^m$ is bi-Lipschitz (see \citep[page 78]{heinonen2001lectures}) if there are constants $c,C>0$ such that, for every $x_1,x_2\in \rr^n$ the estimate holds: 
$
c\|x_1-x_2\| \leq \|f(x_1)-f(x_2)\|\leq C\|x_1-x_2\|
.
$
}.  
In particular, $\Phi(K)_{\epsilon_f}$ and $\Phi(K)$ have diameter at-most:
\begin{equation}
\operatorname{diam}(\Phi(K))\leq \operatorname{Lip}(\Phi)\operatorname{diam}(K)
\mbox{ and }
    \operatorname{diam}(\Phi(K)_{\epsilon_f})
    \leq 
\operatorname{Lip}(\Phi)
\operatorname{diam}(K)
+2\epsilon_f
\label{eq_diameter_bound}
.
\end{equation}

We may therefore apply \citep[Theorem 12.1]{heinonen2001lectures}, as $p_d^m(\Phi(K))$ has a (finite) doubling constant $\lambda(p_d^m(\Phi(K)))$ since it is a subset of $\rr^d$.  
More precisely, we have that:
\begin{equation}
    \lambda(p_d^m(\Phi(K)))
        = 
    \lambda(\rr^d) = 2^{d}
    ,
    \label{eq_thrm_quantitativer_version_NonConvex____doublingconstantbound}
\end{equation}
where the first inequality in~\eqref{eq_thrm_quantitativer_version_NonConvex____doublingconstantbound} follows from \citep[Lemma 9.6 (i)]{RobinsonDimensionEmbeddingsandAttactors2011} and the second in~\eqref{eq_thrm_quantitativer_version_NonConvex____doublingconstantbound} from \citep[Lemma 9.2]{RobinsonDimensionEmbeddingsandAttactors2011}.

Therefore, we may apply \citep[Theorem 3.2]{Brue2021Extension} to conclude that there exists a Lipschitz map $\Pi:\rr^d\rightarrow \mathcal{P}_1(p_d^m(\Phi(K)))$ such that, for all $y \in p_d^m(\Phi(K))$, $\Pi$ satisfies:
\begin{equation}
    \Pi_y = \delta_y
    \label{eq_thrm_quantitativer_version_NonConvex____random_projection_facts}
    .
\end{equation}
Moreover, the same result bounds $\Pi$'s Lipschitz constant, denoted by $\operatorname{Lip}(\Pi)$, by $k \log(\lambda(p_d^m(\Phi(K))))$ where, $k$ is an absolute constant; i.e. it does not depend on $\rr^n,$ $\rr^d$, $\epsilon$, or on $\lambda(p_d^m(\Phi(K)))$.  Consequently, $k$ does not depend on $n$, $d$, $\epsilon$, or on $\kappa_K$.  Combining this with~\eqref{eq_thrm_quantitativer_version_NonConvex____doublingconstantbound}, $\Pi$'s Lipschitz constant is bounded as follows:
\begin{equation}
        \operatorname{Lip}(\Pi)\leq k\log_2(\lambda(p_d^m(\Phi(K))))
        \leq 
        k d
    \label{eq_thrm_quantitativer_version_NonConvex____Lipschitz_constant_bound_intermsof_metriccapacity}
    .
\end{equation}
Since $\Pi$ is 1-Lipschitz (continuous), $\Phi(K)_{\epsilon_f}$ is compact and $p_d^m(\Phi(K))$ is compact then by \citep[Theorem 3]{kratsios2021_GCDs}, there exists a $\hat{D}\in \NN[d,N]$ and $y_{1,1},\dots,y_{N,Q}\in p_d^m(\Phi(K))$ such that the ``probabilistic decoder network'' $\hat{\mathcal{D}}_0$ defined by:
\[
\hat{\mathcal{D}}_0:\rr^d\ni x \mapsto \sum_{k=1}^N
\operatorname{P-attention}\left(
    \hat{D}(x)
,
    Y
\right)\in \mathcal{P}_1(p_d^m(\Phi(K)))
;
\]
where, $Y$ is the $N\times Q\times m$-array with $Y_{n,q}=y_{n,q}$, and $\hat{\mathcal{D}}_0$ satisfies: 
\begin{equation}
    \begin{aligned}
        \sup_{
            y\in p_d^m\circ \Phi(K_{\epsilon_f})
        }\,
    \www_1\left(
        \Pi(y)
            ,
        \hat{\mathcal{D}}_0(y)
    \right)\leq \epsilon_K
    ,
    \end{aligned}
    \label{eq_thrm_quantitativer_version_NonConvex____approximation_quality_probabilistic_decoder_prepush};
\end{equation}
where we have set the activation function parameter $\alpha$ to $\alpha=1$.  Thus, $\sigma_{1}$ is smooth and non-polynomial; in which case \citep[Theorem 3 and Example 7]{kratsios2021_GCDs} and the estimates in~\eqref{eq_diameter_bound} and in~\eqref{eq_thrm_quantitativer_version_NonConvex____Lipschitz_constant_bound_intermsof_metriccapacity} also implies $\hat{\mathcal{D}}_0$ satisfies the following ``complexity estimates'':
\begin{enumerate}
    \item[(i-$\mathscr{D}$)] $Q\leq 
    8 (\epsilon_K^{-1} \operatorname{Lip}(\Phi)\operatorname{diam}(K)d^{\frac{5}{2}})^d
    $,
    \item[(ii-$\mathscr{D}$)] 
    $
    N\leq 
    \left(
    \frac{kd^2
    2^{\frac{9}{2}}
    \operatorname{Lip}(\Phi)(\operatorname{diam}(K)+\epsilon_f)
    }{
    \sqrt{d+1}\epsilon_K
    }
    \right)^d
    $
    \item[(iii-$\mathscr{D}$)] $\hat{\mathcal{D}}_0$ has depth at most 
    $
    \mathscr{O}\left(
(
d N^{\frac{3}{2}}( \operatorname{Lip}(\Phi)\operatorname{diam}(K) + 2\epsilon_f) (1-4^{-1}\epsilon_K^{-1})(1-\epsilon_K^{-1}) (1+4^{-1}d)
)^{2d}
\right)
.
    $
\end{enumerate}
\vspace{-.5em}
Here, we denote the Lipschitz constant of $\Phi|_K$ by $\operatorname{Lip}(\Phi)$. 
We have also used Jung's Theorem \citep{jung1910boundedingdiametresinEuclideanSpace} and the fact that 
$
\frac{
d^{2}
    }{2(d-1)(d+1)}<d^{\frac{5}{2}}
$ allows us to simplify the estimate in \citep[Theorem 3 (ii)]{kratsios2021_GCDs} to simplify the expression in (i-$\mathcal{D}$) and in (iii-$\mathscr{D}$).  

Since $\Phi^{-1}\circ i_d^m:p_d^m(\Phi(K))\rightarrow K$ is Lipschitz and since $\iota_d^m$ is $1$-Lipschitz then, $(\Phi^{-1}\circ i_d^m)_{\#}:\mathcal{P}_1(p_d^m(\Phi(K)))\rightarrow \mathcal{P}_1(K)$ is also Lipschitz with Lipschitz-constant at most $\operatorname{Lip}(\Phi^{-1})$. 
Let $\hat{\mathcal{D}}(\cdot)\triangleq (\Phi^{-1}\circ i_d^m)_{\#}
        \hat{\mathcal{D}}_0(\cdot)$.  
Thus,~\eqref{eq_thrm_quantitativer_version_NonConvex____approximation_quality_probabilistic_decoder_prepush} implies:
\begin{equation}
    \begin{aligned}
\sup_{y\in \Phi(K)_{\epsilon_f}}\,
    \www_1\left(
        (\Phi^{-1}\circ i_d^m)_{\#}(\Pi(y))
            ,
        \hat{\mathcal{D}}(y)
    \right)
    \leq & \operatorname{Lip}(\Phi^{-1})
        \sup_{y\in \Phi(K)_{\epsilon_f}}\,
    \www_1\left(
        \Pi(y)
            ,
        \hat{\mathcal{D}}_0(y)
    \right)
    \\
    \leq & \operatorname{Lip}(\Phi^{-1})\epsilon_K
    ,
    \end{aligned}
    \label{eq_thrm_quantitativer_version_NonConvex____approximation_quality_probabilistic_decoder}
\end{equation}
Moreover, the injectivity of $\Phi^{-1}\circ i_d^m$ implies that $\hat{\mathcal{D}}$ has the following simple expression:
\[
\hat{\mathcal{D}}:\rr^d\ni x \mapsto \sum_{k=1}^N
\operatorname{P-attention}\left(
    \hat{D}(x)
,
\tilde{Y}
\right)\in \mathcal{P}_1(K)
,
\]
where $\tilde{Y}$ is the $N\times Q\times m$-array with $Y_{n,q}=\Phi^{-1}\circ \iota^m_d(y_{n,q})$.  

Therefore, by~\eqref{eq_thrm_quantitativer_version_NonConvex____unique_minimizer_is_fx}, we have the following preliminary estimate:
\begin{align}
\label{PROOF_eq_thrm_quantitativer_version_NonConvex____Main_estimate_term_to_CONTROL__ie_upperbound_TBD}
    \sup_{x\in \xxx}
    \,\www_1\left(\hat{\mathcal{D}}\circ \hat{\mathcal{E}}(x),
            \underset{y\in K}{\operatorname{argmin}}\,
        L(x,y)
        \right)
        \leq &
    \sup_{x\in \xxx}
    \,\www_1\left(\hat{\mathcal{D}}\circ \hat{\mathcal{E}}(x),
    \delta_{f(x)}
    \right)
    \\
\label{PROOF_eq_thrm_quantitativer_version_NonConvex____Main_estimate_term_A}
    \leq &  \sup_{x\in \xxx}
    \,
    \left[\,
     \www_1\left(\hat{\mathcal{D}}\circ\hat{\mathcal{E}}(x)
        ,
        (\Phi^{-1}\circ \iota^m_d)_{\#}\circ\Pi\circ\hat{\mathcal{E}}(x)
        \right)
    \right.
    \\
\label{PROOF_eq_thrm_quantitativer_version_NonConvex____Main_estimate_term_B}
    & + \left.
     \www_1
        \left(
                (\Phi^{-1}\circ \iota^m_d)_{\#}\circ\Pi\circ \hat{\mathcal{E}}(x)
            ,
                (\Phi^{-1}\circ \iota^m_d)_{\#}\circ\Pi\circ f(x)
        \right)
        \right.
     \\
\label{PROOF_eq_thrm_quantitativer_version_NonConvex____Main_estimate_term_C}
    & + \left.
     \www_1
        \left(
                (\Phi^{-1}\circ \iota^m_d)_{\#}\circ\Pi\circ f(x)
            ,
                (\Phi^{-1}\circ \iota^m_d)_{\#}\circ\delta_{f(x)}
        \right)
    \right]
    .
\end{align}

To conclude the proof, we must first bound term~\eqref{PROOF_eq_thrm_quantitativer_version_NonConvex____Main_estimate_term_A}.  Since we found that $\hat{\mathscr{E}}(\xxx)\cup f(\xxx)\subseteq {\xxx_{\epsilon}}$ then, utilizing~\eqref{eq_thrm_quantitativer_version_NonConvex____approximation_quality_probabilistic_decoder} we compute:
\begin{equation}
\begin{aligned}
    \sup_{x\in \xxx}
    \,\www_1\left(\hat{\mathcal{D}}\circ\hat{\mathcal{E}}(x)
        ,
        (\Phi^{-1}\circ \iota^m_d)_{\#}\circ\Pi\circ\hat{\mathcal{E}}(x)
        \right)
= &
    \sup_{x\in \xxx}
    \,\www_1\left(
    (\Phi^{-1}\circ \iota^m_d)_{\#}\circ
    \hat{\mathcal{D}}_0\circ\hat{\mathcal{E}}(x)
        ,
        (\Phi^{-1}\circ \iota^m_d)_{\#}\circ\Pi\circ\hat{\mathcal{E}}(x)
        \right)
\\
\leq &
    \operatorname{Lip}(\Phi^{-1}) 1
    \sup_{y \in {\xxx_{\epsilon}}}\,\www_1\left(
        \hat{\mathcal{D}}_0(\hat{\mathcal{E}}(x))
        ,
        \Pi(\hat{\mathcal{E}}(x))
        \right)
\\
\leq &
    \operatorname{Lip}(\Phi^{-1})
    \sup_{y \in {\xxx_{\epsilon}}}\,\www_1\left(
        \hat{\mathcal{D}}_0(y)
        ,
        \Pi(y)
        \right)
\\ 
\overset{
    \eqref{eq_thrm_quantitativer_version_NonConvex____approximation_quality_probabilistic_decoder}
}{\leq}
&
    \operatorname{Lip}(\Phi^{-1})\epsilon_K
;
    \end{aligned}
    \label{eq_thrm_quantitativer_version_NonConvex____Final_Main_Estimate____BOUND_on_PART_A}
\end{equation}
where $\operatorname{Lip}(\Phi^{-1})$ denotes the Lipschitz constant of $\Phi$ on $\Phi(K)_{\epsilon_f}$.  
For the second term, i.e.~\eqref{PROOF_eq_thrm_quantitativer_version_NonConvex____Main_estimate_term_B}, we combine the fact that $\Pi$ is Lipschitz with $\operatorname{Lip}(\Pi)$ given in~\eqref{eq_thrm_quantitativer_version_NonConvex____Lipschitz_constant_bound_intermsof_metriccapacity} and our estimate on $f$ obtained in~\eqref{eq_thrm_quantitativer_version_NonConvex____epsilonthickeningof_K} to find that:
\begin{equation}
    \begin{aligned}
    \sup_{x\in \xxx}
    &
    \,
        \www_1
        \left(
                (\Phi^{-1}\circ i_d^m)_{\#}\circ\Pi\circ \hat{\mathcal{E}}(x)
            ,
                (\Phi^{-1}\circ i_d^m)_{\#}\circ\Pi\circ f(x)
        \right)
        \\
    \leq &
    \operatorname{Lip}(\Phi^{-1}) 
    \sup_{x\in \xxx}\,
        \www_1
        \left(
                \Pi\circ \hat{\mathcal{E}}(x)
            ,
                \Pi\circ f(x)
        \right)
    \\ \leq &
    \operatorname{Lip}(\Phi^{-1})
    \sup_{x\in \xxx}\,
        \operatorname{Lip}(\Pi)
        \www_1
        \left(
                \hat{\mathcal{E}}(x)
            ,
                f(x)
        \right)
    \\ 
    \overset{
    \eqref{eq_thrm_quantitativer_version_NonConvex____Lipschitz_constant_bound_intermsof_metriccapacity}
    }{\leq}
    &
    k \operatorname{Lip}(\Phi^{-1}) d
    \sup_{x\in \xxx}\,
        \www_1
        \left(
                \hat{\mathcal{E}}(x)
            ,
                f(x)
        \right)
    \\ 
    \overset{
    \overset{
        \eqref{PROOF_thrm_quantitativer_version_NonConvex__eq_extension}
    }{
        +\eqref{eq_thrm_quantitativer_version_NonConvex____eq_first_estimate}}
    }{\leq}
     &
    k \operatorname{Lip}(\Phi^{-1})d
    \epsilon_f
    \end{aligned}
    \label{eq_thrm_quantitativer_version_NonConvex____Final_Main_Estimate____BOUND_on_PART_B}
    .
\end{equation}
The third term, i.e.~\eqref{PROOF_eq_thrm_quantitativer_version_NonConvex____Main_estimate_term_C}, we use the random projection property of $\Pi$ on $K$ defined in~\eqref{eq_thrm_quantitativer_version_NonConvex____random_projection_facts} and the assumption that $f(\xxx)\subseteq K$.  This is done as follows:
\begin{equation}
    \begin{aligned}
    \sup_{x\in \xxx}
    \,
    \mathcal{W}_1\left(
        (\Phi^{-1}\circ \iota_d^m)_{\#}\circ 
        \Pi \circ f(x)
            ,
        (\Phi^{-1}\circ \iota_d^m)_{\#}\delta_{f(x)}
    \right)
\leq &
    \operatorname{Lip}(\Phi^{-1}\circ \iota_d^m)
        \,
    \mathcal{W}_1\left(
        \Pi\circ f(x)
            ,
        f(x)
    \right)
\\
= &   
    \sup_{x\in \xxx}
    \,
    \operatorname{Lip}(\Phi^{-1})
    \operatorname{Lip}(\iota_d^m)
        \,
    \mathcal{W}_1\left(
        \Pi\circ f(x)
            ,
        f(x)
    \right)
\\
= &
    \sup_{x\in \xxx}
    \,
    \operatorname{Lip}(\Phi^{-1})
        \,
    \mathcal{W}_1\left(
        \Pi\circ f(x)
            ,
        f(x)
    \right)
\\
\overset{
\overset{\because f(x)\in K}{
        +
    \eqref{eq_thrm_quantitativer_version_NonConvex____random_projection_facts}
}
}{\leq } &
    \sup_{y \in K}
    \,
    \operatorname{Lip}(\Phi^{-1})
        \,
    \mathcal{W}_1\left(
        \Pi(y)
            ,
        y
    \right)
\\
= & 0
\end{aligned}
\label{eq_thrm_quantitativer_version_NonConvex____Final_Main_Estimate____BOUND_on_PART_C}
    .
\end{equation}
Therefore, incorporating~\eqref{eq_thrm_quantitativer_version_NonConvex____Final_Main_Estimate____BOUND_on_PART_A},~\eqref{eq_thrm_quantitativer_version_NonConvex____Final_Main_Estimate____BOUND_on_PART_B}, and~\eqref{eq_thrm_quantitativer_version_NonConvex____Final_Main_Estimate____BOUND_on_PART_C}, we may control the right-hand of~\eqref{PROOF_eq_thrm_quantitativer_version_NonConvex____Main_estimate_term_to_CONTROL__ie_upperbound_TBD} with the following upper-bound:
\begin{equation}
\begin{aligned}
\sup_{x\in \xxx}\,\www_1\left(\hat{\mathcal{D}}\circ \hat{\mathcal{E}}(x),
            \underset{y\in K}{\operatorname{argmin}}\,
        L(x,y)
        \right)
\leq &
    \epsilon_K
    +
    k\operatorname{Lip}(\Phi^{-1}) d
    \epsilon_f
    +
    0
    .
\end{aligned}
\label{eq_super_formulation_UAT}
\end{equation}
Relabelling $k\operatorname{Lip}(\Phi^{-1})$ as $k$ and the $\Phi^{-1}\circ i_d^m(y_{k,q})$ as $y_{k,q}$, and incorporating the rate $d\in \mathscr{O}(m^{\frac1{s}})$ (implied by Assumption~\ref{ass_low_dimensional_mainfold}) into the complexity estimates (i)-(iii) and (i-$\mathcal{D}$)-(iii-$\mathcal{D}$) yields the rates of Table~\ref{tab_model_complexities} and, the explicit rates of Table~\ref{tab_model_complexities}.  Thus the proof is complete.
\end{proof}

\section{Proof of Corollaries}\label{a_Proofs_ss_Corollaries}
This appendix contains proofs of the paper's main corollaries. 
\subsection{Proof of Corollary~\ref{cor_simplified_version_of_theorem_DeepBMT}}
\begin{proof}[{Proof of Corollary~\ref{cor_simplified_version_of_theorem_DeepBMT}}]
Let $(\Omega,\mathcal{F},\nu)$ be a probability space on which the random-field $(Y^x)_{x\in \rr^n}$, satisfying $Y^x\sim F(x)$, is defined.  Consider the non-empty-valued correspondence\footnote{Also called multifunction, multivalued function, or set-valued function.} $\mathscr{O}(x)$ is defined as in~\eqref{PROOF_theorem_DeepBMT_eq_optimality_set}.  
We continue where Theorem~\ref{theorem_DeepBMT}'s proof left off.  For any $x\in {\xxx_{\epsilon}}$, we now compute the concrete lower-bound on
$\underset{y^{\star}\in 
\mathcal{O}(x)
    }{\inf}\,
    \www_1(\hat{F}(x),\delta_{y^{\star}})
    .
$  
    \allowdisplaybreaks
\begin{align}
\nonumber
    \allowdisplaybreaks
    \underset{y^{\star}\in      
    \mathscr{O}(x)
    }{\inf}\,
        \www_1(\hat{F}(x),\delta_{y^{\star}})
    = &
    \underset{{y \in C_x\cap K}}{\inf}
    \inf_{\pi \in \operatorname{Cpl}(\hat{F}(x),\delta_y)}\,
        \int_{(u,v)\in \rr^n\times \rr^n}\, \|u-v\| \pi(d(u,v))
    \\
    \nonumber
    = &
    \underset{y^{\star}\in      
    \mathscr{O}(x)
    }{\inf}\,
    \int_{(u,v)\in K\times K}\, \|u-v\| \left(\hat{F}(x)\otimes \delta_{y}(d(u,v))\right)
    \\
    \nonumber
    = &
    \underset{y^{\star}\in   
    \mathscr{O}(x)
    }{\inf}\,
    \int_{u \in K}\int_{v\in K}
    \, \|x-v\| \hat{F}(x)(du) \delta_{y}(dv)
    \\
    \nonumber
    = &
    \underset{y^{\star}\in   
    \mathscr{O}(x)
    }{\inf}\,
    \int_{u \in K}
    \, \|x-y^{\star}\| \hat{F}(x)(du)
    \\
    \nonumber
    \overset{\text{(def)}}{=} &
    \underset{y^{\star}\in   
    \mathscr{O}(x)
    }{\inf}\,
        \mathbb{E}_{Y^x\sim \hat{F}(x)}
        \left[\|Y^x-y^{\star}\|\right]
    \\
    \label{eq_proof_theorem_DeepBMT_lowerbound_Fatou}
    \geq  &
        \mathbb{E}_{Y^x\sim \hat{F}(x)}
        \left[
        \underset{y^{\star}\in 
    \mathscr{O}(x)
    }{\operatorname{ess-inf}}\,
            \|Y^x-y^{\star}\|
        \right]
    \\
    \overset{\operatorname{(def)}}{=} & 
    \mathbb{E}_{Y^x\sim \hat{F}(x)}
        \left[
        \left\|
            Y^x
        -
    \underset{{y\in C_x\cap K}}{\operatorname{argmin}}
    \, L(x,y)
        \right\|
        \right]
    .
\label{eq_proof_theorem_DeepBMT_lowerbound}
\end{align}
In more detail: The first equality is just the definition of the Wasserstein distance.
The second equality follows from the fact that for any $y^{\star}\in \rr^n$ (and in particular any such $y^{\star}$ in $C_x\cap K$) the product measure $\hat{F}(x)\otimes \delta_{y^{\star}}$ is the only coupling of $\hat{F}(x)$ with $\delta_{y^{\star}}$ (see Lemma~\ref{lem_closedform_wasserstein_dirac}) and the facts that, by (i), $\hat{F}(x)$ is supported in $K$ and, by definition, $\delta_{y^{\star}}$ is also supported in K.  
The third equality follows the Fubini-Tonelli Theorem (see \citep[Theorem 1.27]{FoundationsofModernProbabilityKallenberg3rd2021}) since all involved quantities are integrable over the compact set $K\times K$.  The inequality~\eqref{eq_proof_theorem_DeepBMT_lowerbound_Fatou} follows from Fatou's Lemma (see \citep[Lemma 1.20]{FoundationsofModernProbabilityKallenberg3rd2021}) and the fact that the \textit{essential-infimum} lower-bounds the infimum.  The final equality is just the definition of the distance from $Y^x(\omega)$ to the optimality set $\mathscr{O}(x)
$ (for each $\omega\in \Omega$).    
Combining the upper-bound on the right-hand side of~\eqref{eq_proof_theorem_DeepBMT_2_main_upperbound_abstract} with the lower-bound in~\eqref{eq_proof_theorem_DeepBMT_lowerbound} yields the result.  
\end{proof}
\subsection{{Proof of Corollary~\ref{cor_random_uniform_UAT}}}\label{a_Appendix_ss_ProofOF_cor_random_uniform_UAT}
\begin{proof}[{Proof of Corollary~\ref{cor_random_uniform_UAT}}]
We continue with the notation of Theorem~\ref{thrm_quantitativer_version_NonConvex}, and specifically with the following estimate derived in~\eqref{eq_super_formulation_UAT}:
\begin{equation}
    \sup_{x \in \xxx}\,
    \www_1\left(
    \hat{\mathcal{D}}\circ \hat{\mathcal{E}}(x)
            ,
        \delta_{f(x)}
    \right)
    \leq 
    \epsilon_K
    +
    k \operatorname{Lip}(\Phi^{-1})d
    \epsilon_f
    \label{PROOF_eq_cor_random_uniform_UAT_strong_estimate}
    .
\end{equation}
Combining~\eqref{PROOF_eq_cor_random_uniform_UAT_strong_estimate}, the monotonicity of integration, 
Assumption~\ref{ass_normcontrolled_loss}, Jensen's inequality (applicable due to the concavity of $l$), and from Lemma~\ref{lem_closedform_wasserstein_dirac}, we deduce the following estimate: for each $x \in \xxx$:
\allowdisplaybreaks
\begin{align}
\nonumber
    \mathbb{E}_{
        Y^x\sim \hat{\mathcal{D}}\circ \hat{\mathscr{E}(x)}
        }\left[
        L(x,Y^x)
        \right]
     \leq &
    \mathbb{E}_{
        Y^x\sim \hat{\mathcal{D}}\circ \hat{\mathscr{E}(x)}
        }\left[
        l(\|f(x)-Y^x\|)
        \right]
\\
\nonumber
\leq & 
l\left(
    \mathbb{E}_{
        Y^x\sim \hat{\mathcal{D}}\circ \hat{\mathscr{E}(x)}
        }\left[
        \|f(x)-Y^x\|
        \right]
\right)
\\
    \nonumber
    \overset{\eqref{lem_closedform_wasserstein_dirac}}{
    =}
&
l\left(
\www_1\left(
    \hat{\mathcal{D}}\circ \hat{\mathcal{E}}(x)
            ,
        \delta_{f(x)}
    \right) 
\right)
\\
\label{PROOF_eq_cor_random_uniform_UAT_monotonicity}
\overset{\eqref{PROOF_eq_cor_random_uniform_UAT_strong_estimate}}{\leq}
&
l( \epsilon_K
    +
    k \operatorname{Lip}(\Phi^{-1})d
    \epsilon_f
).
\end{align}
\end{proof}
\subsection{Proof of Corollary~\ref{cor_convex_case}}\label{a_Appendix_ss_Proof_Corollary_Convex_Case_Constrained_UAT}
\begin{proof}[{Proof of Corollary~\ref{cor_convex_case}}]
Let $L: \rr^n \times \rr^n \ni (x,y) \mapsto  \|f(x) - y\| \in [0, \infty) $.  Then, for each fixed $x \in \rr^n$, the strict convexity of $y\mapsto L(x,y)$ and the assumption that $f(x)\in K$ imply that $\{f(x)\}=\operatorname{argmin}_{y\in K}\, L(x,y)$.  Thus, for each $x \in \xxx$ and each $\pp\in \ppp{K}$, we have that:
\begin{equation}
    \mathcal{W}_1\left(
        \pp
                ,
            \delta_{f(x)}
    \right)
=
    \mathcal{W}_1\left(
\pp
            ,
            \underset{y\in K}{\operatorname{argmin}}\,
        L(x,y)
    \right)
.
\label{PROOF_cor_convex_case_Reduction_to_fx_from_optimality_Set}
\end{equation}
Since $f\in C_{tr}^{k}(\xxx,K)$ and $K$ is non-empty satisfying Assumption~\ref{ass_low_dimensional_mainfold}, and compact Theorem~\ref{thrm_quantitativer_version_NonConvex} implies that for each $\epsilon_K,\epsilon_f>0$ there exist a $\hat{\mathcal{D}}$ and a $\hat{\mathcal{E}}$ as in Table~\ref{tab_model_complexities} satisfying the estimate:
\begin{equation}
    \mathcal{W}_1\left(
\hat{\mathcal{D}}\circ \hat{\mathcal{E}}(x)
            ,
        \underset{y\in K}{\operatorname{argmin}}\,
        L(x,y)
    \right)
        \leq
    \epsilon_K + kd\epsilon_f
.
\label{PROOF_cor_convex_case_application_of_Theorem}
\end{equation}
Define the map $\beta:\ppp{K}\ni \pp \mapsto \mathbb{E}_{Y\sim \pp}[Y]\in \rr^m$.  By Lemma~\ref{lem_convex_barycenter_map} (i), $\beta$ takes values in $K$ and according to Lemma~\ref{lem_convex_barycenter_map} (ii) it is $1$-Lipschitz.  Notice also that $\beta(\delta_y)=y$ for each $y\in K$ and, in particular, $\beta (\delta_{f(x)}=f(x))$.  These observations together with~\eqref{PROOF_cor_convex_case_Reduction_to_fx_from_optimality_Set} and~\eqref{PROOF_cor_convex_case_application_of_Theorem} imply that:
\allowdisplaybreaks
\begin{align}
\nonumber
\sup_{x\in \xxx}\, \|f(x)-\mathbb{E}_{Y^x\sim \hat{\mathcal{D}}\circ \hat{\mathcal{E}}(x)}[Y^x]\| 
=
&
\sup_{x\in \xxx}\, \|\beta(\delta_{f(x)})- \beta(\hat{\mathcal{D}}\circ  \hat{\mathcal{E}})(x)\| 
\\
\nonumber
\leq &
\sup_{x\in \xxx}\,
1\cdot
\www_1\left(
    \delta_{f(x)}
        ,
  \hat{\mathcal{D}}\circ \hat{\mathcal{E}}(x)
\right)
\\
\nonumber
\overset{\eqref{PROOF_cor_convex_case_application_of_Theorem}}{\leq}& 
\sup_{x\in \xxx}\,
1\cdot
\www_1\left(
\hat{\mathcal{D}}\circ \hat{\mathcal{E}}(x)
        ,
    \underset{y\in K}{\operatorname{argmin}}\,
    L(x,y)
\right)
\\  
\nonumber
\overset{\mbox{Thm.} \ref{thrm_quantitativer_version_NonConvex}}{\leq}& 
\epsilon_K  +k d \epsilon_f.
\end{align}
Whence, (i) and (ii) hold.  
\end{proof}
\subsection{{Proof of Corollary~\ref{cor_nonconvex_geometric_case}}}
\begin{proof}[{Proof of Corollary~\ref{cor_nonconvex_geometric_case}}]
Let $L:\rr^n\times \rr^n\ni (x,y)\mapsto \|f(x)-x\|\in [0,\infty)$, note that $L$ satisfies Assumption~\ref{ass_normcontrolled_loss}, and that for each $x\in \xxx$ we have that $\operatorname{argmin}_{y\in K} \, L(x,y)=\{f(x)\}$.  
By Theorem~\ref{thrm_quantitativer_version_NonConvex}, for every $f\in C_{tr}^k(\xxx,K)$ and for every $\epsilon>0$, 
there exist a $\hat{\mathcal{D}}$ and a $\hat{\mathcal{E}}$ as in Table~\ref{tab_model_complexities} satisfying $\hat{\mathcal{D}}\circ \hat{\mathcal{E}}(x)\in \ppp{K}$ for each $x\in \rr^n$ and satisfying the uniform estimate:
\begin{equation}
    \max_{x\in \xxx}\,\www_1(\delta_{f(x)},\hat{\mathcal{D}}\circ \hat{\mathcal{E}}(x))
    \leq 
\epsilon_K
    +
    k \operatorname{Lip}(\Phi^{-1})d
    \epsilon_f
    \label{PROOF_cor_nonconvex_geometric_case_eq_UAT_estimate}
    .
\end{equation}
Since $K$ satisfies Assumption~\ref{ass_geodesic_general} then, Lemma~\ref{lem_cbm} applies.  Therefore,~\eqref{eq_lem_cbm_definition_of_intrinsic_mean_map} Theorem~\ref{thrm_quantitativer_version_NonConvex} imply:
\allowdisplaybreaks
\begin{align}
\nonumber
\max_{x\in \xxx}\,
    d_g\left(
    f(x)
        ,
    \overline{\hat{\mathcal{D}}\circ \hat{\mathcal{E}}(x)}
    \right)
= &
\max_{x\in \xxx}\,
    d_g\left(
    \overline{\delta_{f(x)}}
        ,
    \overline{\hat{\mathcal{D}}\circ \hat{\mathcal{E}}(x)}
    \right)
\\
\nonumber
\overset{\eqref{eq_lem_cbm_definition_of_intrinsic_mean_map}}{\leq}
 &
\max_{x\in \xxx}\,
1\,
    \www_1(\delta_{f(x)},\hat{\mathcal{D}}\circ \hat{\mathcal{E}}(x))
\\
\nonumber
\overset{\mbox{Thm.}\ref{thrm_quantitativer_version_NonConvex}}{\leq}
& 
\epsilon_K
    +
    k \operatorname{Lip}(\Phi^{-1})d
    \epsilon_f
\end{align}
Furthermore,~\eqref{eq_lem_cbm_inclusion} and the fact that $\ppp{K}\ni \pp\mapsto \overline{\pp}\in K$ is a left-inverse of the map $K\ni y \mapsto \delta_y$ imply that: for every $x \in \rr^n$ it follows that:
\[
\overline{\hat{\mathcal{D}}\circ \hat{\mathcal{E}}(x)}\in K
.
\]
This concludes the proof.
\end{proof}

\subsubsection{{Discussion: Theorem~\ref{thrm_quantitativer_version_NonConvex} vs. Corollary~\ref{cor_random_uniform_UAT}}}\label{remark_understanding_estiamte}
The modulus of continuity $\omega$ in Assumption~\ref{ass_normcontrolled_loss} does not enter into the estimate in Theorem~\ref{thrm_quantitativer_version_NonConvex} (ii) but it does appear in the estimate of Corollary~\ref{cor_random_uniform_UAT}.  
This is because\footnote{
The right-most expression in~\eqref{eq_remark_understanding_estiamte} is justified in Lemma~\ref{lem_closedform_wasserstein_dirac}; see Corollary~\ref{cor_random_uniform_UAT}'s proof.  
}:
\begin{equation}
    \www_1(\hat{\mathcal{D}}\circ \hat{\mathcal{E}}(x),\operatorname{argmin}_{y\in K} \,L(x,y))
=
\inf_{y\in \operatorname{argmin}_{y\in K} \,L(x,y)}\, \mathbb{E}_{Y^x\sim \hat{\mathcal{D}}\circ \hat{\mathcal{E}}(x)}\left[
\|Y^x -y\|
\right]
\label{eq_remark_understanding_estiamte}
.
\end{equation}
Thus, the right-hand side is controlled by the  of~\eqref{eq_remark_understanding_estiamte} the average (in $Y^x$) worst-case (in $x$) Euclidean average distance between $Y^x$ and the optimality set $\operatorname{argmin}_{y\in K} \,L(x,y)$; whereas, the estimate in Corollary~\ref{cor_random_uniform_UAT} is controlling the average (in $Y^x$) worst-cast (in $x$) loss $L(x,Y^x)$.  In other words, Corollary~\ref{cor_random_uniform_UAT} controls the optimal value of $L$ on $K$ and Theorem~\ref{thrm_quantitativer_version_NonConvex} approximates the optimal prediction.
\section{Further Corollaries to the Deep Maximum Theorem}\label{a_ss_Additional_Results}
This brief appendix contains additional corollaries of Theorem~\ref{theorem_DeepBMT} which were not included in our manuscript's main body.  The intent here is to show how our ``Deep Maximum Theorem'' simplifies in the convex case, a similar result can be derived for the geodesically convex case.  
\begin{corollary}[Deep Maximum Theorem: Convex Case]\label{cor_Deep_BMT_Convex_Casee}
Assume the context of Theorem~\ref{theorem_DeepBMT}.  Let $\{Y^x\}_{x\in \rr^n}$ be an $\rr^m$-valued random field with $X^x\sim \hat{F}$.  If each $C_x\cap K$ is a convex set and $L$ is strictly convex then, $\rr^n\ni x \mapsto \mathbb{E}[Y^x]\in \rr^m$ has the following representation:
\begin{equation}
    \mathbb{E}[Y^x] = \operatorname{Attention}(\hat{f}(x),Y)
    \label{cor_Deep_BMT_Convex_Case_Appendixed_version}
    ,
\end{equation}
where $Y=(\sum_{q=1}^Q \frac1{Q}y_{k,q})_{k=1}^N$ is an $N\times m$-matrix.  Moreover, the following hold:
\begin{enumerate}
\item[(i)] \textbf{Constraint Satisfaction:} 
    $
    \mathbb{E}[Y^x] \in K
    $
     for each $x\in \rr^n$
    ,
\item[(ii)] \textbf{Probable Optimality:} $\underset{x\in {\xxx_{\epsilon}}}{\max}\,
        \left\|
        \mathbb{E}
        [Y^x]
    -
        y^{\star}(x)
    \right\|
    \leq \epsilon,$
\end{enumerate}
where $y^{\star}(x)$ is the well-defined and unique minimizer of $L(x,\cdot)$ on $C_x\cap K$.  
\end{corollary}
\begin{proof}[{Proof of Corollary~\ref{cor_Deep_BMT_Convex_Casee}}]
First we note that since each $C_x\cap K$ is a non-empty, compact, and convex subset of $\rr^n$ and since $L$ is strictly convex and bounded-below on $K\cap C_x$ (since it is continuous and $K\cap C_x$ is compact) then it must have a unique minimizer (see \cite{GenericNessOfSingleValuedConvexFunctions2016}).  Thus, $y^{\star}(x)$ exists and is uniquely defined for each $x \in \rr^n$.  

Consider the setting of Theorem~\ref{theorem_DeepBMT} and suppose further that $K$ is convex.  Then, we may apply Lemma~\ref{lem_convex_barycenter_map}.  
Thus, in the notation of Theorem~\ref{theorem_DeepBMT}, for each $x \in {\xxx_{\epsilon}}$ and every $y^{\star}\in
\underset{y\in C_x\cap K}{\operatorname{argmin}}\, L(x,y)$ we have the estimate:
\begin{equation}
    \left\|
        \mathbb{E}_{Y\sim \hat{F}(x)}[Y]
    -
        \mathbb{E}_{\tilde{Y}\sim \delta_{y^{\star}}}[\tilde{Y}]
    \right\|
    \leq 
    \www_1\left(
            \hat{F}(x)
                ,
            y^{\star}
        \right)
    \label{eq_cor_Deep_BMT_Convex_Case_barycenter_map_simple_estimate}
    .
\end{equation}
Applying the estimate: $\underset{x\in {\xxx_{\epsilon}}}{\max}\,
    \underset{y^{\star}\in 
    \underset{{y\in C_x\cap K}}{\operatorname{argmin}}
    \, L(x,y)}{\inf}\,
    \www_1(\hat{F}(x),\delta_{y^{\star}})\leq \epsilon$ obtained in Theorem~\ref{theorem_DeepBMT} to the right-hand side of~\eqref{eq_cor_Deep_BMT_Convex_Case_barycenter_map_simple_estimate}, and noting that 
    $\mathbb{E}_{Y\sim \delta_{y^{\star}}}[Y]=y^{\star}$ yields:
\[
  \begin{aligned}
        \underset{x\in {\xxx_{\epsilon}}}{\max}\,
    \underset{y^{\star}\in 
    \underset{{y\in C_x\cap K}}{\operatorname{argmin}}
    \, L(x,y)}{\inf}\,
        \left\|
        \mathbb{E}_{Y\sim \hat{F}(x)}[Y]
    -
        y^{\star}
    \right\|
    = &
     \underset{x\in {\xxx_{\epsilon}}}{\max}\,
    \underset{y^{\star}\in 
    \underset{{y\in C_x\cap K}}{\operatorname{argmin}}
    \, L(x,y)}{\inf}\,
        \left\|
        \mathbb{E}_{Y\sim \hat{F}(x)}[Y]
    -
        \mathbb{E}_{\tilde{Y} \sim \delta_{y^{\star}}}[\tilde{Y}]
    \right\|
    \\
    \leq &
    \underset{x\in {\xxx_{\epsilon}}}{\max}\,
    \underset{y^{\star}\in 
    \underset{{y\in C_x\cap K}}{\operatorname{argmin}}
    \, L(x,y)}{\inf}\,
    \www_1\left(
            \hat{F}(x)
                ,
            y^{\star}
        \right)
    \\ \leq
    & \epsilon
        .
  \end{aligned}
\]
This gives the second part of the statement.  

Since $\operatorname{supp}(\hat{F}(x))\subseteq K$ then, any $\rr^n$-valued random-vector distributed according to $\hat{F}(x)$, $\hat{F}(x)$-a.s. takes values in $K$.  Thus, $\hat{F}(x)\left(X\in K\right)=\mathbb{E}_{X\sim \hat{F}(x)}[I_K(X)]=1$.  This gives the first claim.  
\end{proof}
For completeness, we include the deterministic analogue of Corollary~\ref{cor_Deep_BMT_Convex_Casee} when $K$ is a geodesically convex subset of a complete connected Riemannian submanifold $(M,g)$ of $\rr^m$ satisfying Assumption~\ref{ass_geodesic_general}.  The result is a qualitative generalization of Corollary~\ref{cor_nonconvex_geometric_case}.  
\begin{corollary}[Deep Maximum Theorem: Riemannian Case]\label{cor_deepBerge_non_convex_constraints}
Assume the context of Theorem~\ref{theorem_DeepBMT} and suppose that Assumption~\ref{ass_geodesic_general} holds.  Suppose also that for each $x\in [0,1]^n$ there exists a unique $y(x)\in C_x\cap K$ minimizing $L$; i.e.:
\[
L(x,y(x)) = \inf_{y\in C_x\cap K}\, L(x,y),
\]
moreover, assume that $x\mapsto y(x)$ is continuous on $[0,1]^n$.  Then, the function:
\begin{equation}
    [0,1]^n\ni x \mapsto  
    \overline{
        \operatorname{P-attention}(\hat{f}(x),Y)
    }
    \label{cor_Deep_BMT_Non_Convex_and_Geometric_Case_appenxied_Version}
    ,
\end{equation}
is well-defined; moreover, the following hold:
\begin{enumerate}
\item[(i)] \textbf{Constraint Satisfaction:} 
    $
    \overline{
        \operatorname{P-attention}(\hat{f}(x),Y)
    } \in K
    $
     for each $x\in \rr^n$
    ,
\item[(ii)] \textbf{Probable Optimality:} $\underset{x\in {\xxx_{\epsilon}}}{\max}\,
        d_g\left(
        \overline{
        \operatorname{P-attention}(\hat{f}(x),Y)
    }
    ,
        y^{\star}(x)
    \right)
    \leq \epsilon,$
\end{enumerate}
where $y^{\star}(x)$ is the well-defined and unique minimizer of $L(x,\cdot)$ on $C_x\cap K$.  
\end{corollary}
The proof of Corollary~\ref{cor_deepBerge_non_convex_constraints} is nearly identical to that of Corollary~\ref{cor_Deep_BMT_Convex_Case}.  
\begin{proof}[{Proof of Corollary~\ref{cor_deepBerge_non_convex_constraints}}]
Consider the setting of Theorem~\ref{theorem_DeepBMT} and suppose further that $K$ satisfies Assumption~\ref{ass_geodesic_general}.  Then, we may apply Lemma~\ref{lem_cbm}.  
Thus, in the notation of Theorem~\ref{theorem_DeepBMT}, for each $x \in {\xxx_{\epsilon}}$ and every $y^{\star}\in
\underset{y\in C_x\cap K}{\operatorname{argmin}}\, L(x,y)$ we have the estimate:
\begin{equation}
    d_g\left(
        \overline{\hat{F}(x)}
    ,
        \overline{\delta_{y^{\star}}}
    \right)
    \leq 
    \www_1\left(
            \hat{F}(x)
                ,
            y^{\star}
        \right)
    \label{eq_cor_Deep_BMT_Convex_Case_barycenter_map_simple_estimate_proofing}
    .
\end{equation}
Applying the estimate: $\underset{x\in {\xxx_{\epsilon}}}{\max}\,
    \underset{y^{\star}\in 
    \underset{{y\in C_x\cap K}}{\operatorname{argmin}}
    \, L(x,y)}{\inf}\,
    \www_1(\hat{F}(x),\delta_{y^{\star}})\leq \epsilon$ obtained in Theorem~\ref{theorem_DeepBMT} to the right-hand side of~\eqref{eq_cor_Deep_BMT_Convex_Case_barycenter_map_simple_estimate_proofing}, and noting that 
    $\overline{\delta_{y^{\star}}}=y^{\star}$ yields:
\[
  \begin{aligned}
        \underset{x\in {\xxx_{\epsilon}}}{\max}\,
    \underset{y^{\star}\in 
    \underset{{y\in C_x\cap K}}{\operatorname{argmin}}
    \, L(x,y)}{\inf}\,
    d_g\left(
        \overline{\hat{F}(x)}
    ,
        y^{\star}
    \right)
    = &
     \underset{x\in {\xxx_{\epsilon}}}{\max}\,
    \underset{y^{\star}\in 
    \underset{{y\in C_x\cap K}}{\operatorname{argmin}}
    \, L(x,y)}{\inf}\,
    d_g\left(
        \mathbb{E}_{Y\sim \hat{F}(x)}[Y]
    ,
        \mathbb{E}_{\tilde{Y} \sim \delta_{y^{\star}}}[\tilde{Y}]
    \right)
    \\
    \leq &
    \underset{x\in {\xxx_{\epsilon}}}{\max}\,
    \underset{y^{\star}\in 
    \underset{{y\in C_x\cap K}}{\operatorname{argmin}}
    \, L(x,y)}{\inf}\,
    \www_1\left(
            \hat{F}(x)
                ,
            y^{\star}
        \right)
    \\ \leq
    & \epsilon
        .
  \end{aligned}
    \label{eq_cor_Deep_BMT_Convex_Case_final_string_of_esitmates}
\]
This gives (ii).  Lastly, (i) follows from~\eqref{eq_lem_cbm_inclusion} in Lemma~\ref{lem_cbm}.  
\end{proof}
%
\end{document}